\RequirePackage[final]{graphicx}
\documentclass[journal, twocolumn, final]{IEEEtran}

\usepackage{amsfonts}
\usepackage{amssymb,amscd,amsthm}
\usepackage[cmex10]{amsmath}
\usepackage{mathtools} 
\usepackage{ifthen}
\usepackage[applemac]{inputenc}

\usepackage{endnotes}

\usepackage{subcaption}

\usepackage{color}
\usepackage{bm}

\usepackage[acronym,shortcuts,nonumberlist]{glossaries}
\newacronym{rhs}{RHS}{right-hand side}
\newacronym{lhs}{LHS}{left-hand side}
\newacronym{ssc}{SSC}{sparse subspace clustering}
\newacronym{pca}{PCA}{principal component analysis}
\newacronym{gpca}{GPCA}{generalized principal component analysis}
\newacronym{rip}{RIP}{restricted isometry property}
\newacronym{svd}{SVD}{singular value decomposition}
\newacronym{jl}{JL}{Johnson-Lindenstrauss}
\newacronym{cs}{CS}{compressive sensing}
\newacronym{rp}{RP}{random projections}
\newacronym{rsc}{RSC}{robust subspace clustering}
\newacronym{omp}{OMP}{orthogonal matching pursuit}
\newacronym{mp}{MP}{matching pursuit}
\newacronym{tsc}{TSC}{thre\-shol\-ding-based subspace clustering}
\newacronym{bp}{BP}{basis pursuit}
\newacronym{ce}{CE}{clustering error}
\newacronym{admm}{ADMM}{alternating direction method of multipliers}
\newacronym{fde}{FDE}{feature detection error}
\newacronym{tp}{TP}{true positives}
\newacronym{fp}{FP}{false positives}
\newacronym{tpr}{TPR}{true positive rate}
\newacronym{fpr}{FPR}{false positive rate}
\newacronym{nfc}{NFC}{no false connections}
\newacronym{nsn}{NSN}{nearest subspace neighbor}
\newacronym{dd}{DD}{data-dependent}
\newacronym{di}{DI}{data-independent}
\newacronym{sc}{SC}{spectral clustering}

\usepackage{pgf,tikz}
\usepackage{pgfplots}
\usepackage{amsmath}
\usetikzlibrary{calc,shadows}
\usetikzlibrary{pgfplots.groupplots}
\usetikzlibrary{arrows.meta}
\usepackage{multirow}

\usepackage[applemac]{inputenc}

\usepackage{amsthm}

\pgfplotsset{width=7cm,compat=1.3}

\pgfkeys{/pgf/images/include external/.code={\includegraphics{#1}}}

\usetikzlibrary{external}
\usepackage{hyperref}
\hypersetup{
    bookmarks=true,         
    unicode=false,          
    pdftoolbar=true,        
    pdfmenubar=true,        
    pdffitwindow=false,     
    pdfstartview={FitH},    
    pdfauthor={Michael Tschannen},     
    pdfsubject={Subject},   
    pdfcreator={Michael Tschannen},   
    pdfproducer={Producer}, 
    pdfnewwindow=true,      
    colorlinks=false,       
    linkcolor=red,          
    citecolor=black, 
    filecolor=magenta,      
    urlcolor=cyan           
}

\usepackage{wrapfig}
\usepackage{placeins}

\usepackage{url}

\usepackage{tabularx}

\newcommand{\rev}[1]{#1}
\newcommand{\revb}[1]{#1}

\newtheorem{theorem}{Theorem}
\newtheorem{lemma}{Lemma}
\newtheorem{definition}{Definition}

 \newcommand{\mc}[0]{\mathcal }
 \newcommand{\mb}[0]{\mathbb }

\usepackage{framed}

\newcommand\norm[2][\Tnorm]{\ensuremath{{\left\Vert #2 \right\Vert}_{#1}}}

\newcommand\Tinnerprod{}
\newcommand{\innerprod}[3][\Tinnerprod]{\ifthenelse{\equal{#1}{}}{\ensuremath{\left<#2,#3\right>}}{\ensuremath{\left<#2,#3\right>_{#1}}}}

\newcommand\vect[1]{\mathbf #1}

\definecolor{DarkBlue}{rgb}{0.1,0.1,0.5}
\definecolor{BrickRed}{RGB}{203,65,84}

\newcommand{\US}[1]{\mathbb S^{#1-1}} 

\newcommand\PR[1]{\ensuremath{ {\mathrm{P}}\!\left[#1\right]}}

\newcommand\Tex{}
\newcommand\EX[2][\Tex]{
\ifthenelse{\equal{#1}{}}{{\mathbb E}\!\left[#2\right]}{\ensuremath{{\mathbb E}_{#1}\left[ #2\right]}}}

\newcommand\Var[2][\Tex]{
\ifthenelse{\equal{#1}{}}{{\mathrm{Var} }[#2]}{\ensuremath{\mathrm{Var}_{#1}\left[ #2\right]}}}

\newcommand\ignore[1]{}

\newcommand\defeq{\coloneqq}

\newcommand{\reals}{\mathbb R} 
 
\newcommand\comp[1]{ \overline{#1}}

\newcommand{\pinv}[1]{  {#1}^{ \dagger } } 
\newcommand{\inv}[1]{  {#1}^{ -1 } } 
\newcommand{\maxsingv}[1]{\sigma_{\max} (#1) }
\newcommand{\minsingv}[1]{\sigma_{\min} (#1) }
\newcommand{\maxsingvb}[1]{\sigma_{\max}\! \left( #1 \right) }
\newcommand{\minsingvb}[1]{\sigma_{\min}\! \left( #1 \right) }

\newcommand{\herm}[1]{{#1}^T} 
\newcommand{\transp}[1]{{#1}^T} 

\newcommand{\range}{\mc R}

\renewcommand{\d}{d} 

\newcommand{\aff}{\mathrm{aff}}

\newcommand{\cS}{S} 
\renewcommand{\d}{d} 

\newcommand{\va}{\vect{a}}  
\newcommand{\vb}{\vect{b}}

\newcommand{\vg}{\vect{g}}

\newcommand{\vq}{\vect{q}}
\newcommand{\vr}{\vect{r}}

\newcommand{\vv}{\vect{v}}  

\newcommand{\vx}{\vect{x}}  
\newcommand{\vy}{\vect{y}}  
\newcommand{\vz}{\vect{z}}

\newcommand{\mA}{\vect{A}}  
\newcommand{\mB}{\vect{B}} 
\newcommand{\mC}{\vect{C}} 
\newcommand{\mD}{\vect{D}}
\newcommand{\mE}{\vect{E}}

\newcommand{\mG}{\vect{G}}

\newcommand{\mI}{\vect{I}}

\newcommand{\mL}{\vect{L}}

\newcommand{\mP}{\vect{P}}

\newcommand{\mU}{\vect{U}}
\newcommand{\mV}{\vect{V}}
\newcommand{\mW}{\vect{W}}
\newcommand{\mX}{\vect{X}}
\newcommand{\mY}{\vect{Y}}
\newcommand{\mZ}{\vect{Z}}

\renewcommand{\S}{\mathcal T}

\newcommand{\q}{q} 
\newcommand{\s}{s}

\renewcommand{\l}{{\ell}} 

\newcommand{\subsind}[2]{{#1}^{( #2)}}
\newcommand{\Ul}{\subsind{\mU}{\l}}
\newcommand{\Ult}{\subsind{\tilde \mU}{\l}}
\newcommand{\Al}{\subsind{\mA}{\l}}
\newcommand{\Ak}{\subsind{\mA}{k}}
\newcommand{\Atl}{\subsind{\tilde \mA}{\l}}
\newcommand{\Xl}{\subsind{\mX}{\l}}
\newcommand{\Yl}{\subsind{\mY}{\l}}
\newcommand{\Zl}{\subsind{\mZ}{\l}}
\newcommand{\al}{\subsind{\va}{\l}}
\newcommand{\atl}{\subsind{\tilde \va}{\l}}
\newcommand{\xl}{\subsind{\vx}{\l}}
\newcommand{\yl}{\subsind{\vy}{\l}}
\newcommand{\zl}{\subsind{\vz}{\l}}

\newcommand{\Uk}{\subsind{\mU}{k}}

\newcommand{\xk}{\subsind{\vx}{k}}
\newcommand{\yk}{\subsind{\vy}{k}}
\newcommand{\zk}{\subsind{\vz}{k}}

\newcommand{\Pp}{\mP_\lVert}
\newcommand{\Po}{\mP_\perp}
\newcommand{\pp}{\lVert}
\newcommand{\po}{\perp}

\newcommand{\itr}{s}
\newcommand{\send}{{\itr_\mathrm{a}}}
\newcommand{\rps}{\vr_{\itr \pp}}
\newcommand{\ros}{\vr_{\itr \po}}
\newcommand{\rpsl}{\subsind{\vr}{\l}_{\itr \pp}}
\newcommand{\rosl}{\subsind{\vr}{\l}_{\itr \po}}
\newcommand{\rs}{\vr_\itr}
\newcommand{\rspr}{\vr_{\itr'}}
\newcommand{\rsl}{\subsind{\vr}{\l}_\itr}

\newcommand{\smaxparlb}{{\bar s}}

\newcommand{\qpsl}{\subsind{\vq}{\l}_{\itr \pp}}
\newcommand{\qosl}{\subsind{\vq}{\l}_{\itr \po}}
\newcommand{\qps}{\vq_{\itr \pp}}

\newcommand{\qs}{\vq_\itr}
\newcommand{\qsl}{\subsind{\vq}{\l}_\itr}
\newcommand{\qspr}{\vq_{\itr'}}

\newcommand{\qsprlm}{\subsind{\vq}{\l}_{\itr'-1}}
\renewcommand{\q}[1]{\vq_{#1}}

\newcommand{\cY}{\mc Y}
\renewcommand{\cS}{\mc S}

\newcommand{\abs}[1]{\left\lvert #1 \right\rvert}
\renewcommand{\innerprod}[2]{\left\langle #1, #2 \right\rangle}
\newcommand{\prob}[1]{\mathrm{P} \! \left[ #1 \right]}

\newcommand{\bsy}{\boldsymbol}

\newcommand{\event}[2]{\mc E_{#1}^{#2}}
\newcommand{\Ea}{\event{1}{(\l,i,\itr)}}
\newcommand{\Eb}{\event{2}{(\l,i,\itr)}}
\newcommand{\Ec}{\event{3}{(\l,i,\itr)}}
\newcommand{\Ed}{\event{4}{}}
\newcommand{\Ee}{\event{5}{(\l,i)}}

\newcommand{\Ef}{\event{6}{(\l,i)}}
\newcommand{\Eft}{{\tilde  {\mc E}}_{6}^{(\l,i)}}
\newcommand{\Eg}{\event{7}{(\l,i)}}
\newcommand{\Egt}{{\tilde  {\mc E}}_{7}^{(\l,i)}}
\newcommand{\Estar}{\mc E^\star}
\newcommand{\evcomp}[2]{\comp{\mc E}_{#1}^{(#2)}}

\newcommand{\Fg}{\mc F_7^{(\l,i)}}

\newcommand{\mpspar}{p_{\max}}

\newcommand{\tp}{\#\textsf{TP}}
\newcommand{\fp}{\#\textsf{FP}}

\newcommand{\tpr}{\textsf{TPR}}
\newcommand{\fpr}{\textsf{FPR}}

\allowdisplaybreaks

\makeatletter
\newcommand{\vast}{\bBigg@{3}}
\newcommand{\Vast}{\bBigg@{4}}
\makeatother

\begin{document}

\title{Noisy Subspace Clustering via Matching Pursuits}

\author{Michael Tschannen and Helmut B{\"o}lcskei \thanks{The authors are with the Department of Information Technology and Electrical Engineering, ETH Zurich, Switzerland (e-mail: michaelt@nari.ee.ethz.ch; boelcskei@nari.ee.ethz.ch).}}

\maketitle

\vspace{-1cm}
\begin{abstract}
Sparsity-based subspace clustering algorithms have attracted significant attention thanks to their excellent performance in practical applications.
A prominent example is the \ac{ssc} algorithm by Elhamifar and Vidal, which performs spectral clustering based on an adjacency matrix obtained by sparsely representing each data point in terms of all the other data points via the Lasso. 
When the number of data points is large or the dimension of the ambient space is high, the computational complexity of \ac{ssc} quickly becomes prohibitive. Dyer et al. observed that \ac{ssc}-\acs{omp} obtained by replacing the Lasso 
by the greedy \ac{omp} algorithm results in significantly lower computational complexity, while often yielding comparable performance. 
The central goal of this paper is an analytical performance characterization of \ac{ssc}-\ac{omp} for noisy data. Moreover, we introduce and analyze the \ac{ssc}-\acs{mp} algorithm, which employs \ac{mp} in lieu of \ac{omp}.
Both \ac{ssc}-\ac{omp} and \ac{ssc}-\ac{mp} are proven to succeed even when the subspaces intersect and when the data points are contaminated by severe noise. The clustering conditions we obtain for \ac{ssc}-\ac{omp} and \ac{ssc}-\ac{mp} are similar to those for \ac{ssc} and for the \ac{tsc} algorithm due to Heckel and B{\"o}lcskei. 
Analytical results in combination with numerical results indicate that both \ac{ssc}-\ac{omp} and \ac{ssc}-\ac{mp} with a data-dependent stopping criterion automatically detect the dimensions of the subspaces underlying the data. Experiments on synthetic and on real data show that \ac{ssc}-\ac{mp} often matches or exceeds the performance of the computationally more expensive \ac{ssc}-\ac{omp} algorithm. Moreover, \ac{ssc}-\ac{mp} compares very favorably to \ac{ssc}, \ac{tsc}, and the \acl{nsn} algorithm, both in terms of clustering performance and running time. In addition, we find that, in contrast to \ac{ssc}-\ac{omp}, the performance of \ac{ssc}-\ac{mp} is very robust with respect to the choice of parameters in the stopping criteria. 
\end{abstract}
\vspace{-0.25cm}

\section{Introduction}
\glsresetall[\acronymtype]

Extracting structural information from large high-dimensional data sets in a computationally efficient manner is a major challenge in many modern 
machine learning tasks. 
A structure widely encountered in practical applications is that of unions of (low-dimensional) subspaces. 
The problem of extracting the assignments of the data points in a given data set to the subspaces without prior knowledge of the number of subspaces, their orientations and dimensions 
is referred to as subspace clustering and has found applications in, e.g., image representation and segmentation \cite{hong_multiscale_2006}, face clustering \cite{ho_clustering_2003}, motion segmentation \cite{costeira1998multibody}, system identification \cite{vidal2003algebraic}, and genomic inference \cite{jiang2004cluster}. 
More formally, given a set $\cY = \cY_1 \cup \ldots \cup \cY_L$ of $N$ data points in $\reals^m$, where the points in $\cY_\l$ lie in or near the $d_\l$-dimensional linear subspace $\cS_\l \subset \reals^m$, we want to find the association of the points in $\cY$ to the $\cY_\l$, without prior knowledge on the $\cS_\l$.

The subspace clustering problem has been studied for more than two decades with a correspondingly sizeable body of literature. 
The algorithms available to date can roughly be categorized as algebraic, statistical, and spectral clustering-based; we refer to \cite{vidal_subspace_2011} for a review of the most prominent representatives of each class. While many subspace clustering algorithms exhibit good performance in practice, corresponding analytical results under non-restrictive conditions on the relative orientations of the subspaces are available only for a small set of algorithms. Specifically, during the past few years a number of new  
algorithms, which rely on sparse representations (of each data point in terms of all the other data points) followed by spectral clustering \cite{luxburg_tutorial_2007}, were proposed and mathematically analyzed \cite{elhamifar_sparse_2013, soltanolkotabi2012geometric, soltanolkotabi2014robust, wang2013noisy, heckel_robust_2013, dyer_greedy_2013, park2014greedy, you_sparse_2015}. These algorithms exhibit good empirical performance and succeed provably under quite generous conditions on the relative orientations of the subspaces. 
Almost all analytical performance results available to date apply, however, to the noiseless case, where the data points lie exactly in the union of the $\cS_\l$. 
A notable exception is the \ac{ssc} algorithm by Elhamifar and Vidal \cite{elhamifar_sparse_2013}, which was shown by Soltanolkotabi et al. \cite{soltanolkotabi2014robust} and Wang and Xu \cite{wang2013noisy} to succeed for noisy data even when the subspaces intersect. \ac{ssc} employs the Lasso\footnote{We note that the \ac{ssc} formulation in \cite{elhamifar_sparse_2013} adds a term to the Lasso objective function to account for sparse corruptions of the data points. 
The performance guarantees in \cite{soltanolkotabi2014robust, wang2013noisy} apply, however, to the ``pure'' Lasso version of \ac{ssc}. Throughout this paper, unless explicitly stated otherwise, \ac{ssc} will refer to the ``pure'' Lasso version.\label{fn:lassossc}} (or $\ell_1$-minimization in the noiseless case) to find a sparse 
representation (or, more precisely, approximation) of each data point in terms of all the other data points, then constructs an affinity graph based on the so-obtained sparse representations, and finally determines subspace assignments through spectral clustering of the affinity graph. 
To understand the intuition behind this approach, first note that in the noiseless case every data point $\vy_j $ in $\cS_\l$ can be represented by (at most $d_\l$) other data points in $\cS_\l$ provided that the points in $\cY_\l$ are non-degenerate.
In the noisy case, the hope is now that the sparse representation of $\vy_j \in \cY_\l$ in terms of $\cY \backslash \{ \vy_j \}$ delivered by \ac{ssc} 
involves mostly points belonging to $\cY_\l$ thanks to the sparsity-promoting nature of the Lasso.  
Of course, 
this will happen only if 
the subspaces $\cS_\l$ underlying the $\cY_\l$ are sufficiently far apart. 
The analytical performance results in \cite{soltanolkotabi2012geometric, soltanolkotabi2014robust, wang2013noisy} quantify the impact of subspace dimensions and relative orientations, noise variance, and the number of data points on the performance of \ac{ssc}.

When the data is high-dimensional or the number of data points is large, solving the $N$ Lasso problems (each in $N-1$ variables) in \ac{ssc} can be computationally challenging. Greedy 
algorithms for computing sparse representations of the data points (in terms of all the other data points) are therefore an interesting alternative. Three such alternatives were proposed in the literature, namely the \ac{ssc}-\acl{omp} \glsunset{omp}(\ac{ssc}-\ac{omp}) algorithm by Dyer et al. \cite{dyer_greedy_2013}, the \ac{tsc} algorithm by Heckel and B\"olcskei \cite{heckel_robust_2013}, and the \ac{nsn} algorithm by Park et al. \cite{park2014greedy}. \ac{ssc}-\ac{omp} employs \ac{omp} instead of the Lasso to compute sparse representations of the data points. \ac{tsc} relies on the nearest neighbors---in spherical distance---of each data point to construct the affinity graph, and \ac{nsn} greedily assigns to each data point a subset of the other data points by iteratively selecting the data point closest (in  Euclidean distance) to the subspace spanned by the previously selected data points.

To the best of our knowledge, besides \ac{ssc}, \ac{tsc} is the only subspace clustering algorithm that was proven to succeed under noise. The performance guarantees available for \ac{ssc}-\ac{omp} \cite{dyer_greedy_2013, you_sparse_2015, heckel2015dimensionality} all apply to the noiseless case.

\paragraph*{Contributions} The main contributions of this paper are an analytical performance characterization of \ac{ssc}-\ac{omp} in the noisy case, 
and of a new algorithm, termed \ac{ssc}-\acl{mp} \glsunset{mp}(\ac{ssc}-\ac{mp}), 
which is obtained by replacing \ac{omp} in \ac{ssc}-\ac{omp} by the \ac{mp} algorithm \cite{friedman1981projection, mallat1993matching}. 
Matching pursuit algorithms per se have been studied extensively in the sparse signal representation literature \cite{blumensath2012greedy} and the approximation theory literature \cite{temlyakov2003nonlinear}. Replacing \ac{omp} by \ac{mp} 
is attractive as the per-iteration complexity of \ac{mp} 
is smaller than that of \ac{omp}  
thanks to the absence of the orthogonalization step. 
On the other hand, the representation error (in $\ell_2$-norm) of \ac{mp} may decay slower---as a function of the number of iterations---than that of \ac{omp} \cite{temlyakov2003nonlinear}. We shall see, however, that in the context of subspace clustering, 
in practice, the lower per-iteration cost of \ac{mp} usually translates into lower overall running time, while delivering essentially the same clustering performance as \ac{omp}.

Our main results are sufficient conditions for \ac{ssc}-\ac{omp} and \ac{ssc}-\ac{mp} to succeed in terms of the \acrlong{nfc} property (see Definition \ref{def:nfc}), 
a widely used \cite{soltanolkotabi2012geometric, soltanolkotabi2014robust, dyer_greedy_2013, heckel_robust_2013, wang2013noisy, heckel2015dimensionality, dyer_greedy_2013, you_sparse_2015, park2014greedy} subspace clustering performance measure. 
Specifically, we find 
that both algorithms succeed even when the subspaces intersect and when the signal to noise ratio is as low as $0$dB. 
Furthermore, the sufficient conditions we obtain point at an intuitively appealing tradeoff between the affinity of the subspaces (a similarity measure for pairs of subspaces defined later), the noise variance, and the number of points in the data set corresponding to each subspace. This ``clustering condition'' is structurally similar to those for \ac{ssc} in \cite[Thm. 3.1]{soltanolkotabi2014robust}, \revb{\cite[Thm. 10]{wang2013noisy}} and for \ac{tsc} in \cite[Thm.~3]{heckel_robust_2013}. Moreover, numerical results indicate that our clustering condition is order-wise optimal. 
The main technical challenge in proving our results stems from the need to handle statistical dependencies between quantities computed in different iterations of the \ac{omp} and \ac{mp} algorithms. 

\ac{omp} and \ac{mp} are commonly stopped either after a prescribed maximum number of iterations, which we henceforth call \ac{di}-stopping, or when the representation error 
falls below a threshold value, referred to as \ac{dd}-stopping. 
For a given data point to be represented, \ac{omp} is guaranteed to select a new data point in every iteration and the sparsity level of the resulting representation therefore equals the number of \ac{omp} iterations performed. \ac{mp}, on the other hand, may select individual data points to participate repeatedly in the sparse representation of a given data point. The sparsity level of the representation computed by \ac{mp} may therefore be smaller than the number of iterations performed. 
As it is important for subspace clustering purposes to be able to control the sparsity level, we propose a new hybrid stopping criterion for \ac{mp} terminating the algorithm either when a given maximum number of iterations was performed or when a given target sparsity level is attained. 
We consider \ac{ssc}-\ac{omp} and \ac{ssc}-\ac{mp} both with \ac{di}- and \ac{dd}-stopping. 
For \ac{di}-stopping, we present numerical results which indicate that performing (order-wise) more than $d_\l$ \ac{omp} iterations can severely compromise the performance of \ac{ssc}-\ac{omp}. \ac{ssc}-\ac{omp} with \ac{di}-stopping therefore requires fairly accurate knowledge of the subspace dimensions. 
\ac{ssc}-\ac{mp}, on the other hand, exhibits a much more robust behavior in this regard. 
For \ac{dd}-stopping, we prove that taking the threshold value on the representation error to be linear in the noise standard deviation ensures that both \ac{omp} and \ac{mp} select order-wise at least  
$d_\l$ points from $\cY_\l \backslash \{ \vy_j \}$ to represent $\vy_j \in \cY_\l$, provided that the noise variance is sufficiently small.  
Numerical results further indicate that both algorithms, indeed, select order-wise no more than $d_\l$ points from $\cY_\l \backslash \{ \vy_j \}$ and essentially no points from $\cY \backslash \cY_\l$. 
This means that \ac{ssc}-\ac{omp} and \ac{ssc}-\ac{mp} with \ac{dd}-stopping implicitly estimate (again order-wise) the subspace dimensions $d_\l$.  
This can---in principle---also be accomplished by \ac{ssc} with a selection procedure for the Lasso parameter that is based on solving an auxiliary (constrained Lasso) optimization problem for each data point \cite{soltanolkotabi2014robust}. This procedure imposes, however, significant computational burden; in contrast \ac{dd}-stopping as performed here comes at essentially zero computational cost.

Finally, we present extensive numerical results comparing the performance of \ac{ssc}-\ac{omp}, \ac{ssc}-\ac{mp}, \ac{ssc}, \ac{tsc}, and \ac{nsn} for synthetic and real data. In particular, we find that \ac{ssc}-\ac{mp} outperforms \ac{ssc} in the reference problem of face clustering on the Extended Yale B data set \cite{georghiades_illumination_2001,lee_acquiring_2005} and does so at drastically lower running time.

\paragraph*{Notation} We use lowercase boldface letters to denote (column) vectors and uppercase boldface letters to designate matrices. The superscript $\herm{}$ stands for transposition. For the vector $\vv$, $[\vv]_i$ denotes its $i$th element, $\norm[0]{\vv}$ is the number of non-zero entries, and $\norm[\infty]{\vv} \defeq \max_i \abs{[\vv]_i}$. For the matrix $\mA$, $\mA_{-i}$ stands for the matrix obtained by removing the $i$th column from $\mA$, $\mA_\S$ is the submatrix of $\mA$ consisting of the columns with index in the set $\S$, $\range(\mA)$ is its range space, $\norm[2\to 2]{\mA} \defeq\;$ $\max_{\norm[2]{\vv} = 1  } \norm[2]{\mA \vv}$ its spectral norm, $\norm[F]{\mA} \defeq (\sum_{i,j} |\mA_{ij}|^2 )^{1/2}$ its Frobenius norm, and $\minsingv{\mA}$ and $\maxsingv{\mA}$ refer to its minimum and maximum singular value, respectively. 
For a matrix $\mA \in \reals^{m \times n}$, $m \geq n$, of full column rank, we denote its pseudoinverse by $\pinv{\mA} \defeq \inv{(\transp{\mA} \mA)} \transp{\mA}$. The identity matrix is $\mI$. $\mathcal N( \boldsymbol{\mu},\boldsymbol{\Sigma})$ stands for the distribution of a Gaussian random vector with mean $\boldsymbol{\mu}$ and covariance matrix $\boldsymbol{\Sigma}$. 
The expectation of the random variable $X$ is written as $\mb E [X]$. For random variables $X$ and $Y$, we indicate their equivalence in distribution by $X \sim Y$. 
The set $\{1, \ldots,N\}$ is denoted by $[N]$.  The cardinality of the set $\S$ is $|\S|$ and its complement is $\comp{\S}$. The unit sphere in $\reals^m$ is $\US{m} \defeq \{ \vx \in \reals^m \colon \norm[2]{\vx} = 1 \}$. 
$\log(\cdot)$ refers to the natural logarithm.

We say that a subgraph $H$ of a graph $G$ is connected if every pair of nodes in $H$ can be joined by a path with nodes exclusively in $H$. A connected subgraph $H$ of $G$ is called a connected component of $G$ if there are no edges between $H$ and the remaining nodes in $G$.

\section{Subspace clustering via matching pursuits} \label{sec:algos}

\ac{omp} and \ac{mp} per se were introduced in \cite{chen1989orthogonal} and \cite{friedman1981projection}, respectively, and have been studied extensively in the  
sparse signal representation literature,  
see, e.g., \cite{blumensath2012greedy}, \cite[Chap.~3]{foucart_mathematical_2013}. 
In the context of subspace clustering the premise is that the (sparse) representations of each data point in terms of all the other data points delivered by \ac{omp} and \ac{mp} contain predominantly data points that lie in the same subspace as the data point under consideration. 
We refer to \cite{elhamifar_sparse_2013}, \cite[Sec.~2.B]{heckel_robust_2013} for a detailed discussion on the relation between sparse signal representation theory and subspace clustering. 

\subsection{The algorithms}

We first briefly review the \ac{ssc}-\ac{omp} algorithm, introduced in \cite{dyer_greedy_2013}, and then present the novel \ac{ssc}-\ac{mp} algorithm. The ensuing formulations of \ac{ssc}-\ac{omp} and \ac{ssc}-\ac{mp} assume that the data points are of comparable $\ell_2$-norm. This assumption is relevant in Step 1 in both algorithms, but is not restrictive as the data points can always be normalized prior to processing. Further, an estimate $\hat L$ of the number of subspaces $L$ is assumed to be available. The estimation of $L$ from the data set under consideration is discussed below. 

\vspace{0.25cm}
{\bf The \ac{ssc}-\ac{omp} algorithm \cite{dyer_greedy_2013}.} 
Given a set of $N$ data points $\cY$ in $\reals^m$, an estimate of the number of subspaces $\hat L$, and
\begin{itemize}
\item a maximum number of iterations $s_{\max} \leq \min\{m, N-1\}$ for \ac{di}-stopping,
\item a threshold $\tau$ on the representation error 
for \ac{dd}-stopping,
\end{itemize}
perform the following steps: 

{\it Step 1:} For every $\vy_j \in \cY$, find a representation of $\vy_j$ in terms of $\cY \backslash \{\vy_j\}$ using OMP as follows: Initialize the iteration counter $\itr = 0$, the residual $\vr_0 = \vy_j$, and the set of selected indices $\Lambda_0 = \emptyset$. Denote the data matrix containing the points in $\cY$ by $\mY \in \reals^{m \times N}$. For $\itr = 1, 2, \dots$, perform the updates
\begin{align}
\lambda_\itr &= \underset{i \in [N] \backslash (\Lambda_{\itr-1}\cup \{j\})}{\arg \max}   \left| \innerprod{\vy_i}{\vr_{s-1}} \right| \label{eq:OMPSelRule} \\
\Lambda_\itr &= \Lambda_{\itr-1} \cup \lambda_s \nonumber \\
\rs &= \left(\mI - \mY_{\Lambda_\itr} \pinv{(\mY_{\Lambda_\itr})} \right) \vy_j\label{eq:OMPResFormula} \\ 
&= \left(\mI -  \frac{\tilde \vy_{\lambda_\itr}\transp{(\tilde \vy_{\lambda_\itr})}}{\norm[2]{\tilde \vy_{\lambda_\itr}}^2} \right) \vr_{\itr-1} ,\nonumber
\end{align}
where $\tilde \vy_{\lambda_\itr} = (\mI - \mY_{\Lambda_{\itr-1}} \pinv{(\mY_{\Lambda_{\itr-1}})} ) \vy_{\lambda_\itr}$, until at least one of the following criteria is met
\begin{itemize}
\item for \ac{di}-stopping: \\
$\itr = \itr_{\max}$, $\max_{i \in [N] \backslash (\Lambda_{\itr}\cup \{j\})}   \left| \innerprod{\vy_i}{\vr_\itr} \right| = 0$.\footnote{Throughout the paper, we use the convention of maximization over the empty set evaluating to $0$.}
\item for \ac{dd}-stopping: \\
$\norm[2]{\rs} \leq \tau$, $\max_{i \in [N] \backslash (\Lambda_{\itr}\cup \{j\})}   \left| \innerprod{\vy_i}{\vr_\itr} \right| = 0$.
\end{itemize}
Ties in the maximization \eqref{eq:OMPSelRule} are broken arbitrarily.

{\it Step 2:} With the number of OMP iterations actually performed denoted by $\send$, compute the representation coefficient vectors $\vb_j \in \reals^N$, $j \in [N]$, according to $(\vb_j)_{\Lambda_{\send}} = \pinv{(\mY_{\Lambda_{\send}})} \vy_j$, $(\vb_j)_{\comp{\Lambda}_{\send}} = \vect{0}$, and construct the adjacency matrix $\mA = \mB + \herm{\mB}$, where $\mB = \mathrm{abs} ([\vb_1\, \dots \,\vb_N])$ with $\mathrm{abs}(\cdot)$ denoting absolute values taken element-wise.

{\it Step 3:} Apply normalized spectral clustering \cite{ng_spectral_2001,luxburg_tutorial_2007} to $(\mA, \hat L)$. 
\vspace{0.25cm}

{\bf The \ac{ssc}-\ac{mp} algorithm.} 
Given a set of $N$ data points $\cY$ in $\reals^m$, an estimate of the number of subspaces $\hat L$, and
\begin{itemize}
\item a maximum number of iterations $s_{\max}$ and a target sparsity level $\mpspar$ for \ac{di}-stopping,
\item  a threshold $\tau$ on the representation error  
for \ac{dd}-stopping,
\end{itemize}
perform the following steps: 

{\it Step 1:} For every $\vy_j \in \cY$, find a representation of $\vy_j$ in terms of $\cY \backslash \{\vy_j\}$ using MP as follows: Initialize the iteration counter $\itr = 0$, the residual $\vq_0 = \vy_j$, and the coefficient vector $\vb_j \in \reals^N$ as $\vb_j = \mathbf 0$. For $\itr = 1, 2, \dots$, perform the updates
\begin{align}
\omega_\itr &= \underset{i \in [N]\backslash \{j\}}{\arg \max}   \left| \innerprod{\vy_i}{\vq_{\itr-1}} \right| \label{eq:MPSelRule} \\
[\vb_j]_{\omega_\itr} &\gets [\vb_j]_{\omega_\itr} + \frac{\innerprod{\vy_{\omega_\itr}}{\vq_{\itr-1}}}{\norm[2]{\vy_{\omega_\itr}}^2} \label{eq:MPcoeffcomp} \\
\qs &= \left(\mI - \frac{\vy_{\omega_\itr} \transp{(\vy_{\omega_\itr})}}{\norm[2]{\vy_{\omega_\itr}}^2} \right) \vq_{\itr-1}
\label{eq:MPresorth}
\end{align}
until at least one of the following criteria is met
\begin{itemize}
\item for \ac{di}-stopping: \\
$\itr = \itr_{\max}$, $\norm[0]{\vb_j} = \mpspar$, $\max_{i \in [N] \backslash \{j\}} \left| \innerprod{\vy_i}{\vq_{\itr}} \right| = 0$.
\item for \ac{dd}-stopping: \\
$\norm[2]{\vq_{\itr}} \leq \tau$, $\max_{i \in [N] \backslash \{j\}}   \left| \innerprod{\vy_i}{\vq_{\itr}} \right| = 0$. 
\end{itemize}
Ties in the maximization \eqref{eq:MPSelRule} are broken arbitrarily.

{\it Step 2:} Construct the adjacency matrix $\mA = \mB + \herm{\mB}$, where $\mB = \mathrm{abs} ([\vb_1\, \dots \,\vb_N])$. 

{\it Step 3:} Apply normalized spectral clustering \cite{ng_spectral_2001,luxburg_tutorial_2007} to $(\mA, \hat L)$.  
\vspace{0.25cm}

\paragraph*{Stopping criteria} We emphasize that \ac{omp}, thanks to the orthogonalization \eqref{eq:OMPResFormula} of the residual $\vr_{\itr-1}$ w.r.t. all data points selected previously,  
is guaranteed to select a new data point in every iteration and hence the sparsity level of $\vb_j$ equals the number of \ac{omp} iterations performed. In contrast, \ac{mp} orthogonalizes (see \eqref{eq:MPresorth}) the residual $\vq_{\itr-1}$ w.r.t. the data point $\vy_{\omega_\itr}$ selected in the current iteration $\itr$ only and may therefore select the same data point to participate repeatedly in the representation of $\vy_j$. The sparsity level of $\vb_j$ may hence be smaller than the number of \ac{mp} iterations performed, 
which is why the \ac{di}-stopping criterion for \ac{mp} incorporates termination when a given target sparsity level, namely $\mpspar$, is attained. Choosing $\itr_{\max}$ large enough, stopping will, indeed, be activated by $\norm[0]{\vb_j} = \mpspar$. 
Having control over the sparsity level of the coefficient vectors $\vb_j$ can be important to achieve good clustering performance as discussed below. Setting $\mpspar = N$, on the other hand, guarantees that stopping is activated through $\itr = \itr_{\max}$ and thereby allows to control the maximum number of \ac{mp} iterations through choice of $\itr_{\max}$. This hybrid stopping criterion does not seem to have been considered before in the literature.

For \ac{dd}-stopping, 
\ac{omp} is guaranteed to 
stop as soon as a basis for the subspace $\vy_j$ lies in has been found or, in case $\vy_j$ does not lie in the span of $\cY \backslash \{\vy_j\}$, the best representation---in the least-squares sense---of $\vy_j$ in terms of the points in $\cY \backslash \{\vy_j\}$. On the other hand, \ac{mp} is, in general, not guaranteed to terminate after a finite number of iterations as it may fail to activate either of the conditions $\norm[2]{\vq_{\itr}} \leq \tau$ and $\max_{i \in [N] \backslash \{j\}}   \left| \innerprod{\vy_i}{\vq_{\itr}} \right| = 0$ if $\tau$ is chosen too small \cite{mallat1993matching}. For most data sets encountered in practice this is not an issue. 
It can, however, become a problem when the data set contains outliers 
that cannot be represented sparsely by the other data points.  
In such cases it is advisable to employ the \ac{di}-stopping criterion which guarantees that  
at most $\itr_{\max}$ iterations are performed. 

\paragraph*{Implementation aspects} As \ac{mp} requires the computation of inner products only, whereas \ac{omp} contains a (least-squares) orthogonalization step (typically carried out by QR decomposition or Cholesky factorization \cite{blumensath2012greedy}), the per-iteration computational cost of \ac{ssc}-\ac{mp} is lower than that of \ac{ssc}-\ac{omp}. The numerical results in Section \ref{sec:faceclustering} indicate that this typically also translates into a lower overall running time for \ac{ssc}-\ac{mp} at fixed performance.

\paragraph*{Weak selection rules} For very large data sets one can often speed up \ac{omp} and \ac{mp} by relaxing the selection rules \eqref{eq:OMPSelRule} and \eqref{eq:MPSelRule} to so-called weak selection rules \cite{blumensath2012greedy} as follows. Instead of \eqref{eq:OMPSelRule}, one determines $\lambda_\itr$ such that $\left| \innerprod{\vy_{\lambda_s}}{\vr_{s-1}} \right| \geq \alpha \max_{i \in [N] \backslash (\Lambda_{\itr-1}\cup \{j\})}  \left| \innerprod{\vy_i}{\vr_{s-1}} \right|$ and instead of \eqref{eq:MPSelRule}, one finds $\omega_\itr$ according to $\left| \innerprod{\vy_{\omega_s}}{\vq_{s-1}} \right| \geq \alpha \max_{i \in [N] \backslash \{j\}}  \left| \innerprod{\vy_i}{\vq_{s-1}} \right|$, in both cases for a fixed relaxation parameter $\alpha \in (0,1]$. These weakened selection rules can be implemented efficiently using, e.g., locality-sensitive hashing \cite{jain2011orthogonal,vitaladevuni2011efficient}. For conciseness, we shall not analyze weak selection rules here, but only note that our main results presented in Section \ref{sec:mainres} extend to weak selection rules with minor modifications.

\subsection{Parameter selection}

In both algorithms, $\hat L$ may be estimated in Step 2 based on the adjacency matrix $\mA$ using the \emph{eigengap heuristic} \cite{luxburg_tutorial_2007} (note that $L$ is needed only in Step 3), which relies on the fact that the number of zero eigenvalues of the normalized Laplacian of the graph $G$ with adjacency matrix $\mA$ 
corresponds to the number of connected components of $G$.

The spectral clustering step (Step 3 in both algorithms) recovers the oracle segmentation $\{\cY_1, \ldots, \cY_L\}$ of $\cY$ if $\hat L = L$ and if each connected component in $G$ corresponds to exactly one of the $\cY_\l$ \cite[Prop. 4; Sec. 7]{luxburg_tutorial_2007}. Choosing the parameters $\itr_{\max}$ and $\mpspar$ in the case of \ac{di}-stopping, and $\tau$ for \ac{dd}-stopping, appropriately is therefore crucial. Indeed, as $\itr_{\max}$, $\mpspar$, and $\tau$ determine the sparsity level of the representation coefficient vectors $\vb_j$, they control the number of edges in $G$ and hence the connectivity properties of $G$. Specifically, taking $\itr_{\max}$, $\mpspar$ too small or $\tau$ too large results in 
weak connectivity and may hence cause the subgraphs of $G$ corresponding to individual $\cY_\l$ to split up into multiple connected components.  
Spectral clustering would then assign these components to different clusters, which means that a given set $\cY_\l$ is divided up into multiple (disjoint) sets. 
On the other hand, choosing $\itr_{\max}$, $\mpspar$ too big or $\tau$ too small will result in 
strong connectivity and hence potentially in ``false connections'', i.e., in edges in $G$ between points that correspond to different $\cY_\l$. Spectral clustering exhibits, however, a certain amount of robustness vis-\`a-vis false connections with small corresponding weights in $G$. 
From what was just said it follows that ideally $\itr_{\max}$, $\mpspar$, and $\tau$ should be chosen such that the sparsity levels of the $\vb_j$ are on the order of the subspace dimensions. This can be seen as follows. 
Assume that the points in $\cY_\l$ are well spread out on the subspace $\cS_\l$, $\l \in [L]$, and perturbed by additive isotropic Gaussian noise. 
If the noise variance is not too large the noisy data points $\vy_j \in \cY_\l$ will remain close to $\cS_\l$. In this case, roughly $d_\l$ points from $\cY_\l \backslash \{\vy_j\}$ will suffice to represent $\vy_j \in \cY_\l$ with small representation error. 
Hence, if the subspaces $\cS_\l$ are sufficiently far apart, imposing a sparsity level of $\approx d_\l$ for $\vy_j \in \cY_\l$ through suitable choice of $\itr_{\max}$, $\mpspar$, $\tau$ will force \ac{omp} and \ac{mp} to select points predominantly from $\cY_\l \backslash \{\vy_j\}$. If, however, the subspaces are too close to each other, \ac{omp} and \ac{mp} are likely to select points from $\cY \backslash \cY_\l$ (i.e., false connections) as well. 

We next discuss the selection of the parameters $\itr_{\max}$, $\mpspar$ for \ac{di}- and $\tau$ for \ac{dd}-stopping in the light of what was just said. 
For simplicity of exposition, in the noisy case we assume that the points in $\cY$ are in general position, i.e., every subset of $m$ or fewer points in $\cY$ is linearly independent. Real-world data sets usually have this property and the statistical data model our analysis is based on (see Section \ref{sec:mainres}) conforms with this assumption as well. Note that if the points in $\cY$ are in general position, \ac{omp} under \ac{di}-stopping is guaranteed to perform exactly $\itr_{\max}$ iterations, whereas \ac{mp} will perform at least $\min\{\itr_{\max},\mpspar, m, N-1\}$ iterations under \ac{di}-stopping. 

\paragraph*{\ac{di}-stopping}
We first consider \ac{ssc}-\ac{omp}. In the noiseless case, if the subspaces are sufficiently far apart, 
\ac{omp} under \ac{di}-stopping automatically stops after at most $d_\l$ iterations for all $\vy_j \in \cY_\l$, $\l \in [L]$ \rev{\cite{dyer_greedy_2013, you_sparse_2015}}. Therefore, choosing $\itr_{\max} \geq \max_{\l \in [L]} d_\l$ guarantees that \ac{omp} automatically detects the dimensions of the subspaces the points $\vy_j$ reside in. 
In contrast, in the noisy case, performing more than $\approx d_\l$ iterations ``forces'' \ac{omp} to select points from $\cY \backslash  \cY_\l$ (corresponding to false connections) for the representation of $\vy_j \in \cY_\l$ as after $\approx d_\l$ iterations $\vr_\itr$ will roughly be orthogonal to $\cS_\l$ (see Figure \ref{fig:OMPvsMPillustr}). The choice of $\itr_{\max}$, in practice, therefore requires knowledge of the subspace dimensions. In certain applications such as, e.g., face clustering \cite{basri_lambertian_2003}, 
information on the subspace dimensions may, indeed, be available a priori. 
When the subspace dimensions $d_\l$ vary widely (across $\l$), 
there may be no $\itr_{\max}$ that ensures both connectivity of all the subgraphs of $G$ corresponding to the $\cY_\l$, and at the same time guarantees a small number of false connections. 
In principle this problem could be mitigated by setting $\itr_{\max}$ for each $\vy_j$ individually according to the dimension of the subspace it belongs to. 
This would, however, require knowledge of 
the assignments of the data points $\vy_j$ to the subspaces $\cS_\l$, thereby performing the actual subspace clustering task.

We now turn to \ac{ssc}-\ac{mp}. \ac{mp} with $\mpspar = N$, i.e., the hybrid stopping criterion is activated once $\itr_{\max}$ iterations have been performed, 
exhibits a stopping behavior that is fundamentally different from that of 
\ac{omp}. As already mentioned this is a consequence of \ac{mp} orthogonalizing the residual only w.r.t. the data point selected in the current iteration, thereby allowing for repeated selection of individual data points. In the noiseless case, unlike \ac{ssc}-\ac{omp}, \ac{ssc}-\ac{mp} can select more than $d_\l$ data points from $\cY_\l \backslash \{\vy_j\}$ to represent $\vy_j \in \cY_\l$, thereby potentially producing a graph $G$ with better connectivity than \ac{ssc}-\ac{omp}. \revb{On the other hand, \ac{ssc}-\ac{mp} tends to assign smaller weights to the points in $\cY_\l \backslash \{\vy_j\}$ than \ac{ssc}-\ac{omp}, which can lead to slightly inferior performance (see the experiments reported in Appendix~\ref{sec:noiseless}).} For the noisy case, the experiments on synthetic and real-world data, reported in Section \ref{sec:inflsmax}, show that in the first $d_\l$ or so iterations \ac{mp} adds predominantly new points from the correct cluster $\cY_\l \backslash \{\vy_j\}$ and thereafter tends to revisit
previously selected data points or add new points stemming mostly from $\cY_\l \backslash \{\vy_j\}$. For fixed $\itr_{\max}$, this leads to a smaller number of false connections than what would be obtained by \ac{omp}. The example in Figure \ref{fig:OMPvsMPillustr} illustrates this behavior. The data set $\cY$ is drawn from the union of a two-dimensional subspace $\cS_1$ and a one-dimensional subspace $\cS_2$, where $\cS_1$ and $\cS_2$ span a (principal) angle of approximately 50 degrees. We consider the data point $\vy_j \in \cY_1$. In the first iteration, both \ac{omp} and \ac{mp} select the same data point $\vy_{\lambda_1} = \vy_{\omega_1}$ and we have $\vr_1 = \vq_1$. In the second iteration, the two algorithms again select the same data point, namely $\vy_{\lambda_2} = \vy_{\omega_2}$, but $\vr_2$ is now almost orthogonal to $\cS_1$ whereas $\vq_2$ remains close to $\cS_1$ as \ac{mp} orthogonalizes only w.r.t. $\vy_{\omega_2}$. Specifically, the (principal) angle between $\vq_2$ and $\cS_1$ is approximately 27 degrees, while $\vq_2$ and $\cS_2$ span an angle of approximately 50 degrees. In the third iteration \ac{mp} therefore selects a point from the set $\cY_1 \backslash \{ \vy_j \}$, whereas \ac{omp} chooses a point from the wrong subspace $\cY_2$.

\newcommand{\illfsize}{\tiny}

\newcommand{\colra}{black!60!white}
\newcommand{\colrb}{black!75!white}
\newcommand{\colqb}{black!90!white}

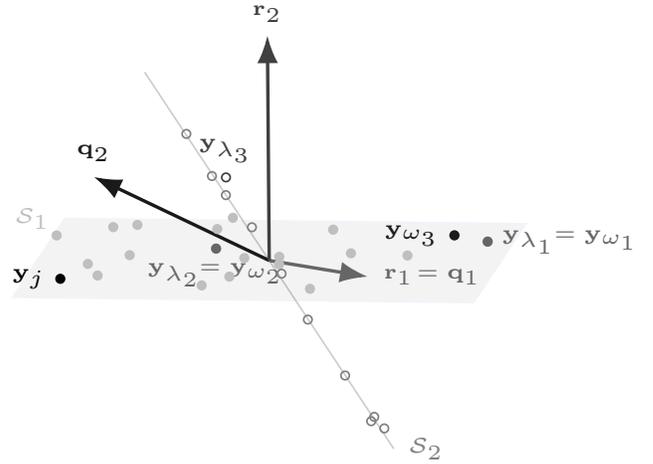
\begin{figure}[h]
\centering
 \begin{tikzpicture}[scale=1.6]
 \begin{axis}[hide axis, axis lines=center, axis line style=semithick, axis on top, samples=5, xmin=-1.1, xmax=1.5, ymin=-1, ymax=1, zmin=-1, zmax=1, xtick=\empty, ytick=\empty, ztick=\empty, view={-5}{-170}]  
 
\addplot3[surf,domain=-1:1,samples=2,opacity=0.05,color=black]{0*x+0*y} node[left,pos=0.0, opacity=1, color=lightgray]{\illfsize $\cS_1$};

 \addplot3 +[mark=none, opacity=0.2, color=black] coordinates {
	 (0.5661, 0.2265, 0.7926)
	 (-0.5661, -0.2265, -0.7926)
 } node[right,pos=0.0, color=gray, opacity=1]{\illfsize $\cS_2$};
 
 \addplot3 +[mark=*, mark size = 1, color=lightgray, only marks,mark options={solid}] table{./data/OMPMPillustration/Y1OMPMPillustration.dat};
 \addplot3 +[mark=o, mark size = 1, color=gray, only marks,mark options={solid}] table{./data/OMPMPillustration/Y2OMPMPillustration.dat};
 \addplot3 +[mark=*, mark size = 1, color=black, only marks,mark options={solid}] coordinates {
 	(-0.845392980184535, 0.534139691621176, -0.00234497111831683)
} node[left,pos=0, opacity=1]{\illfsize $\vy_j$};
\addplot3 +[mark=*, mark size = 1, color=\colra, only marks,mark options={solid}] coordinates {
 	(0.8852, -0.5476, 0.0016)
} node[right,pos=0.0, opacity=1]{\illfsize $\vy_{\lambda_1} \!\!= \vy_{\omega_1}$};
\addplot3 +[mark=*, mark size = 1, color=\colra, only marks,mark options={solid}] coordinates {
 	(-0.2687, -0.3365, 0.0092)
} node[below,pos=0.0, opacity=1]{\illfsize $\vy_{\lambda_2} \!\!= \vy_{\omega_2}$};

\addplot3 +[mark=o, mark size = 1, color=\colrb, only marks,mark options={solid}] coordinates {
 	(-0.2004, -0.1134, -0.3478)
} node[above,pos=-0.0, opacity=1]{\illfsize $\vy_{\lambda_3}$};

\addplot3 +[mark=*, mark size = 1, color=\colqb, only marks,mark options={solid}] coordinates {
 	(0.7236, -0.7013, 0.0025)
} node[left,pos=0.0, opacity=1]{\illfsize $\vy_{\omega_3}$};

\draw[Latex-,solid, thick, color=\colra] (axis cs:0.5244, 0.8474, -0.0829) -- (axis cs:0, 0, 0) node[right,pos=0.0, opacity=1]{\illfsize $\vr_1 \!=\vq_1$};
\draw[Latex-,solid, thick, color=\colqb] (axis cs:-0.6997, 0.5462, -0.4604) -- (axis cs:0, 0, 0) node[above,pos=0.0, opacity=1]{\illfsize $\vq_2$};
\draw[Latex-,solid, thick, color=\colrb] (axis cs:-0.0101, -0.0193, -0.9998) -- (axis cs:0, 0, 0) node[above,pos=0.01, opacity=1]{\illfsize $\vr_2$};
 	
   \end{axis}
 \end{tikzpicture}
 \vspace{-0.5cm}
\caption{\label{fig:OMPvsMPillustr} Evolution of the  (normalized) residuals $\vr_\itr$ and $\vq_\itr$ 
corresponding to \ac{omp} and \ac{mp}, respectively, for the data point $\vy_j \in \cY_1$. The data points belonging to $\cS_2$ are marked by circles. Note that all data points were normalized to unit $\ell_2$-norm prior to clustering.}
\end{figure}

The experiments reported in Section \ref{sec:inflsmax} reveal the following remarkable property. Even when we discount the advantage of \ac{mp}---owing to its ability to reselect data points---by forcing the sparsity levels of \ac{mp} and \ac{omp} to be equal to, say $\itr_\mathrm{high}$, through appropriate choice of $\itr_{\max}$ and $\mpspar$, with $\itr_\mathrm{high} \gg d_\l$ for at least one $\l \in [L]$, \ac{mp} still tends to select more points from $\cY_\l$ to represent $\vy_j \in \cY_\l$ than \ac{omp} does. Moreover, these numerical results indicate that \ac{mp} also tends to assign smaller (in absolute value) coefficients to false connections, i.e., to points in $\cY \backslash \cY_\l$; this can have a favorable effect on performance thanks to the robustness of spectral clustering to false connections with small associated weights.

Our analytical results in Section \ref{sec:mainres} guarantee 
that \ac{ssc}-\ac{omp} and \ac{ssc}-\ac{mp} succeed for $s_{\max}$ and $\mpspar$ linear---up to $\log$-terms---in the smallest subspace dimension, provided that the subspaces are sufficiently far apart, the noise variance is sufficiently small, and the data set contains sufficiently many points from each subspace. 

\paragraph*{\ac{dd}-stopping} We assume throughout that $\tau$ is sufficiently large for \ac{ssc}-\ac{mp} to terminate. 
For \ac{omp} in the context of sparse noisy signal recovery, taking $\tau$ linear in the noise standard deviation $\sigma$ is known to lead to correct recovery of the sparse signal under certain technical conditions \cite{cai2011orthogonal}. In the context of subspace clustering, where the problem is actually of a different nature, the analytical results in Section~\ref{sec:mainres} indicate that such a choice for $\tau$ guarantees that both \ac{ssc}-\ac{omp} and \ac{ssc}-\ac{mp} select order-wise at least $d_\l$ points from $\cY_\l$ to represent $\vy_j \in \cY_\l$, provided that the subspaces $\cS_\l$ are sufficiently far apart and the points in $\cY_\l$ are well spread out on $\cS_\l$, $\l \in [L]$, and perturbed by additive isotropic Gaussian noise of  sufficiently small variance. The numerical results in Section~\ref{sec:tprfprexp} show that the graphs $G$ generated by \ac{ssc}-\ac{omp} and \ac{ssc}-\ac{mp} will also have a small number of false connections if $\tau$, in addition, is not too small. Appropriate choice of $\tau$ therefore makes both \ac{omp} and \ac{mp} automatically adjust the sparsity level for each data point according to the dimension of the subspace the data point lies in. We say that the algorithms detect the dimensions of the subspaces the data points reside in. Recall that under \ac{di}-stopping this has to be accomplished through suitable choice of $\itr_{\max}$, $\mpspar$.

When the data set $\cY$ contains outliers that cannot be represented sparsely in terms of the other points in $\cY$, \ac{dd}-stopping usually leads to a high number of false connections as the representation error for outliers will decay slowly resulting in late activation of the \ac{dd}-stopping criterion. In this case, we need to rely on \ac{di}-stopping. 

\revb{\paragraph*{Summary} We recommend DI-stopping with $\itr_{\max}$ on the order of the subspace dimensions if the subspace dimensions are (approximately) known, and DD-stopping with $\tau^2$ on the order of the noise variance if the noise variance is known and not too large. If no prior knowledge on the subspace dimensions or noise variance is available, or if the noise variance is large, we recommend relying on DI-stopping with $\itr_{\max}$ in the range $\{5, \ldots, 10\}$. These values for $\itr_{\max}$ turned out to work well in the relevant experiments conducted in this paper, and were also used in related works \cite{dyer_greedy_2013, you_sparse_2015}.}

\section{Main results} \label{sec:mainres}

Our analytical performance results are for a statistical data model,  
also employed in \cite{soltanolkotabi2014robust, heckel_robust_2013}. 
Specifically, we take the subspaces $\cS_\l$ to be fixed 
and the points in the corresponding subsets $\cY_\l$ of the data set $\cY = \cY_1 \cup \ldots \cup \cY_L$ to be randomly distributed on $\cS_\l \cap \US{m}$ and  
perturbed by additive random noise. 
Concretely, the points in $\cY_\l$, $\l \in [L]$, are given by $\yl_i = \xl_i + \,\zl_i = \Ul \al_i + \,\zl_i$, $i \in [n_\l]$, where the columns of $\Ul \in \reals^{m \times d_\l}$ constitute an orthonormal basis for $\cS_\l$, the $\al_i$ are independently (across $i \in [n_\l]$, $\l \in [L]$) and uniformly distributed on $\US{d_\l}$, and the $\zl_i$ are i.i.d. $\mc N(\mathbf{0}, (\sigma^2/m) \mI_m)$. The factor $1/m$ in the noise covariance matrix ensures that $\norm[2]{\smash{\zl_i}}^2$ concentrates around $\EX{\smash{\norm[2]{\smash{\zl_i}}^2}} = \sigma^2$ for large $m$. Note further that the $\ell_2$-norm of the data points concentrates around $\sqrt{1 + \sigma^2}$ for large $m$ and that they are hence of comparable $\ell_2$-norm, 
as required by the formulations of \ac{ssc}-\ac{omp} and \ac{ssc}-\ac{mp} in Section~\ref{sec:algos}. Moreover, the data points are in general position w.p. $1$ for $\sigma > 0$.

Prima facie assuming the noiseless data points $\xl_i$ to be uniformly distributed on the subspaces $\cS_\l$ may appear overly stylized. 
However, for any algorithm to have a chance of 
producing correct assignments, 
we need the noiseless data points to be well spread out \rev{to a certain extent (albeit not necessarily around the origin as in our data model)} on the subspaces. To see this, suppose for example, that the points in $\cY_\l$ are concentrated on two distinct subspaces of $\cS_\l$, say $\cS_\l'$ and $\cS_\l''$. Then, one can assign the points in $\cY_\l$ either to two clusters, one containing the points concentrated on $\cS_\l'$ and the other one those concentrated on $\cS_\l''$, or one can assign all the points in $\cY_\l$ to a single cluster.

Our results will depend on the affinity between pairs of subspaces which measures how far apart two subspaces are. The affinity between the subspaces $\cS_k$ and $\cS_\l$ is defined as 
\cite[Def.~2.6]{soltanolkotabi2012geometric}, \cite[Def.~1.2]{soltanolkotabi2014robust}
\begin{equation}
\aff(\cS_k,\cS_\l) \defeq \frac{1}{\sqrt{\min \{d_k , d_\l \}}}   \norm[F]{ \herm{\mU^{(k)}} \mU^{(\l)} } \label{eq:affdef}
\end{equation}
and can equivalently be expressed in terms of the principal angles $\theta_1 \leq \ldots \leq \theta_{\min \{d_k , d_\l \}}$ between $\cS_k$ and $\cS_\l$ \cite[Sec.~6.3.4]{golub_matrix_1996} according to
\begin{equation}
\aff(\cS_k,\cS_\l) =\allowbreak \sqrt{ \frac{\cos^2( \theta_1) + \ldots + \cos^2(\theta_{\min \{d_k , d_\l \}})}{\min \{d_k , d_\l \}}}. \label{eq:princangles}
\end{equation}
We have $0 \leq \aff(\cS_k,\cS_\l) \leq 1$ and for subspaces intersecting in $t$ dimensions, we get $\cos(\theta_1)= \ldots =\cos(\theta_{t})=1$ and hence $\aff(\cS_k,\cS_\l) \geq \sqrt{t/\min\{d_k,d_\l\}}$.

Recall that spectral clustering recovers the oracle segmentation $\{\cY_1, \ldots, \cY_L\}$ if $\hat L = L$ and each connected component in $G$ corresponds to one of the $\cY_\l$. Establishing conditions that guarantee zero clustering error is inherently difficult. To the best of our knowledge the only instances of such a \rev{result for spectral clustering-based subspace clustering algorithms} are \cite[Thm. 2]{heckel_robust_2013} for \ac{tsc} in the noiseless case and a condition in \cite{wang2016graph} guaranteeing that a post-processing procedure for \ac{ssc} yields correct clustering in the noisy case. 
We will rely on the following intermediate, albeit sensible, performance measure, which has become standard in the subspace clustering literature and  was also employed in \cite{soltanolkotabi2012geometric, soltanolkotabi2014robust, dyer_greedy_2013, heckel_robust_2013, wang2013noisy, heckel2015dimensionality, dyer_greedy_2013, you_sparse_2015, park2014greedy}. 
\begin{definition}[\Acf{nfc} property] \label{def:nfc}
The graph $G$ satisfies the \ac{nfc} property if, for all $\l \in [L]$, the nodes corresponding to $\cY_\l$ are connected to other nodes corresponding to $\cY_\l$ only. 
\end{definition}
\rev{In what follows, we often say ``\ac{ssc}-\ac{omp}/\ac{ssc}-\ac{mp} satisfies the \ac{nfc} property'' instead of ``the graph $G$ generated by \ac{ssc}-\ac{omp}/\ac{ssc}-\ac{mp} satisfies the \ac{nfc} property''. } 
To guarantee perfect clustering, we would need to ensure---in addition to the \ac{nfc} property---that the subgraphs of $G$ corresponding to the $\cY_\l$ are connected. This would preclude split-ups of the subgraphs of $G$ corresponding to the individual $\cY_\l$. Sufficient conditions 
guaranteeing this property for \ac{ssc} 
were established in \cite{nasihatkon2011graph} for $m=3$ in the noiseless case.

Note that the \ac{nfc} property does not involve the parameter $\hat L$. The sufficient conditions for  \ac{ssc}-\ac{omp} and \ac{ssc}-\ac{mp} to \rev{satisfy the \ac{nfc} property} 
reported next, therefore, do not require $\hat L = L$.

Our main result for \ac{ssc}-\ac{omp} with \ac{di}-stopping is the following.

\begin{theorem}[\ac{ssc}-\ac{omp} with \ac{di}-stopping] \label{th:SSCOMP} Define the sampling density $\rho_\l \defeq (n_\l-1)/d_\l$, and let  
$d_{\max} \defeq \max_{\l \in [L]} d_\l$ and $\rho_{\min} = \min_{\l \in [L]} \rho_\l$. Assume that $m \geq 2 d_{\max}$, $\rho_{\min} \geq c_\rho$, $\sigma \leq 1/2$,  
and $\itr_{\max} \leq \min_{\l \in [L]} \{ c_s d_\l / \log( (n_\l - 1) e / \itr_{\max}) \}$, where $c_\rho$ and $c_\itr$ are numerical constants satisfying $c_\rho >1$, $0 <c_\itr \leq 1/10$. 
Then, the clustering condition
\begin{align}
&\max_{k, \l \colon k \neq \l} \aff (\cS_k,\cS_\l) + \frac{10 \sigma}{\sqrt{\log (N^3 s_{\max})}} \Bigg( \frac{\sqrt{d_{\max}}}{\sqrt{m}} c(\sigma) \nonumber \\
&\qquad\qquad\quad+ \frac{\sqrt{2}}{\sqrt{\rho_{\min}}} \left( 1 + \frac{3}{2} \sigma \right) \Bigg) \leq \frac{1}{8 \log (N^3 s_{\max})} \label{eq:ClusCondThm} 
\end{align}
with $c(\sigma) = 10 + 13 \sigma$ guarantees that the graph $G$ 
generated by \ac{ssc}-\ac{omp} under \ac{di}-stopping satisfies the \ac{nfc} property 
w.p. at least 
\begin{equation}
P^\star \defeq 1 - 6/N - 5 N e^{-c_m m} - 6 \sum_{\l \in [L]} n_\l e^{-c_d d_\l} \label{eq:pstar}
\end{equation}
for numerical constants $c_d$ and $c_m$ obeying $0< c_d \leq 1/18$, $0 < c_m \leq 1/8$.
\end{theorem}

The main result for \ac{ssc}-\ac{mp} with \ac{di}-stopping is as follows.

\begin{theorem}[\ac{ssc}-\ac{mp} with \ac{di}-stopping] \label{th:SSCMP}
Define the sampling density $\rho_\l \defeq (n_\l-1)/d_\l$, and let $d_{\max} \defeq \max_{\l \in [L]} d_\l$ and $\rho_{\min} = \min_{\l \in [L]} \rho_\l$. Assume that $m \geq 2 d_{\max}$, $\rho_{\min} \geq c_\rho$, $\sigma \leq 1/2$, 
$\itr_{\max} > 0$, and $\mpspar \leq \min_{\l \in [L]} \{c_s d_\l / \log( (n_\l - 1) e /\mpspar) \}$, where $c_\rho$ and $c_s$ are numerical constants satisfying $c_\rho >1$, $0 <c_\itr \leq 1/10$. 
Then, the clustering condition \eqref{eq:ClusCondThm} with $c(\sigma) = 22 + 29 \sigma$ 
guarantees that the graph $G$ generated by \ac{ssc}-\ac{mp} under \ac{di}-stopping satisfies the \ac{nfc} property  
w.p. at least $P^\star$ as defined in~\eqref{eq:pstar}.
\end{theorem}

The proofs of Theorems \ref{th:SSCOMP} and \ref{th:SSCMP} can be found in Appendices \ref{sec:pfsscomp} and \ref{sec:pfsscmp}, respectively.

Theorems \ref{th:SSCOMP} and \ref{th:SSCMP} essentially state that \ac{ssc}-\ac{omp} and \ac{ssc}-\ac{mp} \rev{satisfy the \ac{nfc} property}  
for $s_{\max}$ and $\mpspar$ linear---up to $\log$-terms---in $d_{\min} \defeq \min_{\l \in [L]} d_\l$, provided that the subspaces are not too close (in terms of their pairwise affinities), the noise variance $\sigma^2$ is sufficiently small, and the data set $\cY$ contains sufficiently many points from each subspace $\cS_\l$. 
Specifically, the clustering condition \eqref{eq:ClusCondThm} tells us that the subspaces $\cS_\l$ are allowed to be quite close to each other 
and can even intersect in a substantial fraction of their dimensions, all provided that $\sigma^2$ is not too large. 
Moreover, inspection of the second term on the \ac{lhs} of 
\eqref{eq:ClusCondThm} shows that a higher noise variance $\sigma^2$ is tolerated when $m$ becomes large relative to the largest subspace dimension $d_{\max}$ and/or the data set $\cY$ contains an increasing number of points in each of the subspaces, resulting in an increase in the minimum sampling density $\rho_{\min}$. 
The clustering condition \eqref{eq:ClusCondThm} can hence be satisfied under the condition $\sigma \leq 1/2$ imposed by Theorems \ref{th:SSCOMP} and \ref{th:SSCMP} if $m$ is sufficiently large relative to $d_{\max}$ and if $\rho_{\min}$ is sufficiently large 
(but not too large, in order to prevent the \ac{rhs} of \eqref{eq:ClusCondThm} from becoming too small; for example, $\rho_{\min}$ should not scale exponentially in one of the $d_\l$). 
This shows that \ac{ssc}-\ac{omp} and \ac{ssc}-\ac{mp}, indeed, \rev{satisfy the \ac{nfc} property} 
even when the noise variance $\sigma^2$ is on the order of the signal energy, i.e., when the signal to noise ratio $\text{SNR} \defeq \EX{\smash{\norm[2]{\vx_j}^2}}/\EX{\smash{\norm[2]{\vz_j}^2}} = 1/\sigma^2$\vspace{0.05cm} satisfies $\text{SNR} \approx 0\text{dB}$ (recall that $\yl_i = \xl_i + \zl_i$ with $\EX{\smash{\norm[2]{\smash{\xl_i}}^2}} = 1$). 

\rev{The \ac{rhs} of \eqref{eq:ClusCondThm} going to zero as $N \to \infty$ may appear counter-intuitive as one would expect clustering to become easier when the number of data points increases. Note, however, that \eqref{eq:ClusCondThm} allows the subspaces to intersect, and Theorems \ref{th:SSCOMP} and \ref{th:SSCMP} guarantee the \ac{nfc} property for \emph{all data points}. Now, when $N$ increases, owing to the statistical data model our analysis is based on, the number of data points that are close to the intersection of two subspaces also increases, which in turn leads to an increase in the probability of the \ac{nfc} property being violated for at least one data point. This then results in the clustering condition becoming more restrictive. The clustering conditions for \ac{ssc} in \cite[Eq.~(3.1)]{soltanolkotabi2014robust}, \revb{\cite[Thm. 10]{wang2013noisy}} and for \ac{tsc} in \cite[Eq.~(8)]{heckel_robust_2013} exhibit the same $O(1/\log (N))$ scaling and hence the same seemingly counter-intuitive behavior.}

We hasten to add that the condition $\sigma \leq 1/2$ in Theorems \ref{th:SSCOMP} and \ref{th:SSCMP} was imposed only to get clustering conditions that are of simple form. Removing the restriction $\sigma \leq 1/2$ (which is used to get the bounds \eqref{eq:resoSimpleUB} and \eqref{eq:respSimpleLB} in Appendix~\ref{sec:pfsscomp}) would lead to clustering conditions allowing, in principle, for arbitrarily large $\sigma$ (i.e., even for $\text{SNR} < 0\text{dB}$), provided that the $d_\l$ are sufficiently small compared to $m$, and $\rho_{\min}$ is sufficiently large. \rev{One might further expect that the upper bounds on $\itr_{\max}$ and $\mpspar$ in Theorems \ref{th:SSCOMP} and \ref{th:SSCMP}, respectively, should depend on $\sigma$ because the number of iterations for which \ac{ssc}-\ac{omp} and \ac{ssc}-\ac{mp} are guaranteed to select points from $\cY_\l \backslash \{\vy_j\}$ for $\vy_j \in \cY_\l$ should decrease as $\sigma$ increases. However, this is not the case as the clustering condition \eqref{eq:ClusCondThm} limits the noise variance (more precisely, the variance of the noise components on the subspaces) depending on $\max_{k, \l \colon k \neq \l} \aff (\cS_k,\cS_\l)$, $N$, $\rho_{\min}$, and $\itr_{\max}$. 
}We furthermore note that the conditions in Theorems \ref{th:SSCOMP} and \ref{th:SSCMP} (with different constants in \eqref{eq:ClusCondThm}) continue to guarantee the \ac{nfc} property for bounded noise or sub-gaussian noise, in both cases of isotropic distribution. It is interesting to see that \ac{ssc}-\ac{mp} \rev{satisfies the \ac{nfc} property} 
under virtually the same conditions as \ac{ssc}-\ac{omp}, although in practice \ac{ssc}-\ac{mp} typically exhibits a lower running time at fixed performance.

Comparing the clustering condition \eqref{eq:ClusCondThm} to those for \ac{ssc} in \cite[Thm. 3.1]{soltanolkotabi2014robust} and for \ac{tsc} in \cite[Thm. 3]{heckel_robust_2013}, both of which guarantee the \ac{nfc} property and apply to the same data model as used here, we find that \eqref{eq:ClusCondThm} exhibits the same structure  
(up to $\log$-factors and constants) 
apart from the term proportional to $\sqrt{1/\rho_{\min}}$ on the \ac{lhs} of \eqref{eq:ClusCondThm}. This term dominates the term proportional to $\sqrt{d_{\max} / m}$ only if $\sqrt{d_{\max} / m} \ll \sqrt{1/\rho_{\min}} = \sqrt{ \max_{\l \in [L]} (d_\l/(n_\l-1))}$, i.e., if $\max_{\l \in [L]} n_\l \ll m$ (owing to $\max_{\l \in [L]} (d_\l/(n_\l-1)) \geq (\max_{\l \in [L]} d_\l)/(\max_{\l \in [L]} n_\l-1) > d_{\max}/(\max_{\l \in [L]} n_\l)$). \revb{Similarly, the clustering condition in \cite[Thm. 10]{wang2013noisy} does not have a term proportional to $\sqrt{1/\rho_{\min}}$ as \eqref{eq:ClusCondThm} does, but imposes a slightly more restrictive condition on $\sigma$, requiring $\sigma(c+\sigma)$ to be at most on the order of $\sqrt{m-d}/d$ instead of $\sqrt{m}/\sqrt{d}$ (assuming $d_\l =d$ for all $\l \in [L]$ and neglecting $\log$-terms for simplicity of exposition), where $c$ is a constant.}
Numerical results in Section \ref{sec:synthexp} indicate that the term proportional to $\sqrt{1/\rho_{\min}}$ in \eqref{eq:ClusCondThm} is not an artifact of our proof techniques, but rather fundamental. We further note that setting $\sigma = 0$, the second term on the \ac{lhs} of \eqref{eq:ClusCondThm} vanishes and we recover (up to $\log$-factors and constants) the clustering condition
\revb{\begin{equation}
\max_{k, \l \colon k \neq \l} \aff (\cS_k,\cS_\l) \leq \frac{\sqrt{\log (\rho_{\min})}}{64 \log (N)}\nonumber
\end{equation}}
found in \cite[Cor. 1]{heckel2015dimensionality} for \ac{ssc}-\ac{omp} in the noiseless case. 

In summary, \ac{ssc}-\ac{omp}, \ac{ssc}-\ac{mp}, \ac{tsc}, and \ac{ssc} all \rev{satisfy the \ac{nfc} property} 
under similar (sufficient) conditions, while differing considerably w.r.t. computational complexity. Specifically, \ac{ssc}-\ac{omp} and \ac{ssc}-\ac{mp}, albeit greedy, are computationally more expensive than \ac{tsc}, but significantly less expensive than \ac{ssc}. On the other hand, \ac{ssc}-\ac{mp} can outperform \ac{tsc} quite significantly in certain applications (see Section \ref{sec:faceclustering}). A detailed comparison of \ac{ssc}, \ac{ssc}-\ac{omp}, \ac{ssc}-\ac{mp}, and \ac{tsc} in terms of performance and running times is provided in Section \ref{sec:faceclustering}. The performance of all four algorithms varies across data sets, and none of the algorithms consistently outperforms the other ones. 

Recall that under \ac{di}-stopping the choice of the parameters $\itr_{\max}$, $\mpspar$ is critical for the success of \ac{ssc}-\ac{omp} and \ac{ssc}-\ac{mp}. Taking $\itr_{\max}, \mpspar$ too small or too large may lead to cluster split-ups or to many false connections, respectively. The maximum range for $\itr_{\max}$, $\mpspar$ for our results to guarantee the \ac{nfc} property is determined (up to $\log$-factors) by the smallest subspace dimension $d_{\min}$, which is usually unknown. 
Furthermore, if $d_{\min}$ is small, the range of admissible values for $\itr_{\max}$, $\mpspar$ will also be small. The clustering condition \eqref{eq:ClusCondThm} is, however, only sufficient (for the \ac{nfc} property to hold) and good clustering performance may be obtained in practice for larger values of $\itr_{\max}$, $\mpspar$ than those identified by Theorems~\ref{th:SSCOMP} and \ref{th:SSCMP}. 

We proceed to our main result on \ac{dd}-stopping, which indicates that the problems with choosing $\itr_{\max}$, $\mpspar$ for \ac{di}-stopping due to unknown $d_\l$ can be mitigated---to a certain extent---through \ac{dd}-stopping. Specifically, we show that \ac{ssc}-\ac{omp} and \ac{ssc}-\ac{mp} under \ac{dd}-stopping automatically select at least on the order of $d_\l$ points from $\cY_\l$ to represent $\vy_j \in \cY_\l$. 
We hasten to add, however, that Theorem \ref{co:trueconnections} does not guarantee that no additional data points corresponding to false connections are selected and Theorem \ref{co:trueconnections} hence does not guarantee the \ac{nfc} property.

\begin{theorem}[\ac{ssc}-\ac{omp} and \ac{ssc}-\ac{mp} with \ac{dd}-stopping] \label{co:trueconnections}
Define the sampling density $\rho_\l \defeq (n_\l-1)/d_\l$, and let $d_{\max} \defeq \max_{\l \in [L]} d_\l$ and $\rho_{\min} = \min_{\l \in [L]} \rho_\l$. Suppose that $m \geq 2 d_{\max}$, $\rho_{\min} \geq c_\rho$, and $\sigma \leq 1/2$, where $c_\rho$ is a numerical constant satisfying $c_\rho >1$. Pick $\tau \in [0, 2/3 - (\sqrt{d_{\max}}/\sqrt{m})\sigma]$. Then, the clustering condition \eqref{eq:ClusCondThm} with $\itr_{\max}$ on both sides replaced by $\max_{\l \in [L]} \lfloor c_s d_\l / \log( (n_\l - 1) e) \rfloor$ guarantees w.p. at least $P^\star$ as defined in \eqref{eq:pstar}, for all $\vy_j \in \cY_\l$, $j \in [n_\l]$, $\l \in [L]$, that the corresponding coefficient vectors $\vb_j$ computed by \ac{omp} and \ac{mp} (if it terminates) 
have at least 
\begin{equation}
\!\!\left\lfloor \frac{d_\l}{\log((n_\l-1)e)} \min \left\{\frac{1}{3} \left( \frac{2}{3} - \frac{\tau}{1 - \frac{3}{2}\frac{\sqrt{d_\l}}{\sqrt{m}} \sigma}\right)^2,c_s \right\} \right\rfloor \label{eq:tdlbcor}
\end{equation}
non-zero entries corresponding to points in $\cY_\l \backslash \{\vy_j\}$.
\end{theorem}

Note that \ac{mp} is not guaranteed to terminate under the conditions of Theorem \ref{co:trueconnections} as $\tau$ in the admissible range indicated by Theorem \ref{co:trueconnections} could be too small for termination (see the corresponding discussion in Section \ref{sec:algos}). The ensuing statements on \ac{ssc}-\ac{mp} all apply only if \ac{mp}, indeed, terminates for all points $\vy_j \in \cY$. 
Theorem \ref{co:trueconnections} identifies a range for the threshold parameter $\tau$ guaranteeing that both \ac{ssc}-\ac{omp} and \ac{ssc}-\ac{mp} deliver a graph $G$ which has each $\vy_j \in \cY_\l$, $\l \in [L]$, connected to at least $O(d_\l/\log(n_\l-1))$ other points in $\cY_\l$. 
\rev{If $\sigma$ increases, the probability of OMP and MP selecting false connections increases and more iterations need to be performed for a given number of true connections to be selected. As a consequence, the interval for $\tau$ specified in Theorem \ref{co:trueconnections} decreases as $\sigma$ increases.} As already pointed out, Theorem \ref{co:trueconnections} does not guarantee the \ac{nfc} property, and choosing $\tau$ too small will result in entries in the coefficient vectors $\vb_j$ that correspond to false connections. Intuitively, we expect that choosing $\tau$ sufficiently large, \ac{omp} and \ac{mp} should stop early enough so as to avoid false connections. \rev{More specifically, one would expect that $\tau$ needs to be chosen larger as $\sigma$ increases so as to avoid \ac{omp} and \ac{mp} selecting points from $\cY \backslash \cY_\l$ to represent $\vy_j \in \cY_\l$.} Unfortunately, it seems rather difficult, at least for the statistical data model at hand, to analytically characterize a range for $\tau$ that guarantees the \ac{nfc} property and simultaneously on the order of $d_\l$ connections between $\vy_j \in \cY_\l$ and other points in $\cY_\l$, for all $j \in [n_\l]$, $\l \in [L]$. Nonetheless, it turns out, that in practice $G$ often exhibits both of these properties if $\tau$ is chosen appropriately. Numerical results in Section \ref{sec:tprfprexp} corroborate this claim. In summary, \ac{ssc}-\ac{omp} and \ac{ssc}-\ac{mp} under \ac{dd}-stopping with appropriately chosen $\tau$ detect the dimensions of the subspaces $\cS_\l$ correctly and adapt the sparsity level of the individual representations according to the dimension of the subspace the data  point at hand lies in.

The procedure in \rev{\cite[Alg. 2]{soltanolkotabi2014robust}} for the selection of a per-data-point Lasso parameter in \ac{ssc} has a similar subspace dimension-detecting property but comes with stronger theoretical guarantees. 
Specifically, \rev{\cite[Thm. 3.1]{soltanolkotabi2014robust} and \cite[Thm. 3.2]{soltanolkotabi2014robust} taken together} guarantee the \ac{nfc} property and, concurrently, that each $\vy_j \in \cY_\l$ is connected to on the order of $d_\l$  
other points in $\cY_\l$, all this provided that the noise variance---assumed known---is small enough. \rev{The procedure in \cite[Alg. 2]{soltanolkotabi2014robust} could also be employed to select $\itr_{\max}$ in \ac{ssc}-\ac{omp} and \ac{ssc}-\ac{mp} under \ac{di}-stopping for each data point individually. 
This emulates \ac{dd}-stopping by employing \ac{di}-stopping together with a data-point-wise parameter selection procedure. More specifically, as shown in \cite[Lem. A.2]{soltanolkotabi2014robust}, the optimal cost of the auxiliary Lasso problem in \cite[Eq.~(2.4)]{soltanolkotabi2014robust} for $\vy_j \in \cY_\l$ is proportional to $\sqrt{d_\l}$. Squaring the optimal cost therefore yields an estimate of $d_\l$, which can, in turn, be used to select the parameter $\itr_{\max}$ in \ac{ssc}-\ac{omp} such that it is on the order of $d_\l$ for $\vy_j \in \cY_\l$ and thereby satisfies the condition in Theorem \ref{th:SSCOMP} (we can lower-bound the factor $1/\log( (n_\l - 1) e / \itr_{\max})$ in that condition by $1/\log( N e )$).
This would then guarantee, in addition to the \ac{nfc} property, on the order of $d_\l / \log( n_\l)$ true connections for $\vy_j \in \cY_\l$, $j \in [n_\l]$, $\l \in [L]$, for both \ac{ssc}-\ac{omp} and \ac{ssc}-\ac{mp} under the conditions of Theorems 1 and 2, and would hence realize the ``many true discoveries'' and the \ac{nfc} property at the same time as guaranteed by \cite[Thm.~3.1]{soltanolkotabi2014robust} and \cite[Thm.~3.2]{soltanolkotabi2014robust} 
for \ac{ssc}.
We point out, however, that the selection procedure in \cite[Alg. 2]{soltanolkotabi2014robust} results in considerable computational burden in addition to solving the (already computationally demanding) $N$ Lasso problems required by \ac{ssc} or running the \ac{omp} and \ac{mp} routines in \ac{ssc}-\ac{omp} and \ac{ssc}-\ac{mp}, respectively, to perform the actual clustering.}

Finally, we note that determining the range of admissible threshold parameters $\tau \in [0, 2/3 - (\sqrt{d_{\max}}/\sqrt{m})\sigma]$ in Theorem \ref{co:trueconnections} requires knowledge of the noise variance $\sigma^2$. 
In principle, knowledge of $d_{\max}$ is required as well. We can, however, upper-bound $\sqrt{d_{\max}}/\sqrt{m}$ by 1 thereby obviating the need for knowing $d_{\max}$ at the cost of a reduced range for $\tau$. 

\section[Numerical results]{Numerical results} 

\label{sec:numres}

We compare the performance of \ac{ssc}-\ac{omp}, \ac{ssc}-\ac{mp}, \ac{ssc}, \ac{tsc}, and \ac{nsn}.
\ac{ssc}-\ac{omp} and \ac{ssc}-\ac{mp} were implemented in Matlab exactly following their descriptions in Section \ref{sec:algos}. 
For \ac{ssc}, \ac{tsc}, and \ac{nsn}, we used the implementations provided in the corresponding references \cite{elhamifar_sparse_2013}, \cite{heckel_robust_2013}, and \cite{park2014greedy}, respectively. The code for reproducing all experiments in this paper can be found at \url{http://www.nari.ee.ethz.ch/commth/research/}. This code also contains information on the number of Monte Carlo runs used in the individual experiments.

We provide the algorithms with the true number of subspaces $L$ in all experiments. All running times were measured on a PC with 32 GB RAM and 4-core Intel Core i7-3770K CPU clocked at 3.50 GHz. It is quite common in the computer vision literature to perform post-processing on the adjacency matrix $\mA$ generated by the individual clustering algorithms. This can improve the clustering performance, but will not be pursued here \rev{in order to simplify our comparisons}. 

\newcommand{\plotred}{black!59!white}
\newcommand{\plotgreen}{black!20!white}
\newcommand{\plotblue}{black!80!white} 
\newcommand{\plotblack}{black!97!white}

\newcommand{\plotwidth}{0.35\textwidth}
\newcommand{\plotheight}{0.25\textwidth}
\newcommand{\plotseph}{1.5cm}
\newcommand{\plotsepv}{2.5cm}
\newcommand{\labelsize}{\normalsize}

\subsection{Comparison of \ac{ssc}-\ac{omp} and \ac{ssc}-\ac{mp}}
\label{sec:synthexp}
As the clustering condition \eqref{eq:ClusCondThm} is only a sufficient condition and guarantees the \ac{nfc} property only, it is unclear to what extent the behavior predicted by \eqref{eq:ClusCondThm} is reflected in the \ac{ce}, i.e., the fraction of misclustered data points. 
The following experiment is devoted to answering this question while comparing \ac{ssc}-\ac{omp} and \ac{ssc}-\ac{mp}. 
We generate data sets $\cY$ according to the statistical data model described in Section \ref{sec:mainres}. Specifically, we set $d_\l = d$, $\l \in [L]$, and we choose the bases $\Ul \in \reals^{m \times d}$ 
to all intersect in a shared $t$-dimensional space and to be mutually orthogonal on the orthogonal complement of this intersection. More specifically, the bases $\Ul$ are obtained by choosing a matrix $\mU \in \reals^{m \times(L(d-t) + t)}$ uniformly at random from the set of all orthonormal matrices of dimension $m \times (L(d-t) + t)$ and setting $\Ul \defeq [ \mU_{[t]} \; \mU_{\mc T_\l} ]$, where $\mc T_\l \defeq \{ t+ (\l-1) (d-t) + 1, \dots, t+ \l (d-t)  \}$. This results in $\aff (\cS_k,\cS_\l) = \sqrt{t/d}$, $k,\l \in [L]$, $k \neq \l$. 
Varying the parameter $t$ therefore allows us to vary the pairwise affinities. 
We furthermore set $n_\l = n$, $\l \in [L]$, and generate instances of $\cY$ by sampling $n$ data points uniformly at random from each subspace and adding $\mathcal N(\mathbf 0, (\sigma^2/m) \mI )$ noise to each data point (the noise vectors are generated independently across data points). 
We let $L=3$, $d=20$, $m = 200$, and vary $t$, $\rho_{\min} = \rho = n/d$, and $\sigma^2$. Furthermore, we employ \ac{di}-stopping and set $\itr_{\max} = d/2 = 10$ for both \ac{ssc}-\ac{omp} and \ac{ssc}-\ac{mp}, and $\mpspar = N$. Figure \ref{fig:phasediag} shows the \ac{ce} as a function of $\max_{k,\l \colon k \neq \l} \aff (\cS_k,\cS_\l) = \sqrt{t/d}$, $\rho$, and $\sigma^2$. The results nicely reflect the qualitative behavior indicated by the clustering condition \eqref{eq:ClusCondThm}. Specifically, both \ac{ssc}-\ac{omp} and \ac{ssc}-\ac{mp} tolerate higher noise variance as the affinities between the subspaces decrease and the number of points in $\cY$ drawn from each subspace, $n$, increases. It is furthermore interesting to observe that the performance of \ac{ssc}-\ac{omp} and \ac{ssc}-\ac{mp} is virtually identical.

Recall that the clustering condition \eqref{eq:ClusCondThm}, apart from the term proportional to $\sqrt{1/\rho_{\min}}$, exhibits the same scaling behavior as those guaranteeing the \ac{nfc} property for \ac{ssc} in \cite[Thm. 3.1]{soltanolkotabi2014robust} and for \ac{tsc} in \cite[Thm. 3]{heckel_robust_2013}. To find out whether this additional term is an artifact of our proof technique, we first note that the clustering condition \eqref{eq:ClusCondThm} takes the form 
\begin{equation}
 \max_{k,\l \colon k \neq \l} \aff (\cS_k,\cS_\l) + \frac{c_1}{\sqrt{\rho}} \leq c_2, \label{eq:phasetrans1}
\end{equation}
for fixed $\sigma$, and 
\begin{equation}
 \sigma ( c_3 + \sigma c_4 + \frac{1}{\sqrt{\rho}} (c_5 + \sigma c_6)) \leq c_7, \label{eq:phasetrans2}
\end{equation}
for fixed maximum affinity, where $c_1$-$c_7> 0$ are constants, $d_{\max}$ and $m$ were assumed constant in both cases, and factors logarithmic in any of the parameters (variable or fixed) were neglected. Rewriting \eqref{eq:phasetrans1} and \eqref{eq:phasetrans2} assuming equality, we get
\begin{equation}
\rho = \left(\frac{c_1}{c_2-\max_{k,\l \colon k \neq \l} \aff (\cS_k,\cS_\l)}\right)^2 \label{eq:curvefit1}
\end{equation}
and
\begin{equation}
\rho = \left( \frac{\sigma (c_5 + c_6 \sigma)}{c_7 - \sigma(c_3 + \sigma c_4)} \right)^2, \label{eq:curvefit2}
\end{equation}
respectively. 
In the top and bottom rows of Figure \ref{fig:phasediag}, we now fit 
\eqref{eq:curvefit1} and \eqref{eq:curvefit2}, respectively (by manually adjusting the constants $c_1$-$c_7 $), to the boundaries between the regions of success and the regions of failure. 
The shape of the fitted curves follows the boundaries indicated by the numerical results closely. This can be taken as an indication---at least to a certain extent---of the term in \eqref{eq:ClusCondThm} proportional to $\sqrt{1/\rho_{\min}}$ being fundamental.

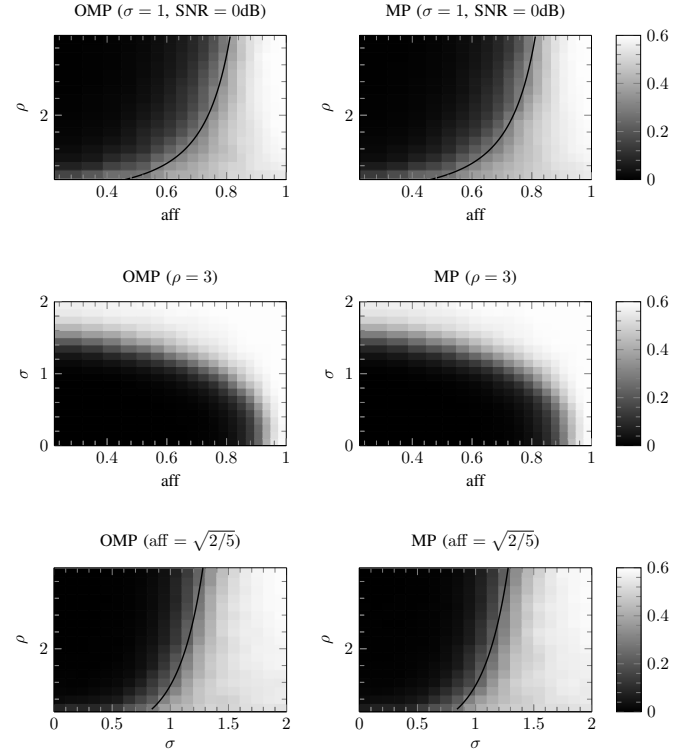
\begin{figure}[ph!]
\centering
\begin{tikzpicture}[scale=0.65] 
\begin{groupplot}[group style={group size=2 by 3,horizontal sep=\plotseph,vertical sep=\plotsepv,xlabels at=edge bottom, ylabels at=edge left}, width=\plotwidth,height=\plotheight, /tikz/font=\normalsize, colormap/blackwhite, view={0}{90}, point meta min=0.0, point meta max=0.6, minor tick num=4]

\nextgroupplot[title = {OMP ($\sigma = 1$, $\text{SNR}=0\text{dB}$)}, xlabel={\labelsize aff}, ylabel={\labelsize $\rho$}]
\addplot3[surf, shader=flat] file {./data/phasediag/CE-tvsrho_omp.dat};
\addplot[black,thick,domain=0.46:0.812,samples=50] {(0.37/(1-x))^2};

\nextgroupplot[title = {MP ($\sigma = 1$, $\text{SNR}=0\text{dB}$}), xlabel={\labelsize aff}, ylabel={\labelsize $\rho$}, colorbar]
\addplot3[surf, shader=flat, colorbar] file {./data/phasediag/CE-tvsrho_mp.dat};
\addplot[black,thick,domain=0.46:0.812,samples=50] {(0.37/(1-x))^2};

\nextgroupplot[title = {OMP ($\rho = 3$)}, xlabel={\labelsize aff}, ylabel={\labelsize $\sigma$}]
\addplot3[surf, shader=flat] file {./data/phasediag/CE-tvssig_omp.dat};

\nextgroupplot[title = {MP ($\rho = 3$)}, xlabel={\labelsize aff}, ylabel={\labelsize $\sigma$}, colorbar]
\addplot3[surf, shader=flat, colorbar] file {./data/phasediag/CE-tvssig_mp.dat};

\nextgroupplot[title = {OMP ($\aff = \sqrt{2/5}$)}, xlabel={\labelsize $\sigma$}, ylabel={\labelsize $\rho$}]
\addplot3[surf, shader=flat] file {./data/phasediag/CE-rhovssig_omp.dat};

\addplot[black,thick,domain=0.84:1.28,samples=50] {(x*(1+0.7*x)/(2.3 - 0.2*x - 0.5*x^2))^2};

\nextgroupplot[title = {MP ($\aff = \sqrt{2/5}$)}, xlabel={\labelsize $\sigma$}, ylabel={\labelsize $\rho$}, colorbar]
\addplot3[surf, shader=flat] file {./data/phasediag/CE-rhovssig_mp.dat};
\addplot[black,thick,domain=0.84:1.28,samples=50] {(x*(1+0.7*x)/(2.3 - 0.2*x - 0.5*x^2))^2};

\end{groupplot}
\end{tikzpicture}
\caption{\label{fig:phasediag} \ac{ce} as a function of $\aff \defeq \max_{k,\l \colon k \neq \l} \aff (\cS_k,\cS_\l)$, $\rho = n/d$, and $\sigma^2$. 
The (fitted) black curves in the top and bottom rows correspond to the curves \eqref{eq:curvefit1} and \eqref{eq:curvefit2}, respectively, and delineate the boundary between success and failure.}
\end{figure}

\subsection{Face clustering} \label{sec:faceclustering}

Next, we consider the problem of clustering face images of different individuals taken under varying illumination conditions. The rationale for employing subspace clustering to solve this problem stems from the observation that vectorized images of a given individual taken under different lighting conditions lie (approximately) in a $9$-dimensional linear subspace \cite{basri_lambertian_2003}. 
We consider the Extended Yale B data set \cite{georghiades_illumination_2001,lee_acquiring_2005}, which contains $192 \times 168$ pixel frontal face images of $38$ individuals, each taken under $64$ different illumination conditions. As \ac{ssc}-\ac{mp} does not seem to have been considered before in the literature, we compare its performance in terms of \ac{ce} 
to that of \ac{ssc}-\ac{omp}, \ac{ssc}, \ac{tsc}, and \ac{nsn}. Similar experiments comparing the performance of \ac{ssc}-\ac{omp} to that of other subspace clustering algorithms for face clustering were presented in \cite{dyer_greedy_2013, you_sparse_2015, heckel2015dimensionality}. To accomodate the memory requirements incurred by the optimization problems in \ac{ssc}, we apply \ac{ssc} (and, to ensure a fair comparison, also the other algorithms) to the downsampled $48 \times 42$ pixel versions of the images in the Extended Yale B data set provided in \cite{elhamifar_sparse_2013}. Note, however, that \ac{ssc}-\ac{omp}, \ac{ssc}-\ac{mp}, \ac{tsc}, and \ac{nsn} all could handle the original $192 \times 168$ pixel images. Instances of the data set $\cY$ are obtained by first choosing a subset of $L$ individuals uniformly at random from the set of all individuals and then collecting the $64$ (vectorized) images corresponding to each of the chosen individuals. As \ac{dd}-stopping leads to a large number of false connections---arguably due to (sparse) corruptions induced by shadows and specular reflections in the face images \cite{elhamifar_sparse_2013}---we rely on \ac{di}-stopping \revb{(i.e., $\tau=0$)} with $\itr_{\max} = 5$ (which corresponds to the choice made in \cite{dyer_greedy_2013} for \ac{ssc}-\ac{omp}) for both \ac{ssc}-\ac{omp} and \ac{ssc}-\ac{mp}, and we set $\mpspar = N$. Note that both \ac{omp} and \ac{mp} perform exactly $\itr_{\max}$ iterations in this experiment as the points in $\cY$ are in general position. 
For \ac{ssc}, \ac{tsc}, and \ac{nsn} we use the parameter values employed for the face clustering experiments in \cite{elhamifar_sparse_2013}, \cite{heckel_robust_2013}, and \cite{park2014greedy}, respectively. We emphasize that the experiments in \cite{elhamifar_sparse_2013}, \cite{heckel_robust_2013}, and \cite{park2014greedy} all also provide the true number of subspaces $L$ to the individual algorithms. Note furthermore that we employ \ac{ssc} with the objective function as formulated in \cite{elhamifar_sparse_2013} accounting for sparse corruptions of the data points. 

Table \ref{tab:fcce} shows the \ac{ce} for different choices of $L$, averaged over $100$ instances of $\cY$ for each $L$. In Table \ref{tab:fcrt}, we report the corresponding average running times.
\ac{ssc}-\ac{mp} outperforms \ac{ssc}-\ac{omp} (and \ac{tsc}) for all values of $L$, but $L = 2$, and does so at consistently lower running times (recall that \ac{ssc}-\ac{omp} and \ac{ssc}-\ac{mp} both perform exactly $\itr_{\max}$ iterations in this experiment, but \ac{ssc}-\ac{mp} has a lower per-iteration cost than \ac{ssc}-\ac{omp}). Furthermore, \ac{ssc}-\ac{omp} uniformly outperforms \ac{ssc}. The lowest \ac{ce} is obtained with \ac{nsn}. We note, however, that the performance of \ac{ssc}-\ac{mp}, except for $L=2$, almost matches that of \ac{nsn}, and \ac{ssc}-\ac{mp} consistently has roughly half the running time of \ac{nsn}. TSC has the lowest running time but yields the largest \ac{ce} (which is particularly high for this very data set \cite{heckel_robust_2013}). The running time of \ac{ssc} is one to two orders of magnitude higher than that of all other algorithms. Also note that the difference between the running times of \ac{ssc}-\ac{omp} and \ac{ssc}-\ac{mp} is small because $\itr_{\max}$ is small. \rev{Finally, the difference between the results for \ac{ssc}-\ac{omp}, \ac{ssc}, \ac{tsc}, and \ac{nsn} reported here compared to the results reported in published works \cite{you_sparse_2015, elhamifar_sparse_2013,heckel_robust_2013, park2014greedy} can be attributed to the fact that we do not perform post-processing on the adjacency matrix $\mA$ produced by the individual algorithms.}

Finally, we emphasize that for other applications such as, e.g., handwritten digit clustering \cite{heckel_robust_2013}, \ac{ssc} and \ac{tsc} yield lower CE than the other algorithms considered here, and none of the algorithms outperforms the others uniformly across applications as seen from experiments in \cite{elhamifar_sparse_2013,heckel_robust_2013,park2014greedy}.

\begin{table}[h]
\centering
\begin{tabular}{ l c c c c c}
  \hline
   $L$ & 2 & 3 & 5 & 8 & 10 \\
  \hline
  SSC-OMP & 2.83  &  4.04  &  6.81  &  12.98  &  14.14 \\
  SSC-MP & 4.16  &  3.72  &  5.24  &  9.09  &  11.36 \\
  SSC & 2.88  &  4.59  &  8.65  &  16.95  &  21.28 \\
  TSC &  10.71  &  16.81  &  29.73  &  38.34  &  41.22 \\
  NSN & 1.81  &  2.89  &  5.37  &  8.15  &  10.05 \\
  \hline
\end{tabular}
\caption{\label{tab:fcce}Average \ac{ce} (in percent) for face clustering.}
\end{table}

\begin{table}[h]
\centering
\begin{tabular}{l c c c c c}
  \hline
   $L$ & 2 & 3 & 5 & 8 & 10 \\
  \hline
  SSC-OMP & 0.21  &  0.36  &  0.80  &  2.01  &  3.13 \\
  SSC-MP & 0.15  &  0.28  &  0.65  &  1.74  &  2.80 \\
  SSC & 12.63  & 17.50 &  28.24  & 45.78 &  60.33 \\
  TSC &  0.08  &  0.14  &  0.29  &  0.62  &  0.95 \\
  NSN & 0.34  &  0.55  &  1.10  &  3.05  &  5.64 \\ 
  \hline
\end{tabular}
\caption{\label{tab:fcrt}Average running times (in seconds) for face clustering.}
\end{table}

\subsection{True positives/false positives tradeoff in \ac{dd}-stopping} \label{sec:tprfprexp}

Recall that for \ac{dd}-stopping Theorem \ref{co:trueconnections} guarantees that each data point $\vy_j \in \cY_\l$, $\l \in [L]$, is connected (in $G$) to at least $O(d_\l/\log(n_\l-1))$ other points in $\cY_\l$; we here refer to such connections as \ac{tp}. As already mentioned, Theorem \ref{co:trueconnections}, does, however, not guarantee the absence of false connections and provides a lower bound on the number of \ac{tp} only. 
It is therefore not clear, a priori, to what extent \ac{ssc}-\ac{omp} and \ac{ssc}-\ac{mp} under \ac{dd}-stopping, indeed, do detect the subspace dimensions. First, recall that by ``detecting the subspace dimensions'' we mean that 
the coefficient vectors $\vb_j$ for all $\vy_j \in \cY_\l$, $\l \in [L]$, have on the order of $d_\l$ non-zero entries corresponding to \ac{tp} and essentially no non-zero entries corresponding to false connections. 
We designate the number of \ac{tp} as $\tp_{\! \l}$, and the number of \ac{fp} as $\fp_\l$, 
both averaged over the data points corresponding to the subspace $\cS_\l$. 
We will also need the notions of \ac{tpr} and \ac{fpr} defined, in this experiment, as $\tpr_\l = \tp_{\! \l} /d_\l$ and $\fpr_\l = \fp_\l /(m-d_\l)$, respectively. 

We set $m = 300$, $L=4$, $\rho_\l = 4$, $\l \in [L]$, and choose the basis matrices $\Ul$ for the subspaces of dimensions $20$, $40$, $60$, and $80$, according to $\Ul \defeq [\bar \mU \; \; \Ult]$, where $\bar \mU$ and the $\Ult$ are drawn uniformly at random from the set of all orthonormal matrices of dimensions $300 \times 4$ and $300 \times (d_\l - 4)$, respectively. This guarantees that $\max_{k,\l \colon k \neq \l} \aff (\cS_k,\cS_\l) \geq \sqrt{1/5}$ as the subspaces all intersect in a shared $4$-dimensional space and possibly overlap in the orthogonal complement of this shared space. The data points are then drawn according to the statistical data model described in Section \ref{sec:mainres}.

Figure \ref{fig:tprfpr} shows $\tp_{\! \l}$ along with $\tpr_\l$ and $\fpr_\l$  
as a function of $\tau$ and for different values of $\sigma^2$. 
The middle and bottom rows in Figure \ref{fig:tprfpr} show that the \ac{tpr} curves corresponding to different $\l$ are almost on top of each other, i.e., $\tpr_1(\tau) \approx \ldots \approx \tpr_4(\tau) \approx c_\tpr(\tau)$, 
which means that the number of \ac{tp} for each subspace is, indeed, roughly proportional to the subspace dimension as $\tp_{\! \l}(\tau) = \tpr_\l(\tau) d_\l \approx c_{\,\tpr}(\tau) d_\l$, $\l \in [L]$. This indicates that the result in Theorem \ref{co:trueconnections} is order-wise optimal in $d_\l$. As, in addition, for large enough $\tau$, $\fpr_\l \approx 0$, $\l \in [L]$, we conclude that \ac{ssc}-\ac{omp} and \ac{ssc}-\ac{mp}, indeed, exhibit excellent subspace dimension detection properties provided that $\tau$ is chosen appropriately. 

We finally note that a similar experiment investigating the \ac{tpr}/\ac{fpr} tradeoff as a function of the Lasso parameter in \ac{ssc} was conducted in \cite[Sec. 2.4.3]{soltanolkotabi2014robust} with the main conclusion that a Lasso parameter on the order of $1/\sqrt{d_\l}$ extracts the subspace dimensions correctly order-wise with essentially no \ac{fp}. 

\renewcommand{\plotwidth}{0.45\textwidth}
\renewcommand{\plotheight}{0.35\textwidth}
\renewcommand{\plotseph}{1.5cm}
\renewcommand{\plotsepv}{2.5cm}
\newcommand{\pllw}{1.2pt}

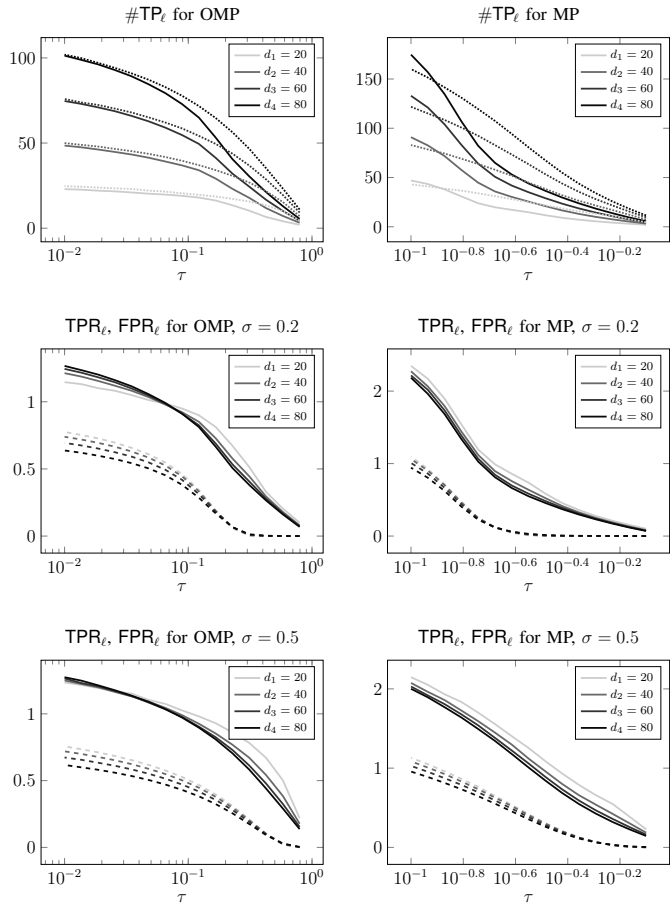
\begin{figure}[h!]
    \centering
    
\begin{tikzpicture}[scale=0.57, /tikz/font=\large] 
    \begin{semilogxaxis}[name=tdomp,title ={$\tp_{\! \l}$ for OMP},
            xlabel=$\tau$,
    width =\plotwidth,
    height =\plotheight,
    legend entries={ $d_1=20$, $d_2=40$, $d_3=60$, $d_4=80$} ,
    legend style={
            cells={anchor=west},
                    legend pos= north east,
                    font=\footnotesize,}
    ]
    \addplot +[line width=\pllw,mark=none,solid,\plotgreen] table[x index=0,y index=1]{./data/TPRFPR/TDOMP-0.20.dat};
    \addplot +[line width=\pllw,mark=none,solid,\plotred] table[x index=0,y index=2]{./data/TPRFPR/TDOMP-0.20.dat};
    \addplot +[line width=\pllw,mark=none,solid,\plotblue] table[x index=0,y index=3]{./data/TPRFPR/TDOMP-0.20.dat};
    \addplot +[line width=\pllw,mark=none,solid,\plotblack] table[x index=0,y index=4]{./data/TPRFPR/TDOMP-0.20.dat};
    \addplot +[line width=\pllw,mark=none,densely dotted,\plotgreen] table[x index=0,y index=1]{./data/TPRFPR/TDOMP-0.50.dat};
    \addplot +[line width=\pllw,mark=none,densely dotted,\plotred] table[x index=0,y index=2]{./data/TPRFPR/TDOMP-0.50.dat};
    \addplot +[line width=\pllw,mark=none,densely dotted,\plotblue] table[x index=0,y index=3]{./data/TPRFPR/TDOMP-0.50.dat};
    \addplot +[line width=\pllw,mark=none,densely dotted,\plotblack] table[x index=0,y index=4]{./data/TPRFPR/TDOMP-0.50.dat};
    
    \end{semilogxaxis}
    
    \begin{semilogxaxis}[ name=tdmp,title ={$\tp_{\! \l}$ for MP},anchor=west,at={($(tdomp.east)+(\plotseph,0)$)},
            xlabel=$\tau$,
    width =\plotwidth,
    height =\plotheight,
    legend entries={ $d_1=20$, $d_2=40$, $d_3=60$, $d_4=80$} ,
    legend style={
            cells={anchor=west},
                    legend pos= north east,
                    font=\footnotesize,}
    ]
    \addplot +[line width=\pllw,mark=none,solid,\plotgreen] table[x index=0,y index=1]{./data/TPRFPR/TDMP-0.20.dat};
    \addplot +[line width=\pllw,mark=none,solid,\plotred] table[x index=0,y index=2]{./data/TPRFPR/TDMP-0.20.dat};
    \addplot +[line width=\pllw,mark=none,solid,\plotblue] table[x index=0,y index=3]{./data/TPRFPR/TDMP-0.20.dat};
    \addplot +[line width=\pllw,mark=none,solid,\plotblack] table[x index=0,y index=4]{./data/TPRFPR/TDMP-0.20.dat};
    \addplot +[line width=\pllw,mark=none,densely dotted,\plotgreen] table[x index=0,y index=1]{./data/TPRFPR/TDMP-0.50.dat};
    \addplot +[line width=\pllw,mark=none,densely dotted,\plotred] table[x index=0,y index=2]{./data/TPRFPR/TDMP-0.50.dat};
    \addplot +[line width=\pllw,mark=none,densely dotted,\plotblue] table[x index=0,y index=3]{./data/TPRFPR/TDMP-0.50.dat};
    \addplot +[line width=\pllw,mark=none,densely dotted,\plotblack] table[x index=0,y index=4]{./data/TPRFPR/TDMP-0.50.dat};
    
    \end{semilogxaxis}
    
     \begin{semilogxaxis}[ name=tpfpomp02,title ={$\tpr_\l$, $\fpr_\l$ for OMP, $\sigma = 0.2$},anchor=north,at={($(tdomp.south)+(0,-\plotsepv)$)},
            xlabel=$\tau$,
    width =\plotwidth,
    height =\plotheight,
    legend entries={ $d_1=20$, $d_2=40$, $d_3=60$, $d_4=80$} ,
    legend style={
            cells={anchor=west},
                    legend pos= north east,
                    font=\footnotesize,}
    ]
    \addplot +[line width=\pllw,mark=none,solid,\plotgreen] table[x index=0,y index=1]{./data/TPRFPR/TPROMP-0.20.dat};
    \addplot +[line width=\pllw,mark=none,solid,\plotred] table[x index=0,y index=2]{./data/TPRFPR/TPROMP-0.20.dat};
    \addplot +[line width=\pllw,mark=none,solid,\plotblue] table[x index=0,y index=3]{./data/TPRFPR/TPROMP-0.20.dat};
    \addplot +[line width=\pllw,mark=none,solid,\plotblack] table[x index=0,y index=4]{./data/TPRFPR/TPROMP-0.20.dat};
    \addplot +[line width=\pllw,mark=none,dashed,\plotgreen] table[x index=0,y index=1]{./data/TPRFPR/FPROMP-0.20.dat};
    \addplot +[line width=\pllw,mark=none,dashed,\plotred] table[x index=0,y index=2]{./data/TPRFPR/FPROMP-0.20.dat};
    \addplot +[line width=\pllw,mark=none,dashed,\plotblue] table[x index=0,y index=3]{./data/TPRFPR/FPROMP-0.20.dat};
    \addplot +[line width=\pllw,mark=none,dashed,\plotblack] table[x index=0,y index=4]{./data/TPRFPR/FPROMP-0.20.dat};
    
    \end{semilogxaxis}
    
    \begin{semilogxaxis}[ name=tpfpmp02,title ={$\tpr_\l$, $\fpr_\l$ for MP, $\sigma = 0.2$},anchor=north,at={($(tdmp.south)+(0,-\plotsepv)$)},
            xlabel=$\tau$,
    width =\plotwidth,
    height =\plotheight,
    legend entries={ $d_1=20$, $d_2=40$, $d_3=60$, $d_4=80$} ,
    legend style={
            cells={anchor=west},
                    legend pos= north east,
                    font=\footnotesize,}
    ]
    \addplot +[line width=\pllw,mark=none,solid,\plotgreen] table[x index=0,y index=1]{./data/TPRFPR/TPRMP-0.20.dat};
    \addplot +[line width=\pllw,mark=none,solid,\plotred] table[x index=0,y index=2]{./data/TPRFPR/TPRMP-0.20.dat};
    \addplot +[line width=\pllw,mark=none,solid,\plotblue] table[x index=0,y index=3]{./data/TPRFPR/TPRMP-0.20.dat};
    \addplot +[line width=\pllw,mark=none,solid,\plotblack] table[x index=0,y index=4]{./data/TPRFPR/TPRMP-0.20.dat};
    \addplot +[line width=\pllw,mark=none,dashed,\plotgreen] table[x index=0,y index=1]{./data/TPRFPR/FPRMP-0.20.dat};
    \addplot +[line width=\pllw,mark=none,dashed,\plotred] table[x index=0,y index=2]{./data/TPRFPR/FPRMP-0.20.dat};
    \addplot +[line width=\pllw,mark=none,dashed,\plotblue] table[x index=0,y index=3]{./data/TPRFPR/FPRMP-0.20.dat};
    \addplot +[line width=\pllw,mark=none,dashed,\plotblack] table[x index=0,y index=4]{./data/TPRFPR/FPRMP-0.20.dat};
    
    \end{semilogxaxis}
    
     \begin{semilogxaxis}[ name=tpfpomp05,title ={$\tpr_\l$, $\fpr_\l$ for OMP, $\sigma = 0.5$},anchor=north,at={($(tpfpomp02.south)+(0,-\plotsepv)$)},
            xlabel=$\tau$,
    width =\plotwidth,
    height =\plotheight,
    legend entries={ $d_1=20$, $d_2=40$, $d_3=60$, $d_4=80$} ,
    legend style={
            cells={anchor=west},
                    legend pos= north east,
                    font=\footnotesize,}
    ]
    \addplot +[line width=\pllw,mark=none,solid,\plotgreen] table[x index=0,y index=1]{./data/TPRFPR/TPROMP-0.50.dat};
    \addplot +[line width=\pllw,mark=none,solid,\plotred] table[x index=0,y index=2]{./data/TPRFPR/TPROMP-0.50.dat};
    \addplot +[line width=\pllw,mark=none,solid,\plotblue] table[x index=0,y index=3]{./data/TPRFPR/TPROMP-0.50.dat};
    \addplot +[line width=\pllw,mark=none,solid,\plotblack] table[x index=0,y index=4]{./data/TPRFPR/TPROMP-0.50.dat};
    \addplot +[line width=\pllw,mark=none,dashed,\plotgreen] table[x index=0,y index=1]{./data/TPRFPR/FPROMP-0.50.dat};
    \addplot +[line width=\pllw,mark=none,dashed,\plotred] table[x index=0,y index=2]{./data/TPRFPR/FPROMP-0.50.dat};
    \addplot +[line width=\pllw,mark=none,dashed,\plotblue] table[x index=0,y index=3]{./data/TPRFPR/FPROMP-0.50.dat};
    \addplot +[line width=\pllw,mark=none,dashed,\plotblack] table[x index=0,y index=4]{./data/TPRFPR/FPROMP-0.50.dat};
    
    \end{semilogxaxis}
    
    \begin{semilogxaxis}[ name=tpfpmp05,title ={$\tpr_\l$, $\fpr_\l$ for MP, $\sigma = 0.5$},anchor=north,at={($(tpfpmp02.south)+(0,-\plotsepv)$)},
            xlabel=$\tau$,
    width =\plotwidth,
    height =\plotheight,
    legend entries={ $d_1=20$, $d_2=40$, $d_3=60$, $d_4=80$} ,
    legend style={
            cells={anchor=west},
                    legend pos= north east,
                    font=\footnotesize,}
    ]
    \addplot +[line width=\pllw,mark=none,solid,\plotgreen] table[x index=0,y index=1]{./data/TPRFPR/TPRMP-0.50.dat};
    \addplot +[line width=\pllw,mark=none,solid,\plotred] table[x index=0,y index=2]{./data/TPRFPR/TPRMP-0.50.dat};
    \addplot +[line width=\pllw,mark=none,solid,\plotblue] table[x index=0,y index=3]{./data/TPRFPR/TPRMP-0.50.dat};
    \addplot +[line width=\pllw,mark=none,solid,\plotblack] table[x index=0,y index=4]{./data/TPRFPR/TPRMP-0.50.dat};
    \addplot +[line width=\pllw,mark=none,dashed,\plotgreen] table[x index=0,y index=1]{./data/TPRFPR/FPRMP-0.50.dat};
    \addplot +[line width=\pllw,mark=none,dashed,\plotred] table[x index=0,y index=2]{./data/TPRFPR/FPRMP-0.50.dat};
    \addplot +[line width=\pllw,mark=none,dashed,\plotblue] table[x index=0,y index=3]{./data/TPRFPR/FPRMP-0.50.dat};
    \addplot +[line width=\pllw,mark=none,dashed,\plotblack] table[x index=0,y index=4]{./data/TPRFPR/FPRMP-0.50.dat};
    
    \end{semilogxaxis}
    
    \end{tikzpicture}
    \caption{\label{fig:tprfpr} \ac{tp}/\ac{fp}     tradeoff in \ac{dd}-stopping as a function of $\tau$. Top row: Solid lines: $\sigma = 0.2$, dotted lines: $\sigma = 0.5$. Middle and bottom rows: Solid lines: \ac{tpr}, dashed lines: \ac{fpr}.}
    \end{figure}

\subsection{Influence of $\itr_{\max}$ and $\mpspar$ in \ac{di}-stopping} \label{sec:inflsmax}

Recall that Theorems \ref{th:SSCOMP} and \ref{th:SSCMP} provide a range of admissible values for $\itr_{\max}$ and $\mpspar$ in \ac{di}-stopping. For small $d_{\min}$, these ranges are, however, small. As already pointed out, taking $\itr_{\max}$, $\mpspar$ too small leads to cluster split-ups, whereas $\itr_{\max}$, $\mpspar$ too large results in a large number of false connections in $G$. It is therefore important to determine the sensitivity of \ac{ssc}-\ac{omp} and \ac{ssc}-\ac{mp} performance w.r.t. the choice of $\itr_{\max}$ and $\mpspar$. The next experiment is devoted to this matter. 
We consider the \ac{ce} as well as the average \ac{tpr} and \ac{fpr} which, in this experiment, are defined as $\tpr'_\l = \tp_{\! \l} /n_\l$ and $\fpr'_\l = \fp_\l / (N-n_\l)$, respectively. These alternative normalizations are used as we are not interested in investigating the dependence of \ac{tp} and \ac{fp} on the $d_\l$, as was done in Section~\ref{sec:tprfprexp}; rather, we want to ensure that $\tpr'_\l, \fpr'_\l \in [0,1]$. 
Furthermore, we define the \ac{tp}-$\ell_1$-norm and the \ac{fp}-$\ell_1$-norm as the $\ell_1$-norm of the entries of the coefficient vectors $\vb_j$ corresponding to \ac{tp} and \ac{fp}, respectively, both averaged over all data points. The \ac{tp}- and \ac{fp}-$\ell_1$-norms hence correspond to half of the average weight associated with \ac{tp} and \ac{fp}, respectively; 
the factor $1/2$ stems from the adjacency matrix of $G$ being given by $\mA = \mB + \transp{\mB}$. The motivation for considering \ac{tp}-/\ac{fp}-$\ell_1$-norms comes from the fact that the performance of spectral clustering is determined not only by the number of \ac{tp} and \ac{fp}, but also by the 
weights associated with the \ac{tp} and the \ac{fp}. 
Specifically, even when the $\vb_j$ contain a considerable number of \ac{fp}, good performance can still be obtained provided that the corresponding \ac{fp}-$\ell_1$-norm is sufficiently small. 

We cluster synthetic data generated according to the statistical data model described in Section~\ref{sec:mainres} as well as images taken from the Extended Yale B data set (we use downsampled $48 \times 42$ pixel versions of the images, see Section \ref{sec:faceclustering} for a detailed description). More specifically, in the case of synthetic data, we set $L=3$, $m = 80$, $\sigma = 0.5$, $d_\l = 15$, $\l \in [L]$, and $\rho_\l = 4$, $\l \in [L]$, and we generate the bases $\Ul$, $\l \in [L]$, to intersect in a shared $3$-dimensional space by following the construction employed in Section \ref{sec:tprfprexp} such that $\max_{k,\l \colon k \neq \l} \aff (\cS_k,\cS_\l) = \sqrt{1/5}$. In the case of face clustering, we follow the procedure described in Section \ref{sec:faceclustering} to obtain instances of $\cY$ containing the face images of $L=3$ randomly selected individuals. 

Figure \ref{fig:fracesrprt} shows the \ac{ce}, \ac{tpr}/\ac{fpr}, and \ac{tp}-/\ac{fp}-$\ell_1$-norms as a function of $\itr_{\max}$ and $\mpspar$. For both clustering problems the \ac{ce} of \ac{ssc}-\ac{omp} is seen to increase rapidly as a function of $\itr_{\max}$, while for \ac{ssc}-\ac{mp} with $\mpspar = N$ (i.e., stopping is activated by $\itr = \itr_{\max}$), the \ac{ce} increases very slowly 
for the face clustering problem and does not increase at all in the case of synthetic data. This indicates that \ac{ssc}-\ac{mp} 
exhibits significantly smaller sensitivity to the choice of $\itr_{\max}$ than \ac{ssc}-\ac{omp}. As already pointed out in Section \ref{sec:algos}, this is due to the ability of \ac{ssc}-\ac{mp} to select points $\vy_i \in \cY_\l \backslash \{ \vy_j \}$ repeatedly to participate in the representation of $\vy_j \in \cY_\l$; for given $\itr_{\max}$ this results in \ac{ssc}-\ac{mp} producing fewer \ac{fp} than \ac{ssc}-\ac{omp}. Moreover, for \ac{ssc}-\ac{mp}, the \ac{tp}-$\ell_1$-norm exceeds the \ac{fp}-$\ell_1$-norm for all values of $\itr_{\max}$, while for \ac{ssc}-\ac{omp} the \ac{fp}-$\ell_1$-norm exceeds the \ac{tp}-$\ell_1$-norm for large $\itr_{\max}$, and does so significantly. The coefficient vectors produced by \ac{ssc}-\ac{mp} hence lead to more favorable conditions for the spectral clustering step than those produced by \ac{ssc}-\ac{omp}.

We further observe that the \ac{tpr} and \ac{fpr} of \ac{ssc}-\ac{mp}, with $\itr_{\max} = \infty$ (i.e., stopping is activated by $\norm[0]{\vb_j}=\mpspar$), as a function of $\mpspar$, increase at the same rate as the \ac{tpr} and \ac{fpr} of \ac{ssc}-\ac{omp} as a function of $\itr_{\max}$. 
However, the ratio of \ac{tp}- and \ac{fp}-$\ell_1$-norms for \ac{ssc}-\ac{mp} 
significantly exceeds that for \ac{ssc}-\ac{omp} for all values of $u$ ($u$ is the variable on the $x$-axis in Figure \ref{fig:fracesrprt} and corresponds to $\itr_{\max}$ for \ac{ssc}-\ac{omp} and $\mpspar$ for \ac{ssc}-\ac{mp}). 
In other words, even when we force the representations computed by \ac{ssc}-\ac{mp} and \ac{ssc}-\ac{omp} to have the same sparsity level, thereby discounting the advantage \ac{mp} has through its ability to reselect data points, \ac{ssc}-\ac{mp} still produces weight assignments in $G$ that are more favorable in terms of spectral clustering. 
Indeed, in both the face clustering and the synthetic data clustering problem the \ac{ce} incurred by \ac{ssc}-\ac{omp} significantly exceeds that of \ac{ssc}-\ac{mp} for most values of $u$. 
In summary, this indicates that when we enforce the same target sparsity level for \ac{ssc}-\ac{mp} and \ac{ssc}-\ac{omp}, \ac{ssc}-\ac{mp} is less sensitive to the choice of the sparsity level than \ac{ssc}-\ac{omp}. 

\revb{We finally note that while in the noisy case \ac{ssc}-\ac{mp} is much more robust than \ac{ssc}-\ac{omp} w.r.t. the choice of the parameters for \ac{di}-stopping, in the noiseless case \ac{ssc}-\ac{omp} yields slightly lower \ac{ce} than \ac{ssc}-\ac{mp} for \ac{di}-stopping if the subspace affinities are large. This matter is investigated numerically in Appendix \ref{sec:noiseless}.} 

\renewcommand{\plotwidth}{0.45\textwidth}
\renewcommand{\plotheight}{0.35\textwidth}
\renewcommand{\plotseph}{1.5cm}
\renewcommand{\plotsepv}{2.5cm}
\newcommand{\xlabis}{u}

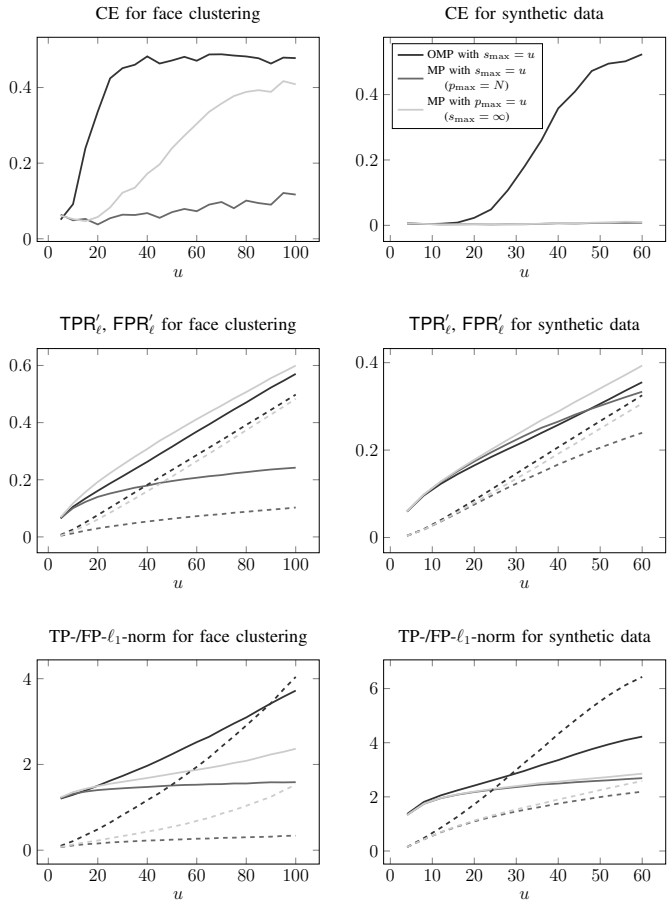
\begin{figure}[t!] 
\hspace{-0.6cm}   
    \centering
\begin{tikzpicture}[scale=0.57, /tikz/font=\large] 
    \begin{axis}[ name=cefaces,title ={CE for face clustering},
            xlabel=$\xlabis$,
    width =\plotwidth,
    height =\plotheight]
    \addplot +[line width=\pllw,mark=none,solid,\plotblue] table[x index=0,y index=1]{./data/itersens/CEs-itersens_faces.dat};
    \addplot +[line width=\pllw,mark=none,solid,\plotred] table[x index=0,y index=2]{./data/itersens/CEs-itersens_faces.dat};
    \addplot +[line width=\pllw,mark=none,solid,\plotgreen] table[x index=0,y index=3]{./data/itersens/CEs-itersens_faces.dat};
    
    \end{axis}
    
    \begin{axis}[ name=cesynth,title ={CE for synthetic data},anchor=west,at={($(cefaces.east)+(\plotseph,0)$)},
            xlabel=$\xlabis$,
    width =\plotwidth,
    height =\plotheight,
    legend entries={ {OMP with $\itr_{\max} = \xlabis$}, {\shortstack{MP with $\itr_{\max} = \xlabis$ \\ ($\mpspar = N$)}}, {\shortstack{MP with $\mpspar = \xlabis$ \\ ($\itr_{\max} = \infty$)}}} ,
    legend style={
            cells={anchor=west},
                    legend pos= north west,
                    font=\footnotesize,}
    ]
    \addplot +[line width=\pllw,mark=none,solid,\plotblue] table[x index=0,y index=1]{./data/itersens/CEs-itersens_synth_0.50.dat};
    \addplot +[line width=\pllw,mark=none,solid,\plotred] table[x index=0,y index=2]{./data/itersens/CEs-itersens_synth_0.50.dat};
    \addplot +[line width=\pllw,mark=none,solid,\plotgreen] table[x index=0,y index=3]{./data/itersens/CEs-itersens_synth_0.50.dat};
        
    \end{axis}
    
     \begin{axis}[ name=tpfpfaces,title ={$\tpr'_\l$, $\fpr'_\l$ for face clustering},anchor=north,at={($(cefaces.south)+(0,-\plotsepv)$)},
            xlabel=$\xlabis$,
    width =\plotwidth,
    height =\plotheight]
    \addplot +[line width=\pllw,mark=none,solid,\plotblue] table[x index=0,y index=1]{./data/itersens/TPFPs-itersens_faces.dat};
    \addplot +[line width=\pllw,mark=none,solid,\plotred] table[x index=0,y index=2]{./data/itersens/TPFPs-itersens_faces.dat};
    \addplot +[line width=\pllw,mark=none,solid,\plotgreen] table[x index=0,y index=3]{./data/itersens/TPFPs-itersens_faces.dat};
    \addplot +[line width=\pllw,mark=none,dashed,\plotblue] table[x index=0,y index=4]{./data/itersens/TPFPs-itersens_faces.dat};
    \addplot +[line width=\pllw,mark=none,dashed,\plotred] table[x index=0,y index=5]{./data/itersens/TPFPs-itersens_faces.dat};
    \addplot +[line width=\pllw,mark=none,dashed,\plotgreen] table[x index=0,y index=6]{./data/itersens/TPFPs-itersens_faces.dat};
    
    \end{axis}
    
    \begin{axis}[ name=tpfpsynth,title ={$\tpr'_\l$, $\fpr'_\l$ for synthetic data},anchor=north,at={($(cesynth.south)+(0,-\plotsepv)$)},
            xlabel=$\xlabis$,
    width =\plotwidth,
    height =\plotheight]
     \addplot +[line width=\pllw,mark=none,solid,\plotblue] table[x index=0,y index=1]{./data/itersens/TPFPs-itersens_synth_0.50.dat};
    \addplot +[line width=\pllw,mark=none,solid,\plotred] table[x index=0,y index=2]{./data/itersens/TPFPs-itersens_synth_0.50.dat};
    \addplot +[line width=\pllw,mark=none,solid,\plotgreen] table[x index=0,y index=3]{./data/itersens/TPFPs-itersens_synth_0.50.dat};
    \addplot +[line width=\pllw,mark=none,dashed,\plotblue] table[x index=0,y index=4]{./data/itersens/TPFPs-itersens_synth_0.50.dat};
    \addplot +[line width=\pllw,mark=none,dashed,\plotred] table[x index=0,y index=5]{./data/itersens/TPFPs-itersens_synth_0.50.dat};
    \addplot +[line width=\pllw,mark=none,dashed,\plotgreen] table[x index=0,y index=6]{./data/itersens/TPFPs-itersens_synth_0.50.dat};
    
    \end{axis}
    
     \begin{axis}[ name=tfnorm,title ={\ac{tp}-/\ac{fp}-$\ell_1$-norm for face clustering},anchor=north,at={($(tpfpfaces.south)+(0,-\plotsepv)$)},
            xlabel=$\xlabis$,
    width =\plotwidth,
    height =\plotheight]
    \addplot +[line width=\pllw,mark=none,solid,\plotblue] table[x index=0,y index=1]{./data/itersens/TF1norm-itersens_faces.dat};
    \addplot +[line width=\pllw,mark=none,solid,\plotred] table[x index=0,y index=2]{./data/itersens/TF1norm-itersens_faces.dat};
    \addplot +[line width=\pllw,mark=none,solid,\plotgreen] table[x index=0,y index=3]{./data/itersens/TF1norm-itersens_faces.dat};
    \addplot +[line width=\pllw,mark=none,dashed,\plotblue] table[x index=0,y index=4]{./data/itersens/TF1norm-itersens_faces.dat};
    \addplot +[line width=\pllw,mark=none,dashed,\plotred] table[x index=0,y index=5]{./data/itersens/TF1norm-itersens_faces.dat};
    \addplot +[line width=\pllw,mark=none,dashed,\plotgreen] table[x index=0,y index=6]{./data/itersens/TF1norm-itersens_faces.dat};
    
    \end{axis}
   
   \begin{axis}[ name=cesynth,title ={\ac{tp}-/\ac{fp}-$\ell_1$-norm for synthetic data},anchor=north,at={($(tpfpsynth.south)+(0,-\plotsepv)$)},
            xlabel=$\xlabis$,
    width =\plotwidth,
    height =\plotheight]
    \addplot +[line width=\pllw,mark=none,solid,\plotblue] table[x index=0,y index=1]{./data/itersens/TF1norm-itersens_synth_0.50.dat};
    \addplot +[line width=\pllw,mark=none,solid,\plotred] table[x index=0,y index=2]{./data/itersens/TF1norm-itersens_synth_0.50.dat};
    \addplot +[line width=\pllw,mark=none,solid,\plotgreen] table[x index=0,y index=3]{./data/itersens/TF1norm-itersens_synth_0.50.dat};
    \addplot +[line width=\pllw,mark=none,dashed,\plotblue] table[x index=0,y index=4]{./data/itersens/TF1norm-itersens_synth_0.50.dat};
    \addplot +[line width=\pllw,mark=none,dashed,\plotred] table[x index=0,y index=5]{./data/itersens/TF1norm-itersens_synth_0.50.dat};
    \addplot +[line width=\pllw,mark=none,dashed,\plotgreen] table[x index=0,y index=6]{./data/itersens/TF1norm-itersens_synth_0.50.dat};
    
    \end{axis}
    
\end{tikzpicture}
\caption[TP/FP tradeoff]{\label{fig:fracesrprt} Clustering performance of \ac{ssc}-\ac{omp} and \ac{ssc}-\ac{mp} for \ac{di}-stopping, as a function of $\itr_{\max}$ and $\mpspar$, respectively. Middle row: solid lines: \ac{tpr}, dashed lines: \ac{fpr}. Bottom row: solid lines: \ac{tp}-$\ell_1$-norm, dashed lines: \ac{fp}-$\ell_1$-norm.
}
\end{figure}

\section*{Acknowledgments}

The authors would like to thank Reinhard Heckel and Martin Jaggi for insightful discussions.

\FloatBarrier

\appendices

\revb{
\section{Influence of $\itr_{\max}$ and $\mpspar$ in \ac{di}-stopping for noiseless data}
\label{sec:noiseless}

We compare the influence of $\itr_{\max}$ and $\mpspar$ on the performance of \ac{ssc}-\ac{omp} and \ac{ssc}-\ac{mp} with \ac{di}-stopping and for noiseless data. To this end, we generate data lying in a union of three subspaces as described in Section \ref{sec:inflsmax} (with $\sigma = 0$), considering the pairs $(5,3)$, $(10,3)$, and $(10,6)$ for $(t,\rho)$, where $t$ denotes the number of dimensions which the three subspaces intersect in. Figure \ref{fig:noiseless} shows the \ac{ce} along with the quantities \ac{tpr}/\ac{fpr} and the \ac{tp}-/\ac{fp}-$\ell_1$-norm as a function of $\itr_{\max}$ (or $\mpspar$ for \ac{ssc}-\ac{mp} if the maximum sparsity level is used as stopping criterion) in the range $\{1, \ldots, 2d\}$. We observe that \ac{ssc}-\ac{omp} yields a slightly lower \ac{ce} than \ac{ssc}-\ac{mp} for $t=10$. While \ac{ssc}-\ac{mp} yields a higher \ac{tpr} and a lower \ac{fpr} than \ac{ssc}-\ac{omp} for most of the values of $t$, $\rho$, and $\itr_{\max}$ or $\mpspar$, \ac{ssc}-\ac{mp} assigns smaller values, than \ac{ssc}-\ac{omp}, to entries in the adjacency matrix $\mA$ corresponding to the true connections (see the plot in the last row, left, in Figure \ref{fig:noiseless}). This arguably leads to the slightly higher \ac{ce} of \ac{ssc}-\ac{mp} compared to \ac{ssc}-\ac{omp} for $t=10$.

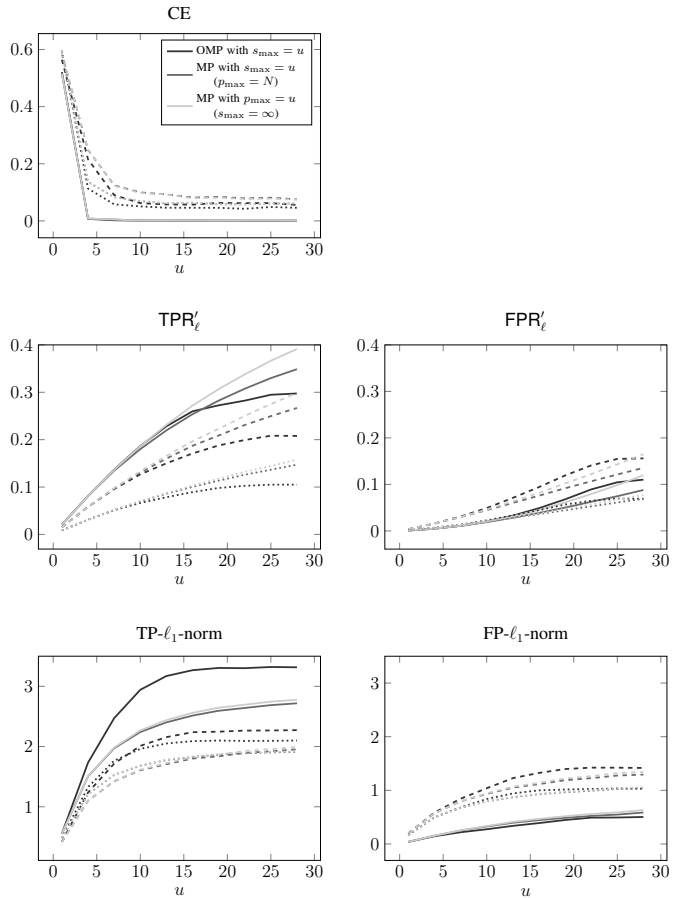
\begin{figure}[t!]
\hspace{-0.2cm}      
    \centering
\begin{tikzpicture}[scale=0.57, /tikz/font=\large] 
    
    \begin{axis}[ name=cesynth,title ={CE},
            xlabel=$\xlabis$,
    width =\plotwidth,
    height =\plotheight,
    legend entries={ {OMP with $\itr_{\max} = \xlabis$}, {\shortstack{MP with $\itr_{\max} = \xlabis$ \\ ($\mpspar = N$)}}, {\shortstack{MP with $\mpspar = \xlabis$ \\ ($\itr_{\max} = \infty$)}}} ,
    legend style={
            cells={anchor=west},
                    legend pos= north east,
                    font=\footnotesize,}
    ]
    \addplot +[mark=none,solid,line width=\pllw,\plotblue] table[x index=0,y index=1]{./data/noiseless/CEs-itersens_noiseless_synth_5_3.dat};
    \addplot +[mark=none,solid,line width=\pllw,\plotred] table[x index=0,y index=2]{./data/noiseless/CEs-itersens_noiseless_synth_5_3.dat};
    \addplot +[mark=none,solid,line width=\pllw,\plotgreen] table[x index=0,y index=3]{./data/noiseless/CEs-itersens_noiseless_synth_5_3.dat};
    
    \addplot +[mark=none,dashed,line width=\pllw,\plotblue] table[x index=0,y index=1]{./data/noiseless/CEs-itersens_noiseless_synth_10_3.dat};
    \addplot +[mark=none,dashed,line width=\pllw,\plotred] table[x index=0,y index=2]{./data/noiseless/CEs-itersens_noiseless_synth_10_3.dat};
    \addplot +[mark=none,dashed,line width=\pllw,\plotgreen] table[x index=0,y index=3]{./data/noiseless/CEs-itersens_noiseless_synth_10_3.dat};
    
    \addplot +[mark=none,dotted,line width=\pllw,\plotblue] table[x index=0,y index=1]{./data/noiseless/CEs-itersens_noiseless_synth_10_6.dat};
    \addplot +[mark=none,dotted,line width=\pllw,\plotred] table[x index=0,y index=2]{./data/noiseless/CEs-itersens_noiseless_synth_10_6.dat};
    \addplot +[mark=none,dotted,line width=\pllw,\plotgreen] table[x index=0,y index=3]{./data/noiseless/CEs-itersens_noiseless_synth_10_6.dat};
        
    \end{axis}
    
     \begin{axis}[ name=tpfpfaces,title ={$\tpr'_\l$},anchor=north,at={($(cesynth.south)+(0,-\plotsepv)$)}, ymax=0.4,
            xlabel=$\xlabis$,
    width =\plotwidth,
    height =\plotheight]
    \addplot +[mark=none,solid,line width=\pllw,\plotblue] table[x index=0,y index=1]{./data/noiseless/TPFPs-itersens_noiseless_synth_5_3.dat};
    \addplot +[mark=none,solid,line width=\pllw,\plotred] table[x index=0,y index=2]{./data/noiseless/TPFPs-itersens_noiseless_synth_5_3.dat};
    \addplot +[mark=none,solid,line width=\pllw,\plotgreen] table[x index=0,y index=3]{./data/noiseless/TPFPs-itersens_noiseless_synth_5_3.dat};
    
 \addplot +[mark=none,dashed,line width=\pllw,\plotblue] table[x index=0,y index=1]{./data/noiseless/TPFPs-itersens_noiseless_synth_10_3.dat};
    \addplot +[mark=none,dashed,line width=\pllw,\plotred] table[x index=0,y index=2]{./data/noiseless/TPFPs-itersens_noiseless_synth_10_3.dat};
    \addplot +[mark=none,dashed,line width=\pllw,\plotgreen] table[x index=0,y index=3]{./data/noiseless/TPFPs-itersens_noiseless_synth_10_3.dat};
   
    \addplot +[mark=none,dotted,line width=\pllw,\plotblue] table[x index=0,y index=1]{./data/noiseless/TPFPs-itersens_noiseless_synth_10_6.dat};
    \addplot +[mark=none,dotted,line width=\pllw,\plotred] table[x index=0,y index=2]{./data/noiseless/TPFPs-itersens_noiseless_synth_10_6.dat};
    \addplot +[mark=none,dotted,line width=\pllw,\plotgreen] table[x index=0,y index=3]{./data/noiseless/TPFPs-itersens_noiseless_synth_10_6.dat};
    
    \end{axis}
    
    \begin{axis}[ name=tpfpsynth,title ={$\fpr'_\l$},anchor=west,at={($(tpfpfaces.east)+(\plotseph,0)$)}, ymax=0.4,
            xlabel=$\xlabis$,
    width =\plotwidth,
    height =\plotheight]
     \addplot +[mark=none,solid,line width=\pllw,\plotblue] table[x index=0,y index=4]{./data/noiseless/TPFPs-itersens_noiseless_synth_5_3.dat};
    \addplot +[mark=none,solid,line width=\pllw,\plotred] table[x index=0,y index=5]{./data/noiseless/TPFPs-itersens_noiseless_synth_5_3.dat};
    \addplot +[mark=none,solid,line width=\pllw,\plotgreen] table[x index=0,y index=6]{./data/noiseless/TPFPs-itersens_noiseless_synth_5_3.dat};
    
 \addplot +[mark=none,dashed,line width=\pllw,\plotblue] table[x index=0,y index=4]{./data/noiseless/TPFPs-itersens_noiseless_synth_10_3.dat};
    \addplot +[mark=none,dashed,line width=\pllw,\plotred] table[x index=0,y index=5]{./data/noiseless/TPFPs-itersens_noiseless_synth_10_3.dat};
    \addplot +[mark=none,dashed,line width=\pllw,\plotgreen] table[x index=0,y index=6]{./data/noiseless/TPFPs-itersens_noiseless_synth_10_3.dat};
   
    \addplot +[mark=none,dotted,line width=\pllw,\plotblue] table[x index=0,y index=4]{./data/noiseless/TPFPs-itersens_noiseless_synth_10_6.dat};
    \addplot +[mark=none,dotted,line width=\pllw,\plotred] table[x index=0,y index=5]{./data/noiseless/TPFPs-itersens_noiseless_synth_10_6.dat};
    \addplot +[mark=none,dotted,line width=\pllw,\plotgreen] table[x index=0,y index=6]{./data/noiseless/TPFPs-itersens_noiseless_synth_10_6.dat};
    
    \end{axis}
    
     \begin{axis}[ name=tfnorm,title ={TP-$\ell_1$-norm},anchor=north,at={($(tpfpfaces.south)+(0,-\plotsepv)$)}, ymax=3.5,
            xlabel=$\xlabis$,
    width =\plotwidth,
    height =\plotheight]
    \addplot +[mark=none,solid,line width=\pllw,\plotblue] table[x index=0,y index=1]{./data/noiseless/TF1norm-itersens_noiseless_synth_5_3.dat};
    \addplot +[mark=none,solid,line width=\pllw,\plotred] table[x index=0,y index=2]{./data/noiseless/TF1norm-itersens_noiseless_synth_5_3.dat};
    \addplot +[mark=none,solid,line width=\pllw,\plotgreen] table[x index=0,y index=3]{./data/noiseless/TF1norm-itersens_noiseless_synth_5_3.dat};
    
 \addplot +[mark=none,dashed,line width=\pllw,\plotblue] table[x index=0,y index=1]{./data/noiseless/TF1norm-itersens_noiseless_synth_10_3.dat};
    \addplot +[mark=none,dashed,line width=\pllw,\plotred] table[x index=0,y index=2]{./data/noiseless/TF1norm-itersens_noiseless_synth_10_3.dat};
    \addplot +[mark=none,dashed,line width=\pllw,\plotgreen] table[x index=0,y index=3]{./data/noiseless/TF1norm-itersens_noiseless_synth_10_3.dat};
   
    \addplot +[mark=none,dotted,line width=\pllw,\plotblue] table[x index=0,y index=1]{./data/noiseless/TF1norm-itersens_noiseless_synth_10_6.dat};
    \addplot +[mark=none,dotted,line width=\pllw,\plotred] table[x index=0,y index=2]{./data/noiseless/TF1norm-itersens_noiseless_synth_10_6.dat};
    \addplot +[mark=none,dotted,line width=\pllw,\plotgreen] table[x index=0,y index=3]{./data/noiseless/TF1norm-itersens_noiseless_synth_10_6.dat};
    
    \end{axis}
   
   \begin{axis}[ name=cesynth,title ={FP-$\ell_1$-norm},anchor=west,at={($(tfnorm.east)+(\plotseph,0)$)}, ymax=3.5,
            xlabel=$\xlabis$,
    width =\plotwidth,
    height =\plotheight ]
    \addplot +[mark=none,solid,line width=\pllw,\plotblue] table[x index=0,y index=4]{./data/noiseless/TF1norm-itersens_noiseless_synth_5_3.dat};
    \addplot +[mark=none,solid,line width=\pllw,\plotred] table[x index=0,y index=5]{./data/noiseless/TF1norm-itersens_noiseless_synth_5_3.dat};
    \addplot +[mark=none,solid,line width=\pllw,\plotgreen] table[x index=0,y index=6]{./data/noiseless/TF1norm-itersens_noiseless_synth_5_3.dat};
    
 \addplot +[mark=none,dashed,line width=\pllw,\plotblue] table[x index=0,y index=4]{./data/noiseless/TF1norm-itersens_noiseless_synth_10_3.dat};
    \addplot +[mark=none,dashed,line width=\pllw,\plotred] table[x index=0,y index=5]{./data/noiseless/TF1norm-itersens_noiseless_synth_10_3.dat};
    \addplot +[mark=none,dashed,line width=\pllw,\plotgreen] table[x index=0,y index=6]{./data/noiseless/TF1norm-itersens_noiseless_synth_10_3.dat};
   
    \addplot +[mark=none,dotted,line width=\pllw,\plotblue] table[x index=0,y index=4]{./data/noiseless/TF1norm-itersens_noiseless_synth_10_6.dat};
    \addplot +[mark=none,dotted,line width=\pllw,\plotred] table[x index=0,y index=5]{./data/noiseless/TF1norm-itersens_noiseless_synth_10_6.dat};
    \addplot +[mark=none,dotted,line width=\pllw,\plotgreen] table[x index=0,y index=6]{./data/noiseless/TF1norm-itersens_noiseless_synth_10_6.dat};
    
    \end{axis}
    
\end{tikzpicture}
\caption[TP/FP tradeoff]{\label{fig:noiseless} \revb{Clustering performance of \ac{ssc}-\ac{omp} and \ac{ssc}-\ac{mp} for \ac{di}-stopping, as a function of $\itr_{\max}$ and $\mpspar$ for noiseless data. Solid line: $t=5$, $\rho=3$; dashed line: $t=10$, $\rho=3$; dotted line: $t=10$, $\rho=6$.
}}
\end{figure}
}

\section{Proof of Theorem \ref{th:SSCOMP}} \label{sec:pfsscomp}

Throughout the proof, we shall use $\Yl \defeq \Xl + \Zl = \Ul \Al + \Zl$, $\Xl \in \reals^{m \times n_\l}$, $\Al \in \reals^{d_\l \times n_\l}$, and $\Zl \in \reals^{m \times n_\l}$ to denote the matrices whose columns are the $\yl_i$, $\xl_i$, $\al_i$, and $\zl_i$, $i \in [n_\l]$, respectively. 
Furthermore, $\Pp \defeq \Ul \transp{\Ul}$\label{def:pppo} and $\Po \defeq \mI - \Ul \transp{\Ul}$ stand for the orthogonal projection onto $\cS_\l$ and its orthogonal complement (in $\reals^m$) $\cS_\l^\po$, respectively. We do not indicate the dependence of $\Pp$ and $\Po$ on $\l$ as this is always clear from the context. For $\vv \in \reals^m$ and $\mA \in \reals^{m \times n}$, we use the shorthands\vspace{0.05cm} $\vv_\pp \defeq \Pp \vv$, $\mA_\pp \defeq \Pp \mA$, $\vv_\po \defeq \Po \vv$, and $\mA_\po \defeq \Po \mA$. Further, $\atl_i \in \reals^{d_\l}$ denotes the coefficients of $\yl_{i \pp}$ in the basis $\Ul$,\vspace{0.05cm} i.e., $\yl_{i \pp} = \Ul \atl_i$, and similarly $\Yl_\pp = \Pp \Yl = \Ul \Atl$. Note that the distribution of $\atl_i = \al_i + \transp{\Ul} \zl_i$ is rotationally invariant as $\al_i$ and $\transp{\Ul} \zl_i$ are statistically independent and both have rotationally invariant distributions. Finally, $\rs(\vx, \mD)$ and $\Lambda_\itr(\vx,\mD)$ denote the residual and the index set, respectively, after iteration $\itr$, obtained by \ac{omp} applied to $\vx$ with the columns of $\mD$ as dictionary elements.

If $\min_{\l \in [L]} \{c_s d_\l / \log( (n_\l - 1) e / s_{\max})\} < 1$, then the condition in Theorem \ref{th:SSCOMP} admits zero \ac{omp} iterations, i.e., the graph $G$ delivered by \ac{ssc}-\ac{omp} has an empty edge set and thereby trivially no false connections. 
We therefore consider the case $1 \leq \itr_{\max} \leq \min_{\l \in [L]} \{c_s d_\l / \log( (n_\l - 1) e / s_{\max})\}$ in the remainder of the proof. 

The graph $G$ obtained by \ac{ssc}-\ac{omp} has no false connections 
if for each $\yl_i \in \cY_\l$, for all $\l \in [L]$, the \ac{omp} algorithm selects points from $\cY_\l$ in all $\itr_\mathrm{max}$ iterations.\footnote{
For \ac{di}-stopping the \ac{omp} algorithm terminates w.p. $1$ after exactly $\itr_{\max}$ iterations as in our data model the points in $\cY$ are in general position w.p. 1 and $\itr_{\max} < \min_{\l \in [L]} c_s d_\l < \min\{m,N-1\}$ by the condition on $\itr_{\max}$ in Theorem \ref{th:SSCOMP}.

} 
Now, the \ac{omp} selection rule \eqref{eq:OMPSelRule} for $\yl_i$ implies that \ac{omp} selects a point from $\cY_\l$ in the $(\itr+1)$-st iteration if
\begin{equation}\label{eq:CorrSelCond}
\max_{k \neq \l, j} \abs{\innerprod{\yk_j}{\rs}} < \max_{j \in [n_\l] \backslash (\Lambda_{\itr}\cup \{i\})} \abs{\innerprod{\yl_j}{\rs}},
\end{equation}
where $\max_{k \neq \l, j}$ denotes maximization over subspaces $k \in [L]$, $k \neq \l$, and over the indices $j$ of the points $\yk_j \in \cY_k$ in these subspaces. Hence, the graph $G$ obtained by \ac{ssc}-\ac{omp} satisfies the \ac{nfc} property if \eqref{eq:CorrSelCond} holds for all $\itr_{\max}$ iterations, for each $\yl_i \in \cY_\l$, for all $\l \in [L]$. We now establish that this holds for our statistical data model w.p. at least $P^\star$ as defined in \eqref{eq:pstar}. 

Our analysis will be based on an auxiliary algorithm termed ``reduced \ac{omp}'', which has access to the reduced dictionary $\cY_\l \backslash \{ \yl_i \}$ only---instead of the full dictionary $\cY \backslash \{ \yl_i \}$---to represent $\yl_i$. We henceforth use the shorthands $\rsl$ for the residuals $\rs(\yl_i, \Yl_{-i})$ corresponding to reduced \ac{omp}. The dependence of the $\rsl$ on the index $i$ of the data point $\yl_i$ to be represented is not made explicit for notational ease. If the $\rsl$ satisfy \eqref{eq:CorrSelCond} for \emph{all iterations} $\itr \in [\itr_{\max}]$, the reduced OMP algorithm and the original OMP algorithm select exactly the same data points and do so in exactly the same order. In this case, we also have $\rsl = \rs$, for all $\itr \in [\itr_{\max}]$. As $\rsl$ satisfying \eqref{eq:CorrSelCond} for all $\itr \in [\itr_{\max}]$ is necessary and sufficient for $\rs$ to satisfy \eqref{eq:CorrSelCond} for all $\itr \in [\itr_{\max}]$, a lower bound $P^\star$ on the probability of $\rsl$ satisfying \eqref{eq:CorrSelCond} for all $\itr \in [\itr_{\max}]$ also constitutes a lower bound on the probability of $\rs$ satisfying \eqref{eq:CorrSelCond} for all $\itr \in [\itr_{\max}]$. Working with the reduced OMP algorithm is beneficial as $\rsl$ is a function of the points in $\cY_\l$ only and is therefore statistically independent of the points in $\cY \backslash \cY_\l$. This is significant as it will allow us to apply standard concentration of measure inequalities for independent random variables. In the remainder of the proof, we work with reduced \ac{omp} exclusively.

We start by providing intuition on the proof idea. To this end, we expand the inner products in \eqref{eq:CorrSelCond} according to 
\begin{align}
\innerprod{\yk_j}{\rsl} &= \innerprod{\xk_j+\zk_j}{\rpsl + \rosl} \nonumber \\
&= \innerprod{\xk_j}{\rpsl} + \innerprod{\xk_j}{\rosl} \nonumber \\ &\qquad+ \innerprod{\zk_j}{\rpsl} + \innerprod{\zk_j}{\rosl}. \label{eq:innerprodexp}
\end{align} 
The first term in \eqref{eq:innerprodexp} quantifies the similarity of the portions of $\yk_j$ and $\rsl$ that lie in $\cS_k$ and $\cS_\l$, respectively, i.e., the ``signal components'' of $\yk_j$ and $\rsl$, while the other terms all account for interactions with or between ``undesired components'' residing in $\cS_k^\po$, $\cS_\l^\po$. 
If the similarities---in absolute value---of the ``signal components'' for $k = \l$, $j \in [n_\l] \backslash ( \Lambda_\itr \cup \{ i \})$, are sufficiently large relative to those for $k \neq \l$, $j \in [n_k]$, and if the interactions of all ``undesired components'' are sufficiently small, for $k, \l \in [L]$, then \eqref{eq:CorrSelCond} holds. Following \cite[Proofs of Thm. 3, Cor. 1]{heckel2015dimensionality} this intuition will be made quantitative and rigorous by introducing events that, when conditioned on, yield bounds on the absolute values of the individual terms in \eqref{eq:innerprodexp} that are of analytically amenable form. These bounds will then be employed to derive an upper bound on the \ac{lhs} and a lower bound on the \ac{rhs} of \eqref{eq:CorrSelCond} that both hold conditionally on the intersection of the underlying events. Based on these bounds, we will then show that the clustering condition \eqref{eq:ClusCondThm} implies \eqref{eq:CorrSelCond} w.p. at least $P^\star$. The particular choice of the events that we condition on is delicate, but when done properly, allows us to make the statistical dependencies between $\rsl$ and $\yl_j$, $j \in [n_\l] \backslash \{i\}$, analytically tractable. 
We finally note that although the general idea of conditioning on suitably defined events is taken from previous work by the authors \cite[Proofs of Thm.~3, Cor.~1]{heckel2015dimensionality}, the choice of the specific events as well as other technical aspects of the present proof differ significantly from \cite[Proofs of Thm.~3, Cor.~1]{heckel2015dimensionality}.

We commence the formal proof by upper-bounding the \ac{lhs} of \eqref{eq:CorrSelCond} according to
\begin{align}
&\max_{k \neq \l, j} \abs{\innerprod{\xk_j+\zk_j}{\rpsl + \rosl}} \nonumber \\
	&\qquad\qquad\leq \max_{k \neq \l, j} \abs{\innerprod{\xk_j}{\rpsl}}
	+ \max_{k \neq \l, j} \abs{\innerprod{\xk_j}{\rosl}} \nonumber \\
	&\qquad\qquad\qquad+\max_{k \neq \l, j} \abs{\innerprod{\zk_j}{\rsl}} \nonumber \\
	&\qquad\qquad\leq 4 \log( N^3 \itr_{\max}) \frac{\norm[F]{\transp{\Uk} \Ul}}{\sqrt{d_k} \sqrt{d_\l}} \norm[2]{\rpsl} \nonumber \\
	&\qquad\qquad\qquad+ \frac{\sqrt{2 \log( N^3 \itr_{\max})}}{\sqrt{m - d_\l}} \norm[2]{\rosl} \nonumber \\
	&\qquad\qquad\qquad+ \frac{\sqrt{2 \log( N^3 \itr_{\max})}}{\sqrt{m}} \frac{3}{2} \sigma \left(1 + \frac{3}{2} \sigma \right), \label{eq:CorrSelCondLHSUB}
\end{align}
where the second inequality holds on the event $\Ea \cap \Eb \cap \Ec \cap \Ed$ with
\begin{align}
\Ea &\defeq \vast\{ \max_{k \neq \l, j} \abs{\innerprod{\xk_j}{\rpsl}} \nonumber \\
&\quad\qquad \leq 4 \log( N^3 \itr_{\max}) \frac{\norm[F]{\transp{\Uk} \Ul}}{\sqrt{d_k} \sqrt{d_\l}} \norm[2]{\rpsl} \vast\}, \label{eq:DefEa} \\
\Eb &\defeq \left\{ \max_{k \neq \l, j} \abs{\innerprod{\xk_j}{\rosl}} \! \leq \! \frac{\sqrt{2 \log( N^3 \itr_{\max})}}{\sqrt{m - d_\l}} \norm[2]{\rosl} \right\}\!\!, \label{eq:DefEb} \\
\Ec &\defeq \Bigg\{ \max_{k \neq \l, j} \abs{\innerprod{\zk_j}{\rsl}} \nonumber \\
&\quad \leq \frac{\sqrt{2 \log( N^3 \itr_{\max})}}{\sqrt{m}} \left(1 + \norm[2]{\zl_i} \right) \max_{k \neq \l, j} \norm[2]{\zk_j} \Bigg\}\!, \label{eq:DefEc} \\ 
\Ed &\defeq \Bigg\{ \left\{ \norm[2]{\zl_{j \pp}} \leq \frac{3}{2} \frac{ \sqrt{d_\l}}{\sqrt{m}} \sigma \right\} \cap \left\{ \norm[2]{\zl_{j}} \leq \frac{3}{2} \sigma \right\}, \nonumber \\
&\qquad\qquad \forall \l \in [L], j \in [n_\l] \Bigg\}. \label{eq:DefEd}
\end{align}
Note that the dependence of $\Ea$ and $\Eb$ on $i$ is due to $\rpsl$ and $\rosl$, both of which are functions of $\yl_i$. 
Here, $\Ea$ pertains to the similarities of the ``signal components'' for $k \neq \l$, $\Eb$ and $\Ec$ quantify the similarity of ``undesired components'', and $\Ed$ controls the magnitude of the ``undesired components'' of the $\yl_j$.

We proceed by lower-bounding the \ac{rhs} of \eqref{eq:CorrSelCond}.\vspace{0.05cm} Using $\norm[2]{\smash{\rsl}} = \norm[2]{\smash{(\mI - \Yl_{\Lambda_{\itr}} \pinv{(\Yl_{\Lambda_{\itr}})}) \yl_i}} \leq \norm[2]{\smash{\yl_i}} \leq 1 + \norm[2]{\smash{\zl_i}}$\vspace{0.05cm} (where the inequality is thanks to $\norm[2]{\smash{\xl_i}} = 1$ and $\mI - \Yl_{\Lambda_{\itr}} \pinv{(\Yl_{\Lambda_{\itr}})}$ being an orthogonal projection matrix), we find that on the event $\Ed \cap \Ee$ with
\begin{align}
 \Ee \defeq &\Bigg\{ \max_{j \in [n_\l] \backslash (\Lambda_{\itr}\cup \{i\})} \abs{\innerprod{\yl_j}{\rsl}} \nonumber \\
 & \qquad \geq \left(1-\frac{c_4+1}{\sqrt{\rho_\l}}\right)\frac{\norm[2]{\rpsl}}{\sqrt{d_\l}} \nonumber \\ 
 & \qquad \qquad - \sigma \left(\frac{1}{\sqrt{m}} + \frac{2}{\sqrt{n_\l - 1}} \right) \norm[2]{\rsl} \Bigg\}, \label{eq:DefEe}
\end{align}
where $c_4 > 0$ is the numerical constant in Lemma \ref{le:RHSLB}, 
the \ac{rhs} of \eqref{eq:CorrSelCond} obeys
\begin{align}
 &\max_{j \in [n_\l] \backslash (\Lambda_{\itr}\cup \{i\})} \abs{\innerprod{\yl_j}{\rsl}} \geq \left(1-\frac{c_4+1}{\sqrt{\rho_\l}}\right)\frac{\norm[2]{\rpsl}}{\sqrt{d_\l}} \nonumber \\
 & \qquad\qquad\qquad - \sigma \left(\frac{1}{\sqrt{m}} + \frac{2}{\sqrt{n_\l - 1}} \right) \left(1 + \frac{3}{2} \sigma \right). \label{eq:CorrSelCondRHSLB}
\end{align}
On $\Ea \cap \Eb \cap \Ec \cap \Ed \cap \Ee$, \eqref{eq:CorrSelCond} is now implied by [\ac{rhs} of \eqref{eq:CorrSelCondLHSUB}] $<$ [\ac{rhs} of \eqref{eq:CorrSelCondRHSLB}]; multiplying this inequality by $\sqrt{d_\l}/(4 \log (N^3s_{\max}) \norm[2]{\smash{\rpsl}})$, we get 
\begin{align}
&\underbrace{\frac{\norm[F]{\transp{\Uk} \Ul}}{\sqrt{d_k}}}_{\leq \underset{k \colon k \neq \l}{\max} \aff (\cS_k,\cS_\l)}
+ \frac{1}{\sqrt{8 \log( N^3 \itr_{\max})}}
\Bigg( \underbrace{\frac{\sqrt{d_\l}}{\sqrt{m - d_\l}}}_{\leq \sqrt{2 d_\l}/\sqrt{m}} \frac{\norm[2]{\rosl}}{\norm[2]{\rpsl}} \nonumber \\
& \hspace{4cm}+ \frac{ \sigma}{\norm[2]{\rpsl}} \frac{\sqrt{d_\l}}{\sqrt{m}} \frac{3}{2} \left(1 + \frac{3}{2} \sigma \right)\Bigg) \nonumber \\
& \quad < \frac{1}{4 \log( N^3 \itr_{\max})}
\Bigg( \underbrace{\left(1-\frac{c_4+1}{\sqrt{\rho_\l}}\right)}_{\geq 1/2} \nonumber \\
&\hspace{1.1cm}- \frac{\sigma}{\norm[2]{\rpsl}} \Bigg(\frac{\sqrt{d_\l}}{\sqrt{m}} + \underbrace{\frac{2\sqrt{d_\l}}{\sqrt{n_\l - 1}}}_{= 2/\sqrt{\rho_\l}} \Bigg) \left(1 + \frac{3}{2} \sigma \right) \Bigg),
    \label{eq:CorrSelCondCombA}
\end{align}
where $1-(c_4+1)/\sqrt{\rho_\l} \geq 1/2$ follows from $\rho_\l \geq \rho_{\min} \geq c_\rho \defeq 4(c_4+1)^2$, for all $\l \in [L]$, and $\sqrt{d_\l}/\sqrt{m - d_\l} \leq \sqrt{2 d_\l}/\sqrt{m}$ is by $m \geq 2 d_{\max} \geq 2 d_\l$, for all $\l \in [L]$. Rearranging terms in \eqref{eq:CorrSelCondCombA} and using $\sqrt{8 \log(N^3 \itr_{\max})} < 4 \log(N^3 \itr_{\max})$, for $N \geq 2$, we can see that \eqref{eq:CorrSelCondCombA} is implied by
\begin{align}
&\max_{k \colon k \neq \l} \aff (\cS_k,\cS_\l) + \frac{1}{\norm[2]{\rpsl} \sqrt{8 \log (N^3 s_{\max})}} \Bigg(\frac{\sqrt{2 d_\l}}{\sqrt{m}} \norm[2]{\rosl} \nonumber \\
& \quad+ \sigma \left( \frac{5}{2} \frac{\sqrt{d_\l}}{\sqrt{m}} + \frac{2}{\sqrt{\rho_\l}} \right) \left( 1 + \frac{3}{2} \sigma \right) \Bigg) \leq \frac{1}{8 \log(N^3 s_{\max})}. \label{eq:CorrSelCondCombB}
\end{align}
To further simplify \eqref{eq:CorrSelCondCombB}, we upper-bound $\norm[2]{\smash{\rosl}}$  and lower-bound $\norm[2]{\smash{\rpsl}}$. To this end, we introduce the events
\begin{align}
\Ef &\defeq \left\{ \norm[2]{\rosl} \leq \norm[2]{\zl_{i \po}} + \frac{3 \sigma}{\tilde a} \norm[2]{\yl_i}, \forall \itr \leq \itr_{\max} \right\}, \nonumber \\
\Eg &\defeq \vast\{ \norm[2]{\rpsl} \nonumber \\
& \quad > \norm[2]{\yl_{i \pp}}\! \left(\frac{2}{3} - \sqrt{ \frac{3 \smaxparlb \log ((n_\l - 1)e / \smaxparlb)}{d_\l}} \right), \forall \itr \leq \smaxparlb \vast\}\!, \label{eq:resplbdef}
\end{align}
where\vspace{0.05cm} $\tilde a \defeq \min_{j \in [n_\l] \backslash\{ i \}} \norm[2]{\smash{\yl_{j \pp}}} \geq \min_{j \in [n_\l] \backslash\{ i \}} (\norm[2]{\smash{\xl_j}} - \norm[2]{\smash{\zl_{j \pp}}}) \geq 1 - \max_{j \in [n_\l] \backslash\{ i \}} \norm[2]{\smash{\zl_{j \pp}}}$.\vspace{0.05cm} Setting $\smaxparlb = \itr_{\max}$ in \eqref{eq:resplbdef}, on $\Ed \cap \Ef \cap \Eg$, we have
\begin{align}
\norm[2]{\rosl} &\leq \frac{3}{2} \sigma + 3 \sigma \frac{1 + \frac{3}{2} \sigma}{1 - \frac{3}{2} \frac{\sqrt{d_\l}}{\sqrt{m}} \sigma} \leq \sigma (8 + 10 \sigma), \label{eq:resoSimpleUB} \\
\norm[2]{\rpsl} & > \left(1 - \frac{3}{2} \frac{\sqrt{d_\l}}{\sqrt{m}} \sigma \right) \left(\frac{2}{3} - \sqrt{3 c_\itr}\right) \nonumber \\
&> \left(1 - \frac{3}{2} \frac{\sqrt{d_\l}}{\sqrt{m}} \sigma \right) \frac{1}{9} > \frac{1}{20}, \label{eq:respSimpleLB}
\end{align}
where we employed the assumptions $m \geq 2 d_{\max} \geq 2 d_\l$, for all $\l \in [L]$, and $\sigma \leq 1/2$ to get \eqref{eq:resoSimpleUB}, and used $\s_{\max} \leq c_s d_\l / \log( (n_\l - 1) e / s_{\max})$, for all $\l \in [L]$, and $c_s \defeq \min\{1/10, c_1\}$ 
(with $c_1$ the constant in Lemma \ref{le:ResOrthUB} below), to arrive at \eqref{eq:respSimpleLB}. With \eqref{eq:resoSimpleUB} and \eqref{eq:respSimpleLB}, it follows that \eqref{eq:CorrSelCondCombB} is implied by
\begin{align}
\max_{k \colon k \neq \l} \aff (\cS_k,\cS_\l) &+ \frac{10 \sigma}{\sqrt{\log (N^3 s_{\max})}} \Bigg( \frac{\sqrt{d_\l}}{\sqrt{m}}\left(10 + 13 \sigma \right) \nonumber \\
&\quad + \frac{\sqrt{2}}{\sqrt{\rho_\l}} \left( 1 + \frac{3}{2} \sigma \right) \Bigg) \leq \frac{1}{8 \log (N^3 s_{\max})}. \nonumber
\end{align}
This inequality holds for all $\l \in [L]$ by the clustering condition \eqref{eq:ClusCondThm} with $c(\sigma) = 10 + 13 \sigma$. Hence, on the event 
\begin{align}
\Estar \defeq \bigcap_{\l,i,s}  &\big(\Ea \cap \Eb \cap \Ec \nonumber \\
& \qquad \qquad \cap \Ed \cap \Ee \cap \Ef \cap \Eg\big), \label{eq:estar}
\end{align}
\eqref{eq:ClusCondThm} implies \eqref{eq:CorrSelCond} for every $\yl_i \in \cY_\l$, for all $\l \in [L]$, and the graph $G$ obtained by \ac{ssc}-\ac{omp} has no false connections. 
It remains to lower-bound $\prob{\Estar}$. 
By the union bound, we have
\begin{align}
&\prob{\Estar} = 1 - \prob{\comp{\mc E}^\star} \nonumber \\
&\quad\geq 1 - \prob{\comp{\mc E}_4} \nonumber \\
&\qquad - \sum_{\l \in [L], i \in [n_\l]} \Big( \prob{\evcomp{5}{\l,i}} + \prob{\evcomp{6}{\l,i}} + \prob{\evcomp{7}{\l,i}} \Big)  \nonumber \\
&\qquad -  \sum_{\substack{\l \in [L], i \in [n_\l],\\ s \in [s_{\max}]}} \!\!\Big(\prob{\evcomp{1}{\l,i,s}} + \prob{\evcomp{2}{\l,i,s}} + \prob{\evcomp{3}{\l,i,s}} \Big)  \nonumber \\
&\quad\geq 1 - \sum_{\l \in [L]} n_\l ( e^{- d_\l/8} + e^{-m/8} ) \nonumber \\ 
&\qquad- \!\!\! \sum_{\l \in [L], i \in [n_\l]} \Big( 2e^{-c_5 d_\l} + 2e^{-2 m} \nonumber \\
&\qquad\qquad\quad+ 2e^{-c_2 m} + 2e^{-c_3 d_\l} + e^{-d_\l/18}\Big)  \nonumber \\
&\qquad- \sum_{\substack{\l \in [L], i \in [n_\l],\\ s \in [s_{\max}]}} \Big(\frac{2}{N^2 s_{\max}} + \frac{2}{N^2 s_{\max}} + \frac{2}{N^2 s_{\max}} \Big)  \label{eq:TotProbUnion} \\
&\quad\geq 1 - \frac{6}{N} - \sum_{\l \in [L]} n_\l (6e^{-c_d d_\l} + 5e^{-c_m m}), \nonumber 
\end{align}
where $c_d \defeq \min\{1/18, c_3, c_5\}$, $c_m \defeq \min\{ 1/8, c_2\}$, and \eqref{eq:TotProbUnion} follows from Lemmata \ref{le:EaEbEc}, \ref{le:ResParaLB}, \ref{le:ResOrthUB}, \ref{le:RHSLB}, and \ref{le:Ed}. 

The proofs of Lemmata \ref{le:EaEbEc} and \ref{le:ResParaLB} rely on the rotational invariance of the distributions of $\rpsl$ and $\rosl$ on $\cS_\l$ and $\cS_\l^\perp$, respectively, which is inherited from the rotational invariance of the distributions of the $\xl_j$ and $\zl_j$ 
characterized next. 

\begin{lemma} \label{le:ResRotInv} 
The distributions of $\rpsl$ and $\rosl$ are rotationally invariant on $\cS_\l$ and $\cS_\l^\perp$, respectively, for all $i \in [n_\l]$, $\l \in [L]$, i.e., we have $\mV^\pp \rpsl \sim \rpsl$ and $\mV^\po \rosl \sim \rosl$ for all unitary matrices $\mV^\pp$ and $\mV^\po$ of the form $\mV^\pp = \Ul \mW^\pp \transp{\Ul} + \Po$ and $\mV^\po = \Pp + \mU^{(\l)}_o \mW^\po \transp{\mU^{(\l)}_o}$, 
respectively, where $\mW^\pp \in \reals^{d_\l \times d_\l}$, $\mW^\po \in \reals^{(m-d_\l) \times (m - d_\l)}$ are unitary and the columns of $\mU^{(\l)}_o \in \reals^{m \times (m-d_\l)}$ form an orthonormal basis for $\cS_\l^\perp$. 
\end{lemma}

Note that the unitary transformations $\mV^\pp$ and $\mV^\po$ act only on $\cS_\l$ and $\cS_\l^\perp$, respectively, and leave components in $\cS_\l^\perp$ and $\cS_\l$, respectively, unchanged.

\begin{proof}

We first show that the reduced \ac{omp} residual $\rsl$ is covariant w.r.t. transformations of the points in $\cY_\l$ by a unitary matrix $\mV \in \reals^{m \times m}$, i.e., we establish that $\rs(\mV \yl_i, \mV \Yl_{-i}) = \mV \rs(\yl_i, \Yl_{-i})$, for all $\itr \in [\itr_{\max}]$, $i \in [n_\l]$, $\l \in [L]$. 
Combining this covariance property of $\rsl$ with the rotational invariance on $\cS_\l$ and $\cS_\l^\po$ of the distributions of $\yl_{j \pp} = \Ul \atl_j$ and $\yl_{j \po} = \Po \zl_j$, respectively, then establishes the desired result.

We prove $\rs(\mV \yl_i, \mV \Yl_{-i}) = \mV \rs(\yl_i, \Yl_{-i})$ by induction and start with the inductive step. Assume that after some iteration $\itr' < \itr_{\max}$, the index set $\Lambda_{\itr'}$ corresponding to the transformed data ($\mV \yl_i$, $\mV \Yl_{-i}$) is identical to the index set $\Lambda_{\itr'}$ associated with the original data, i.e., $\Lambda_{\itr'}(\mV \yl_i, \mV \Yl_{-i}) = \Lambda_{\itr'}(\yl_i, \Yl_{-i})$. Using the shorthands $\Lambda_{\itr'}(\mV)$ for $\Lambda_{\itr'}(\mV \yl_i, \mV \Yl_{-i})$ and $\Lambda_{\itr'}$ for $\Lambda_{\itr'}(\yl_i, \Yl_{-i})$, it then follows that
\begin{align}
&\rspr(\mV \yl_i, \mV \Yl_{-i}) = \left(\mI - \mV \Yl_{\Lambda_{\itr'}(\mV)} \pinv{\left(\mV \Yl_{\Lambda_{\itr'}(\mV)}\right)} \right) \mV \yl_i \nonumber \\
&\quad= \left(\mI - \mV \Yl_{\Lambda_{\itr'}} \pinv{\left(\mV \Yl_{\Lambda_{\itr'}}\right)} \right) \mV \yl_i \nonumber \\
&\quad= \left(\mI - \mV \Yl_{\Lambda_{\itr'}} \inv{\left(\transp{\Yl_{\Lambda_{\itr'}}} \transp{\mV} \mV \Yl_{\Lambda_{\itr'}} \right)} \transp{\Yl_{\Lambda_{\itr'}}} \transp{\mV} \right) \mV \yl_i \nonumber \\
&\quad= \mV \left(\mI - \Yl_{\Lambda_{\itr'}} \inv{\left(\transp{\Yl_{\Lambda_{\itr'}}} \Yl_{\Lambda_{\itr'}} \right)} \transp{\Yl_{\Lambda_{\itr'}}} \right) \yl_i \nonumber \\
&\quad= \mV \rspr(\yl_i,\Yl_{-i}). \label{eq:ResRotInv1}
\end{align}
For the index $\lambda_{\itr'+1}(\mV \yl_i, \mV \Yl_{-i})$ selected for the $\mV$-transformed data set in iteration $\itr'+1$, \eqref{eq:ResRotInv1} implies
\begin{align}
&\lambda_{\itr'+1}(\mV \yl_i, \mV \Yl_{-i}) \nonumber \\
&\qquad= \underset{j \in [n_\l] \backslash (\Lambda_{\itr'}\cup \{i\})}{\arg \max} \abs{\innerprod{\mV \yl_j}{\rspr(\mV \yl_i, \mV \Yl_{-i}) }} \nonumber \\
&\qquad= \underset{j \in [n_\l] \backslash (\Lambda_{\itr'}\cup \{i\})}{\arg \max} \abs{\innerprod{ \yl_j}{\rspr(\yl_i, \Yl_{-i})}}\nonumber \\
&\qquad= \lambda_{\itr'+1}(\yl_i, \Yl_{-i}), \label{eq:ResRotInv2}
\end{align}
i.e., the index selected in iteration $\itr'+1$ by operating on the $\mV$-transformed data set is identical to that obtained for the original data set. 
It remains to establish the base case. This is done by noting that thanks to $\vr_0(\mV \yl_i, \mV \Yl_{-i}) = \mV \yl_i$, we have 
\begin{align} 
\lambda_1 (\mV \yl_i, \mV \Yl_{-i})
 &=  \underset{j \in [n_\l] \backslash (\Lambda_{\itr'}\cup \{i\})}{\arg \max} \abs{\innerprod{\smash{\mV \yl_j}}{\smash{\mV \yl_i}}} \nonumber \\
 &= \underset{j \in [n_\l] \backslash (\Lambda_{\itr'}\cup \{i\})}{\arg \max} \abs{\innerprod{\smash{\yl_j}}{\smash{\yl_i}}} \nonumber \\
 &= \lambda_1 (\yl_i, \Yl_{-i}). \nonumber 
\end{align}

We next establish the rotational invariance of $\rpsl$ and $\rosl$. Note that $\mV^\pp \yl_j = \Ul \mW^\pp \transp{\Ul} \yl_{j \pp} \allowbreak + \yl_{j \po} =  \Ul \mW^\pp \atl_{j} + \yl_{j \po} \sim \Ul \atl_{j} + \yl_{j \po} = \yl_j, j \in [n_\l]$, and $\mV^\po \yl_j \! = \yl_{j \pp} + \mU^{(\l)}_o \mW^\po \transp{\mU^{(\l)}_o} \! \yl_{j \po} \allowbreak = \yl_{j \pp} + \mU^{(\l)}_o \mW^\po \transp{\mU^{(\l)}_o} \zl_{j \po} \sim \yl_{j \pp} + \zl_{j \po} =  \yl_j, j \in [n_\l]$, by 
unitarity of $\mV^\pp$ and $\mV^\po$. Together with $\rs(\mV \yl_i, \mV \Yl_{-i}) = \mV \rs(\yl_i, \Yl_{-i})$, this yields
\begin{align}
\rps(\yl_i, \Yl_{-i}) &= \Pp \rs(\yl_i, \Yl_{-i}) \label{eq:ResDistInv1} \\
&\sim \Pp \rs(\mV^\pp \yl_i, \mV^\pp \Yl_{-i}) \nonumber \\
&= \Pp \mV^\pp \rs(\yl_i, \Yl_{-i}) \nonumber \\ 
&= (\underbrace{\Pp \Ul \mW^\pp \transp{\Ul}}_{=\Ul \mW^\pp \transp{\Ul} \Pp} + \underbrace{\Pp \Po}_{= \mathbf 0}) \rs(\yl_i, \Yl_{-i}) \nonumber \\
&= \mV^\pp \Pp \rs(\yl_i, \Yl_{-i}) \nonumber \\ 
&= \mV^\pp \rps(\yl_i, \Yl_{-i}),  \label{eq:ResDistInv2}
\end{align}
and establishes $\mV^\pp \rps \sim \rps$. Repeating the steps leading from \eqref{eq:ResDistInv1} to \eqref{eq:ResDistInv2} with $\Po$ 
and $\mV^\po$ in place of $\Pp$ and $\mV^\pp$, respectively, we analogously obtain $\ros \sim \mV^\po \ros$, 
thereby finishing the proof.
\end{proof}

We next derive lower bounds on $\prob{\smash{\Ea}}$, $\prob{\smash{\Eb}}$, and $\prob{\smash{\Ec}}$.
\begin{lemma} \label{le:EaEbEc} We have 
\begin{align}
&\prob{\Ea} \geq 1 - \frac{2}{N^2 s_{\max}}, \quad \prob{\Eb} \geq 1 - \frac{2}{N^2 s_{\max}}, \nonumber \\
& \qquad \qquad \qquad
\prob{\Ec} \geq 1 - \frac{2}{N^2 s_{\max}}. \label{eq:LeEaEbEc}
\end{align}
\end{lemma}

\begin{proof}
First note that $\rpsl / \norm[2]{\smash{\rpsl}}$ and $\rosl / \norm[2]{\smash{\rosl}}$ are distributed uniformly at random on $\US{m} \cap \cS_\l$ and $\US{m} \cap \cS_\l^\perp$, respectively, as a consequence of rotational invariance (Lemma \ref{le:ResRotInv}) and normalization \cite[Thm. 1.5.6]{muirhead2009aspects}. 
This allows us to apply Lemma \ref{le:InterSubsInnerProd} below with $\mL = \Ak$, $\mC = \transp{\Uk} \Ul$, $\va = \rpsl / \norm[2]{\smash{\rpsl}}$, and $\alpha = 4 \log( N^3 \itr_{\max})$ (note that the condition $\alpha > 12$ is satisfied as $N \geq 3$ and $\itr_{\max} \geq 1$ by the assumptions of Theorem \ref{th:SSCOMP}) to get a lower bound on $\prob{\Ea}$ according to 
\begin{align}
&\mathrm{P} \vast[ \max_{j \in [n_k]} \abs{\innerprod{\xk_j}{\rpsl}} \nonumber \\
&\qquad> 4 \log( N^3 \itr_{\max}) \frac{\norm[F]{\transp{\Uk} \Ul}}{\sqrt{d_k} \sqrt{d_\l}} \norm[2]{\rpsl} \vast] \leq \frac{n_k+1}{N^3 s_{\max}} \nonumber \\
&\hspace{6.9cm}\leq \frac{2 n_k}{N^3 s_{\max}}, \nonumber
\end{align}
for $k \neq \l$. The desired bound on $\prob{\smash{\Ea}}$ then follows by a union bound over $k \in [L] \backslash \{ \l \}$.

\vspace{0.05cm}The lower bound on $\prob{\smash{\Eb}}$ is obtained by invoking\vspace{0.05cm} Lemma \ref{thm:hoeffsphere} below with\vspace{0.05cm} $\va = \rosl / \norm[2]{\smash{\rosl}} \in \cS_\l^\perp$ (hence replacing $\US{m}$ by $\cS_\l^\perp \cap \US{m}$), $\vb = \xk_j$, and $\beta = \sqrt{2 \log( N^3 \itr_{\max})}$, which yields
\begin{equation}
\prob{ \abs{\innerprod{\rosl}{\xk_j}} > \frac{\sqrt{2 \log( N^3 \itr_{\max})}}{\sqrt{m - d_\l}} \norm[2]{\rosl}} \leq \frac{2}{N^3 s_{\max}}, \label{eq:prbndEb}
\end{equation}
for $k \neq \l$. 
Again, a union bound over $k \in [L] \backslash \{ \l \}$, $j \in [n_k]$, gives the desired bound on $\prob{\smash{\Eb}}$.

Finally, for $\prob{\smash{\Ec}}$, we set $\va = \zk_j / \norm[2]{\smash{\zk_j}}$, $\vb = \rsl$, and $\beta = \sqrt{2 \log( N^3 \itr_{\max})}$ in Lemma~\ref{thm:hoeffsphere}, and use $\norm[2]{\smash{\rsl}} \leq 1 + \norm[2]{\smash{\zl_i}}$, to obtain
\begin{align}
&\prob{\abs{\innerprod{\zk_j}{\rsl}} > \frac{\sqrt{2 \log( N^3 \itr_{\max})}}{\sqrt{m}} \left(1 + \norm[2]{\zl_i} \right)  \norm[2]{\zk_j}} \nonumber \\
&\hspace{7cm}\leq \frac{2}{N^3 s_{\max}}, \nonumber 
\end{align}
for all $k \neq \l$. Again, the lower bound on $\prob{\smash{\Ec}}$ follows from a union bound over $k \in [L] \backslash \{ \l \}$, $j \in [n_k]$.

\begin{lemma}[Extracted from the proof of {\cite[Lem. 7.5]{soltanolkotabi2012geometric}}] \label{le:InterSubsInnerProd} Let $\va \in \reals^{d_2}$ be distributed uniformly at random on $\US{d_2}$ and let the columns of $\mL \in \reals^{d_1 \times n_1}$ be independent and distributed uniformly on $\US{d_1}$. Let $\mC \in \reals^{d_1 \times d_2}$. Then, for $\alpha \geq 12$, we have
\begin{equation}
\prob{\norm[\infty]{\mL \mC \va} \geq \frac{\alpha}{\sqrt{d_1} \sqrt{d_2}} \norm[F]{\mC} } \leq (n_1 + 1) e^{-\alpha/4}.
\end{equation}
\end{lemma}

\begin{lemma}[{E.g.,~\cite[Ex.~5.25]{vershynin2012nonasym}}]
Let $\va$ be uniformly distributed on $\US{m}$ and fix $\vb \in \reals^m$. Then, for $\beta\geq 0$, we have
\[
\PR{ \left|\innerprod{\va}{ \vb }\right|  > \frac{\beta}{\sqrt{m}} \norm[2]{ \vb} } 
\leq 2 e^{-\frac{ \beta^2}{2}}. 
\]
\label{thm:hoeffsphere}
\end{lemma}
\vspace{-0.75cm}
\end{proof}

Next, we lower-bound $\prob{\smash{\Eg}}$.
\begin{lemma} Let $n_\l \geq d_\l + 1$ and $\smaxparlb \leq d_\l$. We have 
\label{le:ResParaLB}
\begin{align}
	&\prob{\Eg} = \mathrm{P}\! \vast[\norm[2]{\rpsl} \nonumber \\
	&\qquad> \norm[2]{\yl_{i \pp}} \left(\frac{2}{3} - \sqrt{ \frac{3 \smaxparlb \log ((n_\l - 1)e / \smaxparlb)}{d_\l}} \right), \; \forall \itr \leq \smaxparlb \vast] \nonumber \\
	&\qquad \qquad \geq 1 - e^{-d_\l/18} \label{eq:LeResParaLB}.
\end{align}
\end{lemma}

\begin{proof}
The\vspace{0.05cm} bound obviously holds for $\itr = 0$ as $\norm{\smash{\subsind{\vr}{\l}_{0 \pp}}} = \norm[2]{\smash{\yl_{i \pp}}}$.
For $1 \leq \itr \leq \smaxparlb$ the outline of the proof\vspace{0.075cm} is as follows. As $\norm[2]{\smash{\rpsl}} = \norm[2]{\smash{\Pp (\mI - \Yl_{\Lambda_\itr} \pinv{\Yl_{\Lambda_\itr}}) \yl_i}}$ is hard to analyze directly owing to statistical dependencies between $\yl_i$ and the columns of $\Yl_{\Lambda_\itr}$ induced by the dependence of $\Lambda_\itr$ on $\yl_i$, we rely\vspace{0.075cm} on an auxiliary quantity, namely $\norm[2]{\smash{\Pp (\mI - \Yl_\Gamma \pinv{\Yl_\Gamma}) \yl_i}}$ for a fixed index set $\Gamma \subset [n_\l] \backslash \{ i \}$ with cardinality satisfying $1 \leq \vert \Gamma \vert \leq \smaxparlb$.
We start the proof\vspace{0.075cm} by deriving a lower bound $\varphi_\Gamma$ on $\norm[2]{\smash{\Pp (\mI - \Yl_\Gamma \pinv{\Yl_\Gamma}) \yl_i}}$ 
and then show that $\Eg \supseteq \Fg$, where
\begin{align}
 &\Fg \defeq \vast\{ \varphi_{\Gamma'} > \norm[2]{\yl_{i \pp}} \left(\frac{2}{3} - \sqrt{ \frac{3 \smaxparlb \log ((n_\l - 1)e / \smaxparlb)}{d_\l}} \right),\nonumber \\ 
 &\hspace{7.2cm} \forall \Gamma' \in \mc I \vast\} \nonumber
\end{align}
with
\begin{equation}
\mc I \defeq \{ \Gamma' \subset [n_\l]  \backslash \{ i \} \colon \vert \Gamma' \vert = \smaxparlb \}, \label{eq:SmaxIdxSets}
\end{equation}
which implies $\prob{\smash{\Eg}} \geq \prob{\smash{\Fg}}$. The proof is then completed by establishing a lower bound on $\prob{\smash{\Fg}}$ using a version of Borell's inequality. 

We proceed by lower-bounding $\norm[2]{\smash{\Pp (\mI - \Yl_\Gamma \pinv{\Yl_\Gamma}) \yl_i}}$ for $1 \leq \vert \Gamma \vert \leq \smaxparlb$. Define the orthogonal projection matrices $\mP_{\Gamma} \defeq \Yl_\Gamma \pinv{\Yl_\Gamma}$ and $\mP^\pp _{\Gamma} \defeq \Yl_{\Gamma \pp} \pinv{\Yl_{\Gamma \pp}}$ and note that $\Pp \mP^\pp _{\Gamma} = \mP^\pp _{\Gamma} = \mP^\pp _{\Gamma} \Pp$. We now get 
\begin{align}
&\norm[2]{\Pp (\mI - \mP_{\Gamma}) \yl_i}^2 \nonumber \\
&\qquad= \norm[2]{\Pp (\mI - \mP^\pp _{\Gamma} + \mP^\pp _{\Gamma} - \mP_{\Gamma}) \yl_i}^2  \label{eq:ResParaLB0} \\
&\qquad= \norm[2]{\Pp (\mI - \mP^\pp _{\Gamma})\yl_i + (\mP^\pp _{\Gamma} - \Pp \mP_{\Gamma})\yl_i}^2 \label{eq:ResParaLBa} \\
&\qquad= \norm[2]{\Pp (\mI - \mP^\pp _{\Gamma})\yl_i}^2
 + \norm[2]{(\mP^\pp _{\Gamma} - \Pp \mP_{\Gamma}) \yl_i}^2 \label{eq:ResParaLBb} \\
&\qquad\geq \norm[2]{\Pp (\mI - \mP^\pp _{\Gamma})\yl_i}^2 \nonumber \\ 
 &\qquad= \norm[2]{\yl_{i \pp} - \mP^\pp _{\Gamma}\yl_{i \pp}}^2, \label{eq:ResParaLB1}
\end{align}
where the last equality is thanks to $\Pp \mP^\pp _{\Gamma} = \mP^\pp _{\Gamma} \Pp$ and the step leading from \eqref{eq:ResParaLBa} to \eqref{eq:ResParaLBb} follows from 
\begin{align}
&\transp{(\Pp ( \mI - \mP^\pp _{\Gamma}))} (\mP^\pp _{\Gamma} - \Pp \mP_\Gamma) \nonumber \\
&\qquad\qquad\qquad= (\mI - \mP^\pp _{\Gamma}) \Pp (\mP^\pp _{\Gamma} - \Pp \mP_\Gamma) \nonumber \\
&\qquad\qquad\qquad= (\mI - \mP^\pp _{\Gamma}) \Pp (\mP^\pp _{\Gamma}\mP^\pp _{\Gamma} - \mP^\pp _{\Gamma}\Pp \mP_\Gamma) \nonumber \\
&\qquad\qquad\qquad= (\mI - \mP^\pp _{\Gamma}) \Pp \mP^\pp _{\Gamma} (\mP^\pp _{\Gamma} - \Pp \mP_\Gamma) \nonumber \\
&\qquad\qquad\qquad= \underbrace{(\mI - \mP^\pp _{\Gamma}) \mP^\pp _{\Gamma}}_{= \mathbf 0} (\mP^\pp _{\Gamma} - \Pp \mP_\Gamma) = \mathbf 0. \nonumber
\end{align}
Using \eqref{eq:ResParaLB0}--\eqref{eq:ResParaLB1}, we further have
\begin{align}
\norm[2]{\Pp (\mI - \mP_{\Gamma}) \yl_i} &\geq \norm[2]{\yl_{i \pp} - \mP^\pp _{\Gamma}\yl_{i \pp}} \label{eq:ResParaLB2a} \\
&\geq \norm[2]{\yl_{i \pp}} - \norm[2]{\mP^\pp _{\Gamma}\yl_{i \pp}} \nonumber \\
&\geq \underbrace{\norm[2]{\yl_{i \pp}} - \norm[2]{\mP^\pp_{\Gamma'}\yl_{i \pp}}}_{\eqqcolon \varphi_{\Gamma'}}, \label{eq:ResParaLB2c}
\end{align}
for all $\Gamma \subseteq \Gamma' \in \mc I$, where the second inequality is by the reverse triangle inequality and the third is a consequence of 
$\range(\mP^\pp _{\Gamma}) \subseteq \range(\mP^\pp_{\Gamma'}) \subset \cS_\l$.\vspace{0.05cm} It follows from \eqref{eq:OMPResFormula} and \eqref{eq:ResParaLB2a}--\eqref{eq:ResParaLB2c} that $\norm[2]{\smash{\rpsl}} = \norm[2]{\smash{\Pp (\mI - \mP_{\Lambda_s}) \yl_i}} \geq \varphi_{\Gamma'}$ for $\Lambda_s \subset  \mc I$,\vspace{0.05cm} which implies $\norm[2]{\smash{\rpsl}} \geq \min_{\Gamma' \in \mc I} \varphi_{\Gamma'}$, and thus, indeed, $\Eg \supseteq \Fg$. It remains to lower-bound $\PR{\smash{\Fg}}$.

To this end, we first work on the second term in \eqref{eq:ResParaLB2c} and note that
\begin{align}
\frac{\norm[2]{\mP^\pp_{\Gamma'}\yl_{i \pp}}}{\norm[2]{\yl_{i \pp}}} \nonumber
&= \frac{\norm[2]{\Yl_{\Gamma' \pp} \pinv{\Yl_{\Gamma' \pp}} \yl_{i \pp}}}{\norm[2]{\yl_{i \pp}}} \nonumber \\ 
&= \frac{\norm[2]{\Ul \Atl_{\Gamma'} \pinv{(\Ul \Atl_{\Gamma'})} \Ul \atl_i}}{\norm[2]{\Ul \atl_i}}  \nonumber \\ 
&= \frac{\norm[2]{\Atl_{\Gamma'} \Atl_{\Gamma'}\pinv{\vphantom{\Al}} \atl_i}}{\norm[2]{\atl_i}}. 
\end{align} 
Since the columns of $\Atl_{\Gamma'}$, i.e., the vectors $\atl_j, j \in \Gamma'$, are i.i.d. and of rotationally invariant distribution, $\Atl_{\Gamma'} \Atl_{\Gamma'}\pinv{\vphantom{\Al}}$ is the projector onto a subspace of $\reals^{d_\l}$ that is $\smaxparlb$-dimensional w.p. $1$. In particular, this subspace is distributed uniformly at random on the set of all $\smaxparlb$-dimensional subspaces of $\reals^{d_\l}$. Indeed, note that we have, w.p. $1$,
\begin{align}
\range \left(\Atl_{\Gamma'} \right) &= \range \left(\Atl_{\Gamma'} \, \mathrm{diag}\left(1 \middle/ \norm[2]{\atl_{\gamma'_1}},\ldots,1\middle/ \norm[2]{\atl_{\gamma'_{\smaxparlb}}}\right)\right) \nonumber \\
&\sim \range \left(\mG \, \mathrm{diag}\left(1/\norm[2]{\vg_1},\ldots,1/\norm[2]{\vg_\smaxparlb}\right)\right) \label{eq:unifsubspace1}\\
&= \range(\mG) \label{eq:unifsubspace2} \\
&= \range(\mG (\transp{\mG} \mG)^{-1/2}), \label{eq:unifsubspace3} 
\end{align}
where $\gamma'_j$, $j \in [\smaxparlb]$, denotes the elements of $\Gamma'$ and $\mG=[\vg_1 \, \dots \, \vg_{\smaxparlb}] \in \reals^{d_\l \times \smaxparlb}$ has i.i.d. standard normal random variables as entries. Here, to obtain \eqref{eq:unifsubspace1}, we used that the $\atl_{\gamma'_j}/\norm[2]{\smash{\atl_{\gamma'_j}}}$, $j \in [\smaxparlb]$, and the $\vg_j/\norm[2]{\vg_j}$, $j \in [\smaxparlb]$, are all i.i.d. uniform on $\US{d_\l}$, and for \eqref{eq:unifsubspace2} and \eqref{eq:unifsubspace3} we exploit that $\mathrm{diag}(1/\norm[2]{\vg_1},\ldots,1/\norm[2]{\vg_\smaxparlb})$ and $\mG$, respectively, have full rank w.p. $1$. The claim now follows by noting that $\mG (\transp{\mG} \mG)^{-1/2}$ is distributed uniformly at random on the set of all orthonormal matrices in $\reals^{d_\l \times \smaxparlb}$ \cite[Thm.~2.2.1~iii)]{chikuse2003statistics}.

Next, we note that conditioning on $\Atl_{\Gamma'} \Atl_{\Gamma'}\pinv{\vphantom{\Al}}$\vspace{0.05cm} does not change the distribution of $\norm[2]{\smash{\Atl_{\Gamma'} \Atl_{\Gamma'} \pinv{\vphantom{\Al}} \! \atl_i}\!}^2 \allowbreak / \norm[2]{\smash{\atl_i}}^2$. 
To see this, consider $\va \in \reals^m$ distributed uniformly at random on $\US{d_\l}$ and choose $\mV$ uniformly at random from the set of all orthonormal matrices in $\reals^{d_\l \times \d_\l}$. Further, let $\mP_{\smaxparlb}$ be a projector onto an arbitrary, but fixed $\smaxparlb$-dimensional subspace of $\reals^{d_\l}$. Then, we have $\norm[2]{\smash{\Atl_{\Gamma'} \Atl_{\Gamma'} \pinv{\vphantom{\Al}} \atl_i}}^2 / \norm[2]{\smash{\atl_i}}^2 \sim \norm[2]{\mP_{\smaxparlb} \transp{\mV} \va}^2 \sim \norm[2]{\mP_{\smaxparlb} \va}^2$, where the first distributional equivalence follows from $\Atl_{\Gamma'} \Atl_{\Gamma'} \pinv{\vphantom{\Al}} \sim \mV \mP_{\smaxparlb} \transp{\mV}$ (by \cite[Thm. 2.2.1 ii)]{chikuse2003statistics}) and the second from $\transp{\mV} \va \sim \va$ (by rotational invariance of the distributions of $\mV$ and $\va$). 

Now, using $\norm[2]{\smash{\mP^\pp_{\Gamma'}\yl_{i \pp}}} = \norm[2]{\smash{\Atl_{\Gamma'} \Atl_{\Gamma'} \pinv{\vphantom{\Al}} \atl_i}}^2$\vspace{0.05cm} and conditioning $\norm[2]{\smash{\Atl_{\Gamma'} \Atl_{\Gamma'} \pinv{\vphantom{\Al}} \atl_i}}^2$ on $\Atl_{\Gamma'} \Atl_{\Gamma'}\pinv{\vphantom{\Al}}$ allows us to apply the following version of Borell's inequality to get an upper bound on the second term on the \ac{rhs} of \eqref{eq:ResParaLB2c}. 

\begin{lemma}[Extracted from the proof of {\cite[Lem. 7.5]{soltanolkotabi2012geometric}}] \label{lem:borell} Let $\bsy \Sigma \in \reals^{d_1 \times d_2}$ be a deterministic matrix and take $\bsy \lambda \in \reals^{d_2}$ to be distributed uniformly at random on $\US{d_2}$. Then, we have
\begin{equation}
\prob{\norm[2]{\bsy \Sigma \bsy \lambda} - \frac{\norm[F]{\bsy \Sigma}}{\sqrt{d_2}} \geq \varepsilon} < e^{-d_2 \varepsilon^2 / (2 \maxsingv{\bsy \Sigma}^2)}. \nonumber
\end{equation}
\end{lemma} 

Setting $\bsy \Sigma = \Atl_{\Gamma'} \Atl_{\Gamma'} \pinv{\vphantom{\Al}}$ and $\bsy \lambda = \atl_i / \norm[2]{\smash{\atl_i}}$ in Lemma \ref{lem:borell} and noting that 
$\norm[F]{\bsy \Sigma} = \norm[F]{\smash{\Atl_{\Gamma'} \Atl_{\Gamma'} \pinv{\vphantom{\Al}} }} \allowbreak = \sqrt{ \smaxparlb}$ 
and $\maxsingv{\bsy \Sigma} = 1$ yields 
\begin{equation}
\prob{\norm[2]{\Atl_{\Gamma'} \Atl_{\Gamma'}\pinv{\vphantom{\Al}} \atl_i} \geq \norm[2]{\atl_i} \left( \sqrt{\frac{\smaxparlb}{d_\l}} + \varepsilon \right) } < e^{-d_\l \varepsilon^2/2}. \label{eq:ResParaProb1}
\end{equation}
We now have
\begin{align}
& \prob{\varphi_{\Gamma'} \geq \norm[2]{\yl_{i \pp}} \left(1 - \sqrt{\frac{\smaxparlb}{d_\l}} - \varepsilon \right), \; \; \forall \Gamma' \in \mc I} \nonumber \\
 &\qquad = \prob{ \norm[2]{\mP_{\pp \Gamma'} \yl_{i \pp}} < \norm[2]{\yl_{i \pp}} \left( \sqrt{\frac{\smaxparlb}{d_\l}} + \varepsilon \right), \; \; \forall \Gamma' \in \mc I} \label{eq:ResParaProb11} \\
 &\qquad \geq 1 - \sum_{\Gamma' \in \mc I} \prob{ \norm[2]{\mP_{\pp \Gamma'} \yl_{i \pp}} \geq \norm[2]{\yl_{i \pp}} \left( \sqrt{\frac{\smaxparlb}{d_\l}} + \varepsilon \right)} \label{eq:ResParaProb12} \\
 &\qquad \geq 1 - \binom{n_\l -1}{\smaxparlb} e^{-d_\l \varepsilon^2/2} \label{eq:ResParaProb2} \\
 &\qquad \geq 1- e^{\smaxparlb \log((n_\l - 1) e / \smaxparlb)} e^{-d_\l \varepsilon^2/2}, \label{eq:ResParaProb3}
\end{align} 
where we used the inequality \eqref{eq:ResParaLB2a}--\eqref{eq:ResParaLB2c} 
to get \eqref{eq:ResParaProb11}, a union bound for the step leading from \eqref{eq:ResParaProb11} to \eqref{eq:ResParaProb12}, \eqref{eq:ResParaProb1} and $\vert \mc I \vert = \binom{n_\l-1}{\smaxparlb}$ to obtain \eqref{eq:ResParaProb2} (recall that $\norm[2]{\smash{\mP_{\pp \Gamma'} \yl_{i \pp}}} = \norm[2]{\smash{\Atl_{\Gamma'} \Atl_{\Gamma'} \pinv{\vphantom{\Al}} \atl_i}}$ and $\norm[2]{\smash{\yl_{i }}} = \norm[2]{\smash{\atl_i}}$), and $\binom{n_\l-1}{\smaxparlb} \leq ( (n_\l-1)e / \smaxparlb)^{\smaxparlb}$ to get \eqref{eq:ResParaProb3}. Finally, setting
\begin{equation}
\varepsilon = \!\sqrt{ \frac{1}{9} + \frac{2 \smaxparlb \log ((n_\l - 1)e / \smaxparlb)}{d_\l}} < \frac{1}{3} + \sqrt{ \frac{2 \smaxparlb \log ((n_\l - 1)e / \smaxparlb)}{d_\l}} \nonumber
\end{equation}
in \eqref{eq:ResParaProb3} and noting that 
\begin{align}
1 - \sqrt{\frac{\smaxparlb}{d_\l}} - \varepsilon &> \frac{2}{3} - \sqrt{\frac{\smaxparlb}{d_\l}} - \sqrt{ \frac{2 \smaxparlb \log ((n_\l - 1)e / \smaxparlb)}{d_\l}} \nonumber \\
&> \frac{2}{3} - \sqrt{ \frac{3 \smaxparlb \log ((n_\l - 1)e / \smaxparlb)}{d_\l}}, \nonumber
\end{align}
where we used $\log ((n_\l - 1)e / \smaxparlb) \geq 1$ (as $n_\l - 1 \geq d_\l$ and $\smaxparlb \leq d_\l$) for the last inequality, we have
\begin{align}
\prob{\Fg} \!&=\! 
\mathrm{P}\! \vast[\varphi_{\Gamma'} \!> \norm[2]{\yl_{i \pp}} \! \left(\frac{2}{3} - \sqrt{ \frac{3 \smaxparlb \log ((n_\l - 1)e / \smaxparlb)}{d_\l}} \right)\!, \nonumber \\
& \hspace{4cm} \forall \Gamma' \in \mc I \vast]
\geq\! 1 - e^{-d_\l/18}\!. \nonumber
\end{align}
This completes the proof of Lemma \ref{le:ResParaLB}.
\end{proof}

We continue by deriving a lower bound on $\prob{\smash{\Ef}}$. 

\begin{lemma} \label{le:ResOrthUB} 
Set $\tilde a \defeq \min_{j \in [n_\l] \backslash\{ i \}} \norm[2]{\smash{\yl_{j \pp}}}$ and suppose that $\itr_{\max} \leq c_1 d_\l /\log(e(n_\l-1)/\itr_{\max})$ for a numerical constant $c_1 > 0$. Then, we have
\begin{align}
\prob{\Ef} &= \prob{\norm[2]{\rosl} \leq \norm[2]{\zl_{i \po}} + \frac{3 \sigma}{\tilde a} \norm[2]{\yl_i}, \;  \forall \itr \leq \itr_{\max}} \nonumber \\
&\geq 1 - 2e^{-c_2 m} - 2e^{-c_3 d_\l},\label{eq:LeResOrthUB}
\end{align}
where $c_2, c_3 > 0$ are numerical constants.
\end{lemma}
\begin{proof}
First note that for $s = 0$, $\vr^{(\l)}_{0 \po} = \yl_{i \po} =\zl_{i \po}$\vspace{0.075cm} and the inequality $\norm[2]{\smash{\rosl}} \leq \norm[2]{\smash{\zl_{i \po}}} + (3 \sigma/{\tilde a}) \norm[2]{\smash{\yl_i}}$ holds trivially. For $1 \leq s \leq \itr_{\max}$, as in the proof of Lemma \ref{le:ResParaLB}, we consider fixed index sets $\Gamma \in \mc J$, with
\begin{equation}
\mc J \defeq \{ \Gamma' \subset [n_\l] \backslash \{ i \} \colon \vert \Gamma' \vert = s \in [\itr_{\max}] \}, \label{eq:SIdxSets}
\end{equation} to resolve the issue of statistical  dependencies (between the columns of $\Yl_{\Lambda_\itr}$ and $\yl_i$) in \vspace{0.075cm} $\norm[2]{\smash{\rosl}} = \norm[2]{\smash{\Po (\mI - \Yl_{\Lambda_\itr} \pinv{\Yl_{\Lambda_\itr}}) \yl_i}}$. Specifically, this is accomplished 
by upper-bounding\vspace{0.05cm} $\| \Po (\mI - \allowbreak \Yl_\Gamma \pinv{\Yl_\Gamma}) \yl_i \|_2$ according to \eqref{eq:ResOrthUB-1} and establishing that the submatrix $\Yl_\Gamma$ of $\Yl$ is well-conditioned with high probability for all $\Gamma \in \mc J$, in particular for the sets $\Lambda_\itr \in \mc J$, $\itr \in [\itr_{\max}]$, which determine $\rosl$. This will be accomplished by employing the \ac{rip} \cite[Sec. 5.6]{vershynin2012nonasym} and standard concentration of measure results from random matrix theory.

By the triangle inequality and the submultiplicativity of the operator norm, we have, for every $\Gamma \in \mc J$, that
\begin{align}
&\norm[2]{\Po (\mI - \Yl_\Gamma \pinv{\Yl_\Gamma}) \yl_i} \nonumber\\
&\qquad\leq \norm[2]{\yl_{i \po}} + \norm[2 \to 2]{\Po \Yl_{\Gamma}} \norm[2 \to 2]{\pinv{\Yl_\Gamma}} \norm[2]{\yl_i} \label{eq:ResOrthUB-1} \\
&\qquad= \norm[2]{\zl_{i \po}} + \frac{1}{\minsingv{\Yl_\Gamma}} \norm[2 \to 2]{\Po \Zl_{\Gamma}} \norm[2]{\yl_i} \label{eq:ResOrthUB0} \\
&\qquad\leq \norm[2]{\zl_{i \po}} +  \frac{1}{\minsingv{\Atl_\Gamma}} \norm[2 \to 2]{\Zl_{\Gamma}} \norm[2]{\yl_i}, \label{eq:ResOrthUB}
\end{align}
where we used $\minsingv{\Yl_\Gamma} = 1/ \norm[2 \to 2]{\smash{\pinv{\Yl_\Gamma}}}$ \cite[Sec. 5.2.1]{vershynin2012nonasym}\vspace{0.05cm} to get \eqref{eq:ResOrthUB0} and $\minsingv{\Yl_\Gamma} \geq \minsingv{\Atl_\Gamma} > 0$ w.p. $1$ (where the first inequality is a consequence\vspace{0.05cm} of $\norm[2]{\smash{\Yl_\Gamma} \vv} \geq \norm[2]{\smash{\Pp \Yl_\Gamma} \vv} = \norm[2]{\smash{\Atl_\Gamma \vv}}$, for all $\vv \in \reals^s$, 
and the second stems from the fact that $\Atl_\Gamma$ has full column rank w.p. $1$) 
to get \eqref{eq:ResOrthUB}. Denote the elements of $\Gamma$ by $\gamma_j$, $j \in [s]$. We continue by decomposing $\Atl_\Gamma$ according to $\Atl_\Gamma = \mE_\Gamma \mD_\Gamma$, where 
$\mE \defeq [ \atl_{1}/\norm[2]{\smash{\atl_{1}}} \; \, \atl_{2}/\norm[2]{\smash{\atl_{2}}} \; \, \dots \; \, \atl_{n_\l}/\norm[2]{\smash{\atl_{n_\l}}}]$\vspace{0.075cm},
 and 
 $\mD_\Gamma \defeq \mathrm{diag}(\norm[2]{\smash{\atl_{\gamma_1}}}, \norm[2]{\smash{\atl_{\gamma_2}}}, \dots, \norm[2]{\smash{\atl_{\gamma_s}}})$. Note that the columns of $\mE$ are distributed i.i.d. uniformly at random on $\US{d_\l}$ and $\minsingv{\mD_\Gamma} \geq \tilde a$. We next establish that $\Atl$ and $\Zl$ satisfy the \ac{rip} with high probability, which will then allow us to bound $\minsingv{\Atl_\Gamma}$ and $\maxsingv{\Zl_\Gamma}$, respectively, in \eqref{eq:ResOrthUB}, for all $\Gamma \in \mc J$. 

We start by recalling the definition of the \ac{rip}.

\begin{definition} 
A matrix $\mA \in \reals^{m \times n}$ satisfies the \ac{rip} of order $p \geq 1$ if there exists $\delta_p > 0$  such that
\begin{equation}
(1-\delta_p) \norm[2]{\vv}^2 \leq \norm[2]{\mA \vv}^2 \leq (1+\delta_p) \norm[2]{\vv}^2 \label{eq:ripdef}
\end{equation}
holds for all $\vv \in \reals^n$ with $\norm[0]{\vv} \leq p$. The smallest number $\delta_p = \delta_p(\mA)$ satisfying \eqref{eq:ripdef} is called the restricted isometry constant of $\mA$.
\end{definition}
If $\mA \in \reals^{m \times n}$ satisfies the \ac{rip} of order $p$, it follows from \cite[Lem. 5.36]{vershynin2012nonasym} that for $\delta \in [\delta_p,1]$, 
\begin{align}
&1 - \delta \leq \minsingv{\mA_{\mc T}} \leq \maxsingv{\mA_{\mc T}} \leq 1+ \delta, \nonumber \\
&\qquad \qquad \qquad \qquad \qquad \text{for all $\mc T \subseteq [n]$ 	with $\vert \mc T \vert \leq p$}. \label{eq:RIPSingVals}
\end{align}

By \cite[Ex. 5.25]{vershynin2012nonasym} the rows of $(\sqrt{m} / \sigma) \Zl_{-i} \in \reals^{m \times (n_\l-1)}$ are independent sub-gaussian isotropic random vectors \cite[Def. 5.19, Def. 5.22]{vershynin2012nonasym}, and by \cite[Ex. 5.25]{vershynin2012nonasym} the columns of $\sqrt{d_\l} \mE_{-i} \in \reals^{d_\l \times (n_\l-1)}$ are independent sub-gaussian isotropic random vectors with $\ell_2$-norm $\sqrt{d_\l}$ a.s. 
We can therefore apply the next lemma to show that $(\sqrt{m} / \sigma) \Zl_{-i}$ and $\sqrt{d_\l} \mE_{-i}$ satisfy the \ac{rip} for suitable $p$ and $\delta$ with high probability. This will then allow us to bound $\minsingv{\Atl_\Gamma}$ and $\maxsingv{\Zl_\Gamma}$ for all $\Gamma \in \mc J$.

\begin{lemma}[{\cite[Thm. 5.65]{vershynin2012nonasym}}] \label{le:RIPSubgaussian}
Let $\mA \in \reals^{m \times n}$ be a random matrix with independent sub-gaussian isotropic random vectors 
as rows or independent sub-gaussian isotropic random vectors with $\ell_2$-norm $\sqrt{m}$ a.s. as columns. 
Select $p$ with $1 \leq p \leq n$ and let $\delta \in (0,1)$. If
\begin{equation}
m \geq C \delta ^{-2} p \log( e n / p), \nonumber
\end{equation}
then, w.p. at least $1-2e^{-c \delta^2 m}$, the normalized matrix $\bar \mA \defeq (1/\sqrt{m}) \mA$ satisfies the \ac{rip} of order $p$ with $\delta_p(\bar \mA) \leq \delta$. Here, the constants $c, C > 0$ depend only on the sub-gaussian norm\footnote{\label{fn:subgaussian}
The sub-gaussian norm of  a random variable $X$ is defined as
$\norm[\Psi_2]{X} \defeq \sup_{p \geq 1} p^{-1/2}(\mb E [\abs{X}^p])^{1/p}$ \cite[Def. 5.7]{vershynin2012nonasym}.} 
of the rows or columns of $\mA$.
\end{lemma} 
By Lemma \ref{le:RIPSubgaussian} with $\mA = (\sqrt{m} / \sigma) \Zl_{-i}$ and $\delta = 1/2$, and \eqref{eq:RIPSingVals}, if $\itr_{\max} \leq C_2 m / \log(e(n_\l-1)/\itr_{\max})$, there exist constants $c_2, C_2 > 0$ such that
\begin{align}
&\prob{\frac{1}{2} \sigma \leq \minsingv{\Zl_\Gamma} \leq \maxsingv{\Zl_\Gamma} \leq \frac{3}{2} \sigma, \; \; \forall \Gamma \in \mc J} \nonumber \\
&\hspace{5cm} \geq 1 - 2e^{-c_2 m}. \label{eq:ZGammaSingVals}
\end{align} 
Setting $\mA = \sqrt{d_\l} \mE_{-i}$ and $\delta = 1/2$ in Lemma \ref{le:RIPSubgaussian}, we have similarly
\begin{equation}
\prob{\frac{1}{2} \leq \minsingv{\mE_\Gamma} \leq \maxsingv{\mE_\Gamma} \leq \frac{3}{2}, \; \forall \Gamma \in \mc J} \!\geq\! 1 - 2e^{-c_3 d_\l}, \label{eq:BSingVals}
\end{equation}
if $\itr_{\max} \leq C_3 d_\l / \log(e(n_\l-1)/\itr_{\max})$, for constants $c_3, C_3 > 0$. Putting \eqref{eq:ZGammaSingVals} and \eqref{eq:BSingVals} together, 
we get \eqref{eq:LeResOrthUB} as follows
\begin{align}
\prob{\Ef} &= \mathrm{P}\! \Bigg[\norm[2]{\Po (\mI-\Yl_{\Lambda_\itr} \pinv{\Yl_{\Lambda_\itr}}) \yl_i} \nonumber \\ 
& \qquad \qquad \qquad \qquad \leq \norm[2]{\zl_{i \po}} + \frac{3 \sigma}{\tilde a} \norm[2]{\yl_i} \Bigg] \nonumber \\
&\geq \prob{  \frac{\maxsingv{\Zl_{\Lambda_\itr}}}{\minsingv{\Atl_{\Lambda_\itr}}} \leq \frac{3 \sigma}{\tilde a }} \label{eq:OrthProbBndLB1} \\
&\geq \prob{ \frac{\maxsingv{\Zl_{\Lambda_\itr}}}{\minsingv{\mE_{\Lambda_\itr}}} \leq 3 \sigma} \label{eq:OrthProbBndLB2} \\
&\geq \prob{ \left\{ \minsingv{\mE_{\Lambda_\itr}} \geq \frac{1}{2} \right\} \cap \left\{ \maxsingv{\Zl_{\Lambda_\itr}} \leq \frac{3}{2} \sigma\right\} } \nonumber \\
&\geq 1 - \prob{ \minsingv{\mE_{\Lambda_\itr}} < \frac{1}{2}} - \prob{ \maxsingv{\Zl_{\Lambda_\itr}} > \frac{3}{2} \sigma} \label{eq:OrthProbBndLB3}\\
&\geq 1 - 2e^{-c_2 m} - 2e^{-c_3 d_\l}, \label{eq:OrthProbBndLB4}
\end{align}
for all $\itr \leq \itr_{\max}$. Here \eqref{eq:OrthProbBndLB1} follows from \eqref{eq:ResOrthUB} with $\Gamma = \Lambda_\itr$, \eqref{eq:OrthProbBndLB2} is by 
$$\frac{1}{\minsingv{\Atl_{\Lambda_\itr}}}
\leq \frac{1}{\tilde a \minsingv{\mE_{\Lambda_\itr}}},$$
where we used
\begin{align*}
\minsingv{\Atl_{\Lambda_\itr}} 
&= \min_{\vv \in \reals^\itr}\norm[2]{\mE_{\Lambda_\itr}\mD_{\Lambda_\itr} \vv} \\
&= \min_{\vv \in \reals^\itr} \norm[2]{\mE_{\Lambda_\itr} \frac{\mD_{\Lambda_\itr} \vv}{\norm[2]{\mD_{\Lambda_\itr} \vv}}}\norm[2]{\mD_{\Lambda_\itr} \vv} \\
&\geq \min_{\vv \in \reals^\itr} \norm[2]{\mE_{\Lambda_\itr} \frac{\mD_{\Lambda_\itr} \vv}{\norm[2]{\mD_{\Lambda_\itr} \vv}}} \tilde a \\
&= \tilde a \minsingv{\mE_{\Lambda_\itr}},
\end{align*}
\eqref{eq:OrthProbBndLB3} is by a union bound, and \eqref{eq:OrthProbBndLB4} follows from \eqref{eq:ZGammaSingVals} and \eqref{eq:BSingVals} with $\Lambda_\itr \in \mc J$ for all $\itr \leq \itr_{\max}$. Finally, letting $c_1 \defeq \min \{ 2 C_2 , C_3 \} \leq \min \{ (m/d_\l) C_2 , C_3 \}$ (using the assumption $m \geq 2 d_{\max}$) ensures that $\itr_{\max}$, and thereby $| \Gamma |$, is small enough for both \eqref{eq:ZGammaSingVals} and \eqref{eq:BSingVals} to hold. This concludes the proof of Lemma \ref{le:ResOrthUB}.
\end{proof}

We proceed by establishing a lower bound on $\prob{\smash{\Ee}}$.
\begin{lemma} \label{le:RHSLB} 
For numerical constants $c_4,c_5 > 0$ it holds that
\begin{align}
\prob{\Ee} &= \mathrm{P} \Bigg[ \max_{j \in [n_\l] \backslash (\Lambda_s \cup \{ i \}) } \abs{\innerprod{\yl_j}{\rsl}} \nonumber \\
& \qquad \qquad \geq \left(1-\frac{c_4+1}{\sqrt{\rho_\l}}\right)\frac{\norm[2]{\rpsl}}{\sqrt{d_\l}}\nonumber \\ 
& \qquad \qquad \qquad - \sigma \left(\frac{1}{\sqrt{m}} + \frac{2}{\sqrt{n_\l - 1}} \right) \norm[2]{\rsl} \Bigg] \nonumber \\
&\geq 1 - 2e^{-c_5 d_\l} - 2e^{-m/2}.  \label{eq:LemRHSCorSelCondLB}
\end{align}
\end{lemma}
\begin{proof} We first lower-bound the maximum in \vspace{0.08cm}\eqref{eq:LemRHSCorSelCondLB} 
by a term proportional to $\norm[2]{\smash{\transp{\Yl_{-i}} \rsl}}$ and then establish \eqref{eq:LemRHSCorSelCondLB} by leveraging standard bounds on the singular values of random matrices. We have
\begin{align}
&\max_{j \in [n_\l] \backslash (\Lambda_s \cup \{ i \})} \abs{\innerprod{\yl_j}{\rsl}} = \max_{j \in [n_\l] \backslash \{ i \}} \abs{\innerprod{\yl_j}{\rsl}} \nonumber \\
&\qquad= \norm[\infty]{\transp{\Yl_{-i}} \rsl} \nonumber \\
&\qquad\geq \frac{\norm[2]{\transp{\Yl_{-i}} \rsl}}{\sqrt{n_\l - 1}} \nonumber\\
&\qquad= \frac{\norm[2]{\transp{\Al_{-i}} \transp{\Ul} \rpsl + \transp{\Zl_{-i}} \rsl}}{\sqrt{n_\l - 1}} \nonumber \\
&\qquad\geq \frac{\norm[2]{\transp{\Al_{-i}} \transp{\Ul} \rpsl}}{\sqrt{n_\l - 1}} - \frac{\norm[2]{\transp{\Zl_{-i}} \rsl}}{\sqrt{n_\l - 1}} \nonumber \\
&\qquad\geq \frac{\minsingvb{\transp{\Al_{-i}}}}{\sqrt{n_\l - 1}} \norm[2]{\rpsl} - \frac{\maxsingvb{\transp{\Zl_{-i}}}}{\sqrt{n_\l - 1}} \norm[2]{\rsl}, \label{eq:RHSLBlinalg}
\end{align}
where the first equality is thanks to orthogonality of $\smash{\rsl}$ and $\yl_j$, 
for all $j \in \Lambda_s$, the first inequality follows from $\norm[2]{\vv} \leq \sqrt{n_\l-1} \norm[\infty]{\vv}$, for all $\vv \in \reals^{n_\l -1}$, and the second inequality is by the reverse triangle inequality.

Noting that $\sqrt{d_\l} \transp{\Al_{-i}}$ is a $(n_\l - 1) \times d_\l$ matrix whose rows are independent isotropic subgaussian random vectors (as defined in \cite[Def. 5.19, Def. 5.22]{vershynin2012nonasym}), it follows from \cite[Thm. 5.39]{vershynin2012nonasym} (see Theorem \ref{thm:SupSingval} in Appendix \ref{sec:suppmat}) that 
\begin{equation} \label{eq:SubgaussRowMinSing}
\PR{\sqrt{d_\l} \minsingvb{\transp{\Al_{-i}}} < \sqrt{n_\l - 1} - c_4 \sqrt{d_\l} - t} < 2 e^{-c_5t^2}\!,
\end{equation}
where the constants $c_4, c_5 > 0$ depend only on the sub-gaussian norm of the rows of $\transp{\Al_{-i}}$. Setting $t = \sqrt{d_\l}$ in \eqref{eq:SubgaussRowMinSing}, we get
\begin{align} \label{eq:RHSLBfirst}
&\PR{\frac{\minsingvb{\transp{\Al_{-i}}}}{\sqrt{n_\l - 1}} \norm[2]{\rpsl} < \left(1-\frac{c_4+1}{\sqrt{\rho_\l}}\right)\frac{\norm[2]{\rpsl}}{\sqrt{d_\l}}} \nonumber \\ 
& \hspace{6cm}< 2 e^{-c_5 d_\l}.
\end{align}

Since $(\sqrt{m}/\sigma) \transp{\Zl_{-i}}$ is a $(n_\l - 1) \times m$ matrix with i.i.d. standard normal entries, it follows from \cite[Cor. 5.35]{vershynin2012nonasym} (see Corollary \ref{thm:SupGaussSingval} in Appendix \ref{sec:suppmat}) that
\begin{equation} \label{eq:GaussMaxSing}
\PR{\frac{ \sqrt{m}}{\sigma} \maxsingvb{\transp{\Zl_{-i}}} > \sqrt{n_\l - 1} + \sqrt{m} + t} < 2e^{-t^2/2}.
\end{equation}
Setting $t = \sqrt{m}$ in \eqref{eq:GaussMaxSing}, we obtain
\begin{align} \label{eq:RHSLBsecond}
&\!\!\PR{\frac{\maxsingvb{\transp{\Zl_{-i}}}}{\sqrt{n_\l - 1}} \norm[2]{\rsl} \!>\! \sigma \left(\frac{1}{\sqrt{m}} + \frac{2}{\sqrt{n_\l - 1}} \right) \norm[2]{\rsl} } \nonumber \\
& \hspace{6cm}< 2e^{-m/2}.
\end{align}
The claim in Lemma \ref{le:RHSLB} now follows by lower-bounding the first term in \eqref{eq:RHSLBlinalg} using \eqref{eq:RHSLBfirst}, by upper-bounding the second term in \eqref{eq:RHSLBlinalg} using \eqref{eq:RHSLBsecond}, and by application of a union bound.
\end{proof}

Finally, we derive a lower bound on $\prob{\smash{\Ed}}$.
\begin{lemma}\label{le:Ed} 
We have 
\begin{equation}
\prob{\Ed} \geq 1 - \sum_{\l \in [L]} n_\l ( e^{- d_\l/8} + e^{-m/8} ). \label{eq:EdLB}
\end{equation}
\end{lemma}

\begin{proof}
The proof is effected by applying the following well-known concentration result.
\begin{theorem}[{\cite{ledoux2005concentration}}] \label{th:ConMeasure} Let $f \colon \reals^m \to \reals$ be a Lipschitz function with Lipschitz constant $K$, i.e., $\abs{f(\va) - f(\vb)} \leq K \norm[2]{\va- \vb}$, for all $\va, \vb \in \reals^m$. Let $\vz \in \reals^m$ be a $\mc N(\mathbf 0, \mI_m)$ vector. Then, for $t \geq 0$, we have
\begin{equation}
\prob{f(\vz) - \mb E [f(\vz)] > t} \leq e^{-t^2/(2 K^2)}. \label{eq:conclipschitz}
\end{equation}
\end{theorem}
The functions $f(\vz) = \norm[2]{\vz}$ and $f_\pp (\vz) = \norm[2]{\smash \Pp \vz}$ both have Lipschitz constant $K = 1$ ($\abs{\smash{f_\pp(\va) - f_\pp(\vb)}} = \abs{\norm[2]{\smash{\Pp \va}}-\norm[2]{\smash{\Pp \vb}}} \leq \norm[2]{\smash{\Pp(\va-\vb)}} \leq \norm[2]{\va-\vb}$, for all $\va,\vb \in \reals^m$, by the reverse triangle inequality and $\norm[2 \to 2]{\smash{\Pp}} = 1$). For $\vz \sim \mc N(\mathbf 0, \mI_m)$, we get by Jensen's inequality $\mb E[\norm[2]{\smash{ \vz}}] \leq \sqrt{ \mb E [ \norm[2]{\smash{ \vz} }^2 ] } = \sqrt{m}$ and $\mb E[\norm[2]{\smash{ \Pp \vz}}] \leq \sqrt{ \mb E [\norm[2]{\smash{ \Pp \vz}}^2 ] } =\sqrt{d_\l}$ (where the equality follows from $\Pp \vz = \Ul \transp{\Ul} \vz$ and the fact that $\transp{\Ul} \vz$ is $\mc N (0, \mI_{d_\l})$-distributed). Noting that $\zl_i \sim (\sigma / \sqrt{m}) \vz$, application of \eqref{eq:conclipschitz} to $\vz$ and $\Pp \vz$ 
yields
\begin{align}
&\prob{\norm[2]{\zl_i} > \frac{3}{2} \sigma} \leq e^{-m/8} \qquad \text{and} \nonumber \\
&\prob{\norm[2]{\Pp \zl_i} > \frac{3 \sqrt{d_\l}}{2 \sqrt{m}} \sigma} \leq e^{-d_\l/8},\nonumber
\end{align}
where we set  $t = \sqrt{m}/2$ and $t = \sqrt{d_\l}/2$, respectively. A union bound over $\l \in[L]$, $i \in [n_\l]$, yields the desired lower bound \eqref{eq:EdLB}.
\end{proof}

\section{Proof of Theorem \ref{th:SSCMP}}  \label{sec:pfsscmp}

Most steps of the proof of Theorem~\ref{th:SSCMP} are almost identical to those in the proof of Theorem~\ref{th:SSCOMP}. We therefore elaborate only on the arguments the proofs differ in significantly. 

Analogously to the proof for \ac{ssc}-\ac{omp}, we will henceforth work with the ``reduced \ac{mp}'' algorithm which, for the representation of $\yl_i$ selects elements from the reduced dictionary $\cY_\l \backslash \{\yl_i\}$ only, instead of the full dictionary $\cY \backslash \{\yl_i\}$. The justification for relying on reduced \ac{mp} to establish the desired result is identical to that for \ac{omp}. Throughout the proof the residual of the reduced \ac{mp} algorithm will be denoted by $\qsl$ and the number of iterations actually performed when a stopping condition has been met by $\send$. As in the \ac{omp} case, for expositional convenience, the quantities $\qsl$ and $\send$ do not reflect the dependence on the index $i$ of the data point $\yl_i$. 

Note that the selection rules for \ac{omp} and \ac{mp} are equivalent in the following sense. As the \ac{omp} residual $\smash{\rsl}$ is orthogonal to $\yl_j$,\vspace{0.05cm} for all $j \in \Lambda_s$, we can replace $\max_{j \in [N] \backslash (\Lambda_s \cup \{ i \})} \abs{\innerprod{\smash{\yl_j}}{\smash{\rs}}}$\vspace{0.08cm} on the \ac{rhs} of \eqref{eq:CorrSelCond} by $\max_{j \in [N] \backslash \{ i \}} \abs{\innerprod{\smash{\yl_j}}{\smash{\rs}}}$,\vspace{0.05cm} i.e., we can take the maximization over $j \in [N] \backslash \{ i \}$ as in \ac{mp}. We therefore need to show that \eqref{eq:CorrSelCond} with $\rs$ replaced by $\qs$ and $\max_{j \in [N] \backslash \! (\Lambda_s \cup \{ i \}\!)}$ replaced by $\max_{j \in [N] \backslash \{ i \}}$ holds for all \ac{mp} iterations, and for every $\yl_i \in \cY_\l$, $\l \in [L]$, w.p. at least $P^\star$. 

Next, we systematically revisit the events $\Ea$-- $\Eg$ and adapt the corresponding bounds for \ac{mp} where needed. Recall that the bounds on $\PR{\smash{\Ea}}$ and $\PR{\smash{\Eb}}$ in Lemma \ref{le:EaEbEc} rely on the rotational invariance---as expressed in Lemma \ref{le:ResRotInv}---of the distributions of $\rpsl$ and $\rosl$, respectively. As the residual update rule \eqref{eq:MPresorth} for \ac{mp} differs from that for \ac{omp} in \eqref{eq:OMPResFormula}, we need to establish rotational invariance for $\qpsl$ and $\qosl$, which will be done in Lemma \ref{le:ResRotInvMP} below. The bounds on $\PR{\smash{\Ec}}$, $\PR{\smash{\Ed}}$, and $\PR{\smash{\Ee}}$ in Lemmata \ref{le:EaEbEc}, \ref{le:Ed}, and \ref{le:RHSLB}, respectively, do not depend on a particular property of $\rsl$ apart from $\norm[2]{\smash{\rsl}} \leq \norm[2]{\smash{\yl_i}}$ in the case of $\PR{\smash{\Ec}}$\vspace{0.05cm} (we have $\norm[2]{\smash{\qsl}} \leq \norm[2]{\smash{\yl_i}}$ as a consequence of \cite[Eq. 13]{mallat1993matching}).\vspace{0.05cm} Thus, the bounds on $\PR{\smash{\Ea}}$--$\PR{\smash{\Ee}}$\vspace{0.05cm} continue to hold for $\rsl$, $\rpsl$, and $\rosl$ in $\Ea$--$\Ee$\vspace{0.05cm} replaced by $\qsl$, $\qpsl$, and $\qosl$, respectively, and we readily get the upper bound \eqref{eq:CorrSelCondLHSUB} on the \ac{lhs} of \eqref{eq:CorrSelCond} and the lower bound \eqref{eq:CorrSelCondRHSLB} on the \ac{rhs} of \eqref{eq:CorrSelCond}. The bounds on $\norm[2]{\smash{\rpsl}}$ and $\norm[2]{\smash{\rosl}}$ in $\Ef$ and $\Eg$, respectively, in the proof of Theorem~\ref{th:SSCOMP} require more work. 
In particular, as the stopping behavior of \ac{mp} is different from that of \ac{omp}, we need the corresponding bounds on $\norm[2]{\smash{\qosl}}$ and $\norm[2]{\smash{\qpsl}}$ to hold for all \ac{mp} iterations $\itr \in[\itr_\mathrm{a}]$ and for a maximum sparsity level of $\mpspar$. 
This will be accomplished by deriving an upper bound on $\norm[2]{\smash{\qosl}}$\vspace{0.05cm} and a lower bound on $\norm[2]{\smash{\qpsl}}$ leading to suitably modified events $\Eft$ and $\Egt$, respectively, as defined below. Specifically, the resulting upper bound on $\norm[2]{\smash{\qosl}}$ is slightly weaker than that on $\norm[2]{\smash{\rosl}}$ in $\Ef$, but exhibits the same scaling behavior in $\sigma$, whereas the resulting lower bound on $\norm[2]{\smash{\qpsl}}$ is identical to the one on $\norm[2]{\smash{\rpsl}}$ in $\Eg$.

We proceed by introducing the modified events 
\begin{align}
&\Eft \defeq \left\{ \norm[2]{\qosl} \leq \norm[2]{\zl_{i \po}} + \frac{6 \sigma}{\tilde a} \norm[2]{\yl_i}, \;\; \forall \itr \leq \send \right\}, \quad \text{and} \nonumber \\
&\Egt \defeq \vast\{ \norm[2]{\qpsl} \nonumber \\
&> \norm[2]{\yl_{i \pp}} \!\left(\frac{2}{3} - \!\sqrt{ \frac{3 \mpspar \log ((n_\l - 1)e / \mpspar)}{d_\l}} \right)\!, \forall \itr \leq \send \vast\}\!, \label{eq:qplbdef}
\end{align}
where $\tilde a \defeq \min_{j \in [n_\l] \backslash\{ i \}} \norm[2]{\smash{\yl_{j \pp}}}$,\vspace{0.1cm} and deriving lower bounds on $\PR{\smash{\Eft}}$ and $\PR{\smash{\Egt}}$ in Lemmata \ref{le:ResParaMPLB} and \ref{le:ResOrthMPUB}, respectively. These lower bounds will turn out to be identical to those for \ac{ssc}-\ac{omp}. Although the corresponding proofs rely on arguments similar in spirit to those used for \ac{omp}, the technical details are dissimilar enough to warrant detailed presentation.

We continue by establishing the bounds on $\norm[2]{\smash{\qpsl}}$ and $\norm[2]{\smash{\qosl}}$, as announced. Using the assumptions $m \geq 2 d_{\max}$, $\sigma \leq 1/2$, and $\mpspar \leq \min_{\l \in [L]} \{ c_s d_\l / \log( (n_\l - 1) e /\mpspar) \}$, we have on $\Ed \cap \Eft$ that
\begin{equation}
\norm[2]{\qosl} \leq \frac{3}{2} \sigma + 6 \sigma \frac{1 + \frac{3}{2} \sigma}{1 - \frac{3}{2} \frac{\sqrt{d_\l}}{\sqrt{m}} \sigma} \leq \sigma (15 + 20 \sigma) \label{eq:qoslSimpleUB}
\end{equation}
and, by repeating the steps in \eqref{eq:respSimpleLB}, we get $\norm[2]{\smash{\qpsl}} > 1/20$ on $\Egt$. 
It therefore follows that \eqref{eq:CorrSelCond} for \ac{mp} is implied by \eqref{eq:ClusCondThm} with $c(\sigma) = 17 + 23 \sigma$ on $\tilde \Estar \defeq \bigcap_{\l,i,s} (\Ea \cap \Eb \cap \Ec \cap \Ed \cap \Ee \cap \Eft \cap \Egt)$. The proof is completed by lower-bounding $\PR{\smash{\tilde \Estar}}$ via a union bound. 

We proceed by establishing the rotational invariance properties of $\qpsl$ and $\qosl$\vspace{0.05cm} needed to establish the lower bounds on $\PR{\smash{\Eft}}$ and $\PR{\smash{\Egt}}$ in Lemmata \ref{le:ResParaMPLB} and \ref{le:ResOrthMPUB}, respectively.

\begin{lemma} \label{le:ResRotInvMP} 
The distributions of $\qpsl$ and $\qosl$ are rotationally invariant on $\cS_\l$ and $\cS_\l^\perp$, respectively, i.e., for unitary transformations $\mV^\pp, \mV^\po \in \reals^{m \times m}$ 
of the form specified in Lemma \ref{le:ResRotInv}, we have $\mV^\pp \qpsl \sim \qpsl$ and $\mV^\po \qosl \sim \qosl$.
\end{lemma}

\begin{proof}
The arguments employed in this proof are similar to those in the proof of the corresponding result for \ac{omp}, Lemma \ref{le:ResRotInv}, but the structure of the proof differs as the \ac{mp} residual $\qsl$ can only be expressed recursively, i.e., as a function of previous residuals. In contrast, the reduced \ac{omp} residual $\rsl$ can be written as the projection of $\yl_i$ onto the orthogonal complement of the span of the data points indexed by $\Lambda_s$. Throughout the proof, $\omega_s(\vx, \mD)$ denotes the index obtained by the \ac{mp} algorithm in iteration $\itr$ when applied to $\vx$ with the columns of $\mD$ as dictionary elements, and $\qs(\vx, \mD)$ is the corresponding residual.

We first establish results analogous to \eqref{eq:ResRotInv1} and \eqref{eq:ResRotInv2}. 
Again, the proof is effected through induction. We start with the inductive step and assume that 
\begin{equation}
\qspr(\mV \yl_i, \mV \Yl_{-i}) = \mV \qspr (\yl_i, \Yl_{-i}) \label{eq:RotInvAssqsl}
\end{equation}
 for fixed $\itr' < \send$, for all unitary matrices $\mV \in \reals^{m \times m}$. For iteration $s'+1$ we then have
\begin{align}
&\omega_{\itr'+1}(\mV \yl_i, \mV \Yl_{-i}) \nonumber \\
&\qquad\quad= \underset{j \in [n_\l] \backslash \{ i \}}{\arg \max} \abs{\innerprod{\mV \yl_j}{\q{\itr'}(\mV \yl_i, \mV \Yl_{-i}) }} \nonumber \\
&\qquad\quad= \underset{j \in [n_\l] \backslash \{ i \}}{\arg \max} \abs{\innerprod{ \yl_j}{\q{\itr'}(\yl_i, \Yl_{-i})}} \nonumber \\
&\qquad\quad= \omega_{\itr'+1}(\yl_i, \Yl_{-i}). \label{eq:MPselidxinv}
\end{align}
Now, using the shorthands $\omega_{\itr'+1} (\mV)$ and $\omega_{\itr'+1}$ for $\omega_{\itr'+1} (\mV \yl_i, \mV \Yl_{-i})$ and $\omega_{\itr'+1} (\yl_i, \Yl_{-i})$, respectively, we get
\begin{align}
&\q{\itr'+1}(\mV \yl_i, \mV \Yl_{-i}) = \q{\itr'}(\mV \yl_i, \mV \Yl_{-i}) \nonumber \\
&\quad-\innerprod{\! \q{\itr'}(\mV \yl_i, \mV \Yl_{-i})}{\frac{\mV \yl_{\omega_{\itr'+1}(\mV)}}{\norm[2]{\mV \yl_{\omega_{\itr'+1}(\mV)}}}} \! \frac{\mV \yl_{\omega_{\itr'+1}(\mV)}}{\norm[2]{\mV \yl_{\omega_{\itr'+1}(\mV)}}} \nonumber \\
&= \q{\itr'}(\mV \yl_i, \mV \Yl_{-i}) \nonumber \\
& \qquad\quad- \innerprod{\q{\itr'}(\mV \yl_i, \mV \Yl_{-i})}{\frac{\mV \yl_{\omega_{\itr'+1}}}{\norm[2]{\yl_{\omega_{\itr'+1}}}}} \frac{\mV \yl_{\omega_{\itr'+1}}}{\norm[2]{\yl_{\omega_{\itr'+1}}}} \nonumber \\
&= \mV \vast( \q{\itr'}(\yl_i, \Yl_{-i}) \nonumber \\
&\qquad\quad- \innerprod{\q{\itr'}( \yl_i, \Yl_{-i})}{\frac{ \yl_{\omega_{\itr'+1}}}{\norm[2]{\yl_{\omega_{\itr'+1}}}}} \frac{\yl_{\omega_{\itr'+1}}}{\norm[2]{\yl_{\omega_{\itr'+1}}}} \vast) \nonumber \\
&= \mV \q{\itr' + 1}(\yl_i, \Yl_{-i}), \label{eq:MPresrot}
\end{align}
where the second and third equality follow from \eqref{eq:MPselidxinv} and \eqref{eq:RotInvAssqsl}, respectively. 
Now, the base case is $\q{0}(\mV \yl_i, \mV \Yl_{-i}) = \mV \yl_i = \mV \q{0}(\yl_i, \Yl_{-i})$, and we therefore established that 
\begin{equation}
\q{\itr}(\mV \yl_i, \mV \Yl_{-i}) = \mV \q{\itr}(\yl_i, \Yl_{-i}),
\end{equation}
for all $\itr \leq \send$ and all unitary $\mV \in \reals^{m \times m}$.

Finally, repeating the steps leading from \eqref{eq:ResDistInv1} to \eqref{eq:ResDistInv2} for $\qs$ and $\qps$ instead of $\rs$ and $\rps$, respectively, yields the desired result. 
\end{proof}

We continue with the lower bound on $\PR{\smash{\Egt}}$. 
\begin{lemma}
\label{le:ResParaMPLB}
Let $n_\l > 1$. We have
\begin{align}
	&\PR{\Egt} =\mathrm{P}\! \vast[ \norm[2]{\qpsl} \nonumber \\
	& > \norm[2]{\yl_{i \pp}} \!\left(\frac{2}{3} - \sqrt{ \frac{3 \mpspar \log ((n_\l - 1)e / \mpspar)}{d_\l}} \right)\!, \, \forall \itr \leq \send \vast] \nonumber \\
	 & \hspace{5.7cm}\geq 1 - e^{-d_\l/18}. \label{eq:LeMPResParaLB}
\end{align}
\end{lemma}

\begin{proof}
We start by recalling that the reduced MP algorithm decomposes $\yl_i$ according to (see, e.g., \cite{mallat1993matching})
\begin{align}
\yl_i &= \sum_{s'=1}^{s} \innerprod{\qsprlm}{\frac{\yl_{\omega_{\itr'}}}{\norm[2]{\yl_{\omega_{\itr'}}}}} \frac{\yl_{\omega_{\itr'}}}{\norm[2]{\yl_{\omega_{\itr'}}}} + \qsl. \label{eq:MPresdecomp}
\end{align}
Denote by $\Omega_\itr$ the set containing the indices of the points in $\cY_\l \backslash \{ \yl_i \}$ selected during the first $s$ iterations, i.e.,
\begin{equation}
\Omega_{\itr} \defeq \{ \omega_{s'} \colon s' \in [s] \}. \label{eq:OmegaDef}
\end{equation}
Note that $\Omega_{s}$ may contain fewer than $s$ indices as reduced MP may select one or more points (from $\cY_\l \backslash \{ \yl_i \}$)\vspace{0.05cm} repeatedly. With $\mP^\pp_{\Omega_\itr} \defeq \Yl_{\Omega_\itr \pp} \pinv{\Yl_{\Omega_\itr \pp}}$, using \eqref{eq:MPresdecomp}, we get
\begin{align}
&\norm[2]{\Pp \qsl} \geq \norm[2]{(\mI - \mP^\pp_{\Omega_\itr}) \Pp \qsl} \nonumber \\
&\quad= \vast\|(\mI - \mP^\pp_{\Omega_\itr}) \Pp \vast( \yl_i \nonumber \\
& \qquad\qquad- \sum_{s'=1}^{s} \innerprod{\qsprlm}{\frac{\yl_{\omega_{\itr'}}}{\norm[2]{\yl_{\omega_{\itr'}}}}} \frac{\yl_{\omega_{\itr'}}}{\norm[2]{\yl_{\omega_{\itr'}}}} \vast) \vast\|_2 \nonumber \\
&\quad= \vast\| 
(\mI - \mP^\pp_{\Omega_\itr}) \yl_{i \pp} \nonumber \\
&\qquad\quad- (\mI - \mP^\pp_{\Omega_\itr}) \underbrace{\left( \sum_{s'=1}^{s} \innerprod{\qsprlm}{\frac{\yl_{\omega_{\itr'}}}{\norm[2]{\yl_{\omega_{\itr'}}}}} \frac{\yl_{\omega_{\itr'} \pp}}{\norm[2]{\yl_{\omega_{\itr'}}}} \right)}_{\in \range(\mP^\pp_{\Omega_\itr})}  \vast\|_2 \nonumber \\ 
&\quad= \norm[2]{(\mI - \mP^\pp_{\Omega_\itr}) \yl_{i \pp}} \nonumber \\
&\quad \geq \norm[2]{\smash{(\mI - \mP^\pp_{\Omega_\send})\yl_{i \pp}}}, \label{eq:qpslLB2} 
\end{align}
where the last inequality is by $\range(\mI - \mP^\pp_{\Omega_\send}) \subseteq \range(\mI - \mP^\pp_{\Omega_\itr})$, for $\itr \leq \send$. 
The proof is now completed by replacing $\Omega_\send$ in \eqref{eq:qpslLB2} by a fixed $\Gamma \in \mc I$ (recall the definition of $\mc I$ from \eqref{eq:SmaxIdxSets}, with $\smaxparlb$ in \eqref{eq:SmaxIdxSets} replaced by $\mpspar$) and by lower-bounding $\norm[2]{\smash{(\mI - \mP^\pp_{\Gamma})\yl_{i \pp}}}$ for all $\Gamma \in \mc I$ as in the proof of Lemma \ref{le:ResParaLB} (following the steps starting from \eqref{eq:ResParaLB2a}). 
\end{proof}

Next, we lower-bound $\PR{\smash{\Eft}}$.

\begin{lemma} \label{le:ResOrthMPUB} 
Set $\tilde a \defeq \min_{j \in [n_\l] \backslash\{ i \}} \norm[2]{\smash{\yl_{j \pp}}}$ and assume that $\mpspar \leq c_1 d_\l /\log(e(n_\l-1)/\mpspar)$ for a numerical constant $c_1 > 0$. Then, we have
\begin{align}
\PR{\Eft} &= \prob{\norm[2]{\qosl} \leq \norm[2]{\zl_{i \po}} + \frac{6 \sigma}{\tilde a} \norm[2]{\yl_i}, \forall \itr \leq \send} \nonumber \\
&\geq 1 - 2e^{-c_2 m} - 2e^{-c_3 d_\l}, \label{eq:LeResOrthMPUB} 
\end{align}
where $c_2, c_3 > 0$ are numerical constants.
\end{lemma}

\begin{proof}
We start by rewriting \eqref{eq:MPresdecomp} as
\begin{align}
\yl_i &= \Yl_{\Omega_{s}} \vb_{\Omega_{s}} + \qsl, \label{eq:MPresdecomp1}
\end{align}
where $\Omega_\itr$ was defined in \eqref{eq:OmegaDef} and $\vb_{\Omega_{s}}$ contains the coefficients of the representation of $\yl_i$ computed by \ac{mp}\vspace{0.05cm} according to \eqref{eq:MPcoeffcomp}, i.e., $[\vb]_\omega = \sum_{\itr' \colon \omega_{\itr'} = \omega} \innerprod{\smash{\yl_{\omega_{\itr'}}}}{\smash{\qsprlm}} / \norm[2]{\smash{\yl_{\omega_{\itr'}}}}^2$, $\omega \in \Omega_\itr$ (this is a direct consequence of \eqref{eq:MPcoeffcomp}; we do not reflect dependence of $\vb$ on $i$, $\l$ for expositional ease). 
Next, note that
\begin{align}
\minsingv{\Atl_{\Omega_s}} \norm[2]{\vb_{\Omega_s}} &\leq \minsingv{\Yl_{\Omega_s}} \norm[2]{\vb_{\Omega_s}} \nonumber \\
&\leq \norm[2]{\Yl_{\Omega_{s}} \vb_{\Omega_{s}}} =\norm[2]{\yl_i - \qsl} \nonumber \\
& \leq \norm[2]{\yl_i} + \norm[2]{\qsl} \leq 2 \norm[2]{\yl_i}, \label{eq:bOmegaUB}
\end{align}
where the first inequality is a consequence of $\norm[2]{\smash{\Yl_{\Omega_s}} \vv} \geq \norm[2]{\smash{\Pp \Yl_{\Omega_s}} \vv} = \norm[2]{\smash{\Pp \Ul \Atl_{\Omega_s} } \vv}  = \norm[2]{\smash{\Atl_{\Omega_s} \vv}}$,\vspace{0.08cm} for all $\vv \in \reals^{| \Omega_s |}$, and the last inequality follows from $\norm[2]{\smash{\qsl}} \leq \norm[2]{\smash{\yl_i}}$ \cite[Eq. 13]{mallat1993matching}. 
For $\minsingv{\Atl_{\Omega_s}} > 0$ (we will justify below that $\minsingv{\Atl_{\Omega_s}}$ is, indeed, bounded away from $0$ with high probability) 
it hence follows that $\norm[2]{\vb_{\Omega_s}} \leq (2/\minsingv{\Atl_{\Omega_s}}) \norm[2]{\smash{\yl_i}}$ and therefore, together with \eqref{eq:MPresdecomp1}, we get 
\begin{align}
\norm[2]{\qosl} &= \norm[2]{\yl_{i \po} - \Yl_{\Omega_s \po} \vb_{\Omega_s}} \nonumber \\
&\leq \norm[2]{\yl_{i \po}} + \norm[2]{\Yl_{\Omega_s \po} \vb_{\Omega_s}} \nonumber \\
&\leq \norm[2]{\yl_{i \po}} + \norm[2 \to 2]{\Yl_{\Omega_s \po}} \norm[2]{\vb_{\Omega_s}} \nonumber \\
&= \norm[2]{\zl_{i \po}} + \norm[2 \to 2]{\Zl_{\Omega_s \po}} \norm[2]{\vb_{\Omega_s}} \nonumber \\
&\leq \norm[2]{\zl_{i \po}} + \frac{2}{\minsingv{\Atl_{\Omega_s}}} \norm[2 \to 2]{\Zl_{\Omega_s \po}} \norm[2]{\yl_i}. \label{eq:QSorthUB} 
\end{align} 
Replacing $\Omega_\itr$ in \eqref{eq:QSorthUB} by a fixed $\Gamma \in \mc J$ (recall the definition of $\mc J$ from \eqref{eq:SIdxSets}, with $\itr_{\max}$ in \eqref{eq:SIdxSets} replaced by $\mpspar$),  \eqref{eq:QSorthUB} and \eqref{eq:ResOrthUB} are equal up to the factor $2$ in the second term of \eqref{eq:QSorthUB}. This factor-of-two difference stems from the update rule for $\qsl$ differing from that for $\rsl$ and leads to the difference in $c(\sigma)$ between Theorems \ref{th:SSCMP} and \ref{th:SSCOMP}. We proceed as in the proof of Lemma \ref{le:ResOrthMPUB} (starting from \eqref{eq:ResOrthUB-1}--\eqref{eq:ResOrthUB}) by establishing bounds on the tail probabilities of  $\norm[2 \to 2]{\smash{\Zl_{\Gamma \po}}}$ 
and $\minsingv{\Atl_{\Gamma}}$, for all $\Gamma \in \mc J$, via \eqref{eq:ZGammaSingVals} and \eqref{eq:BSingVals}, respectively, to obtain the result in Lemma \ref{le:ResOrthMPUB}. 
\end{proof}

\section{Proof of Theorem \ref{co:trueconnections}}  \label{sec:pfcor}

We prove the result for \ac{omp}. The proof for \ac{mp} follows simply by replacing the lower bound on $\norm[2]{\smash{\rpsl}}$ in $\Eg$ 
by the lower bound on \vspace{0.05cm}$\norm[2]{\smash{\qpsl}}$ in $\Egt$ (defined in \eqref{eq:qplbdef}), 
and by noting that $\PR{\smash{\Eg}} = \PR{\smash{\Egt}}$.

We only need to address the case
\begin{equation}
\frac{d_\l}{\log((n_\l-1)e)} \min \left\{\frac{1}{3} \left( \frac{2}{3} - \frac{\tau}{1 - \frac{3}{2}\frac{\sqrt{d_\l}}{\sqrt{m}} \sigma}\right)^2,c_s \right\} \geq 1,
\end{equation}
as otherwise there is nothing to prove. We start by noting that under the conditions of Theorem \ref{co:trueconnections} (which are identical to the conditions of Theorem \ref{th:SSCOMP} minus the condition $\itr_{\max} \leq \min_{\l \in [L]} \{ c_s d_\l / \log( (n_\l - 1) e / \itr_{\max}) \}$), it follows from the proof of Theorem \ref{th:SSCOMP} that, conditionally on $\Estar$ as defined in \eqref{eq:estar}, reduced \ac{omp} and \ac{omp} are guaranteed to select the same points to represent $\yl_i$ during the first $\lfloor c_s d_\l / \log( (n_\l - 1) e) \rfloor \geq \eqref{eq:tdlbcor}$ iterations, for all $i \in [n_\l]$, $\l \in [L]$. 
Therefore, conditionally on $\Estar$, for $\tau$ small enough reduced \ac{omp} will perform a number of iterations lower-bounded by \eqref{eq:tdlbcor}, for all $\yl_i$, $i \in [n_\l]$, $\l \in [L]$, which implies that the number of points from $\cY_\l \backslash \{ \yl_i \}$ selected by \ac{omp} is lower-bounded by \eqref{eq:tdlbcor} as well, for all $\yl_i$, $i \in [n_\l]$, $\l \in [L]$. Specifically, we will show that for $\tau \in [0, 2/3 - (\sqrt{d_{\max}}/\sqrt{m})\sigma]$ the number of reduced \ac{omp} iterations is lower-bounded by \eqref{eq:tdlbcor}.
This will be accomplished by establishing that for $\tau \in [0, 2/3 - (\sqrt{d_{\max}}/\sqrt{m})\sigma]$, on $\Estar$, we have $\norm[2]{\smash{\subsind{\vr}{\l}_{\itr_\tau}}} > \tau$ for all $\yl_i$, $i \in [n_\l]$, $\l \in [L]$, where 
\begin{equation}
\itr_\tau \defeq \left\lfloor \frac{d_\l}{\log((n_\l- 1)e)} \frac{1}{3}\left( \frac{2}{3} - \frac{\tau}{1 - \frac{3}{2}\frac{\sqrt{d_\l}}{\sqrt{m}} \sigma} \right)^2 \right\rfloor.
\end{equation}
Indeed, as the stopping criterion $\max_{\yl_j \in \cY_\l \backslash (\Lambda_{\itr}\cup \{\yl_i\})}   \left| \innerprod{\smash{\yl_j}}{\smash{\rsl}} \right| = 0$ is activated only after $\min\{m,n_\l-1\} > d_\l > \itr_\tau$ iterations 
(as the points in $\cY_\l$ are in general position w.p. $1$), 
$\norm[2]{\smash{\subsind{\vr}{\l}_{\itr_\tau}}} > \tau$ implies that reduced \ac{omp} performs at least $\itr_\tau$ iterations. On the event $\Ed \cap \Eg \supset \Estar$ ($\Ed$ and $\Eg$ are defined in \eqref{eq:DefEd} and \eqref{eq:resplbdef}, respectively), setting $\smaxparlb = \itr_\tau$ in $\Eg$, we have
\begin{align}
\norm[2]{\subsind{\vr}{\l}_{\smaxparlb}} &\geq \norm[2]{\subsind{\vr}{\l}_{\smaxparlb \pp}} \nonumber \\
&> \norm[2]{\yl_{i \pp}} \left(\frac{2}{3} - \sqrt{ \frac{3 \smaxparlb \log ((n_\l - 1)e / \smaxparlb)}{d_\l}} \right) \nonumber \\
&\geq \left(1 - \frac{3}{2}\frac{\sqrt{d_\l}}{\sqrt{m}} \sigma \right) \left(\frac{2}{3} - \sqrt{ \frac{3 \smaxparlb \log ((n_\l - 1)e / \smaxparlb)}{d_\l}} \right) \nonumber \\
&\geq 
 \left(1 - \frac{3}{2}\frac{\sqrt{d_\l}}{\sqrt{m}} \sigma \right) \Vast(\frac{2}{3} \nonumber \\
 &\quad -
 \Vast(
\left\lfloor \frac{d_\l}{\log((n_\l- 1)e)} \frac{1}{3}\left( \frac{2}{3} - \frac{\tau}{1 - \frac{3}{2}\frac{\sqrt{d_\l}}{\sqrt{m}} \sigma} \right)^2 \right\rfloor \nonumber \\
&\hspace{4cm} \cdot \frac{3 \log ((n_\l - 1)e)}{d_\l}\Vast)^{\frac12} \Vast) \nonumber \\
&\geq \left(1 - \frac{3}{2}\frac{\sqrt{d_\l}}{\sqrt{m}} \sigma \right) \left(\frac{2}{3} - \left( \frac{2}{3} - \frac{\tau}{1 - \frac{3}{2}\frac{\sqrt{d_\l}}{\sqrt{m}} \sigma} \right) \right) \nonumber \\
&= \tau,  \label{eq:sendlb}
\end{align}
where for the second inequality we used that $\norm[2]{\smash{\yl_{i \pp}}} \geq \norm[2]{\smash{\xl_i}} - \norm[2]{\smash{\zl_{i \pp}}} \geq 1 - 3\sqrt{d_\l}/(2\sqrt{m})$ on $\Ed$,\vspace{0.075cm} $\log ((n_\l - 1)e / \smaxparlb) \leq \log ((n_\l - 1)e)$, for $\smaxparlb \geq 1$, for the third inequality, and $\tau \leq 2/3 - (\sqrt{d_{\max}}/\sqrt{m})\sigma$ for the last inequality. This completes the proof.

\section{Supplementary notes} \label{sec:suppmat}

\begin{lemma}[{\cite[Lem. 5.24]{vershynin2012nonasym}}] Let $X_1,\ldots,X_n$ be independent centered sub-gaussian random variables (see footnote \ref{fn:subgaussian}). Then, $X = (X_1,\ldots,X_n)$ is a centered sub-gaussian random vector in $\reals^n$, and
\begin{equation}
\norm[\Psi_2]{X} \leq C \max_{i \in[n]} \norm[\Psi_2]{X_i},
\end{equation}
where $C$ is a numerical constant.
\end{lemma}

\begin{theorem}[{\cite[Lem. 5.39]{vershynin2012nonasym}}] \label{thm:SupSingval}
Let $\mA$ be an $m \times n$ matrix whose rows 
are independent sub-gaussian isotropic random vectors \cite[Def. 5.19, Def. 5.22]{vershynin2012nonasym} in $\reals^n$. Then, for $t \geq 0$, w.p. at least $1-2 e^{-ct^2}$, we have
\begin{equation}
\sqrt{m}-C\sqrt{n}-t \leq \minsingv{\mA} \leq \maxsingv{\mA} \leq \sqrt{m} + C\sqrt{n} + t.
\end{equation}
Here $C=C_K, c=c_K > 0$ depend only on the sub-gaussian norm $K=\max_i \norm[\Psi_2]{\mA_{:,i}}$ of
the rows of $\mA$.
\end{theorem}

\begin{theorem}[{\cite[Cor. 5.35]{vershynin2012nonasym}}] \label{thm:SupGaussSingval}
Let $\mA$ be an $m \times n$ matrix with i.i.d. standard normal entries. Then, for $t \geq 0$, w.p. at least $1-2 e^{-t^2/2}$, we have
\begin{equation}
\sqrt{m}- \sqrt{n}-t \leq \minsingv{\mA} \leq \maxsingv{\mA} \leq \sqrt{m} +  \sqrt{n} + t.
\end{equation}
\end{theorem}

\bibliography{subspaceclustering.bib}

\end{document}